%% file: main.tex
\documentclass[twoside,11pt]{article}

\usepackage[preprint]{jmlr2e}

\input{preamble}

\begin{document}
\doparttoc
\faketableofcontents
\parttoc

\author{\name Max Qiushi Lin \email maxqslin@gmail.com \\
       \addr  Simon Fraser University 
       \AND
       \name Jincheng Mei  \email jcmei@google.com \\
       \addr Google DeepMind 
       \AND
       \name Matin Aghaei \email m.jaffaraghaei@gmail.com \\
       \addr Simon Fraser University
       \AND
       \name Michael Lu \email michael\_lu\_3@sfu.ca \\
       \addr Simon Fraser University
       \AND
       \name Bo Dai \email bodai@google.com \\
       \addr Google DeepMind
       \AND
       \name Alekh Agarwal  \email alekhagarwal@google.com \\
       \addr Google Research
       \AND
       \name Dale Schuurmans \email daes@ualberta.ca\\
       \addr Google DeepMind and University of Alberta
       \AND 
       \name Csaba Szepesv{\' a}ri \email szepi@google.com \\
       \addr  Google DeepMind and University of Alberta
       \AND
       \name Sharan Vaswani \email vaswani.sharan@gmail.com \\
       \addr Simon Fraser University
}

\title{Rethinking the Global Convergence of Softmax Policy \\ Gradient with Linear Function Approximation: \\ The Case of Multi-Armed Bandits}

\maketitle

\begin{abstract}
Policy gradient (PG) methods have played an essential role in the empirical successes of reinforcement learning. In order to handle large state-action spaces, PG methods are typically used with \textit{function approximation}. In this setting, the \textit{approximation error} in modeling problem-dependent quantities is a key notion for characterizing the global convergence of PG methods. We study Softmax PG with linear function approximation (referred to as $\LinSPG$) and demonstrate that the approximation error is irrelevant to the algorithm's global convergence even in the bandit setting. Consequently, we rethink the effect of approximation error in the standard stochastic multi-armed bandit problem. We first identify the conditions on the policy feature representation that can guarantee the asymptotic global convergence of $\LinSPG$. Under these feature conditions, we further prove that $T$ iterations of $\LinSPG$ with a problem-specific learning rate result in an $O(1/T)$ convergence to the optimal policy. Moreover, we prove that $\LinSPG$ with an \emph{arbitrary} constant learning rate can ensure asymptotic convergence to the optimal policy.
\end{abstract}

\textbf{Keywords:} Softmax Policy Gradient, Linear Function Approximation, Stochastic Bandits, Global Convergence, Approximation Error

\section{Introduction}
\label{sec:introduction}

Policy gradient (PG) methods~\citep{williams1991function,sutton1999policy,konda1999actor,kakade2001natural} are an important class of algorithms in reinforcement learning (RL). These algorithms have been the backbone of prominent successes of RL in real-world applications such as controlling robots~\citep{kober2013reinforcement} and aligning large language models~\citep{uc2023survey}. Therefore, a deeper understanding of these methods is essential for developing more principled and effective algorithms. 

Although the policy optimization objective is non-concave~\citep{agarwal2021theory}, PG methods have been shown to achieve global convergence in the simplified \textit{tabular} setting, where there is one parameter per state-action pair~\citep{agarwal2021theory,mei2020global,mei2022role,yuan2022general,yuan2022linear,mei2024stepsize,bhandari2021linear,lan2023policy}. However, it is impractical to parameterize the policy by explicitly enumerating over the states and actions. Hence, it is common to use \textit{function approximation} techniques (e.g., neural networks) to parameterize the policy~\citep{schulman2015trust,schulman2017proximal,haarnoja2018soft} and generalize across related states and actions. Consequently, understanding the behavior of PG methods under function approximation is crucial in practice.

Throughout this paper, we will consider the classic Softmax PG method~\citep[Section 2.7]{sutton2018reinforcement}. As a representative policy-based method, Softmax PG lays the foundation for widely used RL methods, including REINFORCE~\citep{williams1992simple}, actor-critic \citep{konda1999actor,haarnoja2018soft}, TRPO~\citep{schulman2015trust}, and PPO~\citep{schulman2017proximal}. In the function approximation setting, \citet{sutton1999policy} analyzed the convergence of the standard Softmax PG method with a \textit{compatible function approximation}, i.e., one that can exactly represent the policy value function. Using a compatible function approximation ensures that the resulting policy gradient is unbiased, and Softmax PG can converge to a stationary point of the policy optimization objective~\citep{sutton1999policy}. 

When the exact policy values are not realizable by the function approximation, the \textit{approximation error} is typically used to characterize how well the function approximation can capture the relevant problem quantities. For example, using the concept of approximation error, prior works~\citep{abbasi2019politex,agarwal2020pc,cayci2021linear,chen2022finite,alfano2022linear} have established global convergence results for Natural Policy Gradient (NPG)~\citep{kakade2001natural} and its variants in this additive form:
\begin{align}
\label{eq:standard_global_convergence_reesult_additive_approximation_error}
    \text{suboptimality gap} \leq \text{optimization error} + \text{approximation error} .
\end{align}
This implies that if the approximation error is small, a diminishing optimization error leads to a small suboptimality gap. In particular, these papers can establish global convergence to the optimal policy when the approximation error is zero. However, an additive bound such as~\cref{eq:standard_global_convergence_reesult_additive_approximation_error} has the inherent weakness that the approximation error will never be zero unless the function approximation is able to exactly represent the desired quantities. Moreover, the analyses of NPG cannot be directly applied to the standard Softmax PG. Hence, it remains an open question whether Softmax PG can converge to the optimal policy with non-zero approximation error.

To answer this question, we show that such an approximation error perspective is overly demanding when attempting to characterize the global convergence of the Softmax PG method. Specifically, we focus on stochastic bandits~\citep{lattimore2020bandit} and analyze the convergence of Softmax PG with linear function approximation (referred to as $\LinSPG$) with a fixed set of features. In particular, we make the following contributions. 

\vspace{0.3ex}
\paragraph{Contribution 1:} In~\cref{sec:limitations}, we construct four bandit examples: two with comparable nonzero approximation errors and two with zero approximation error. Within each pair, we show that $\LinSPG$ converges to the optimal policy in one example but fails to converge in the other. These examples are considered in the so-called \textit{exact} setting where the algorithm has complete knowledge of the mean rewards, and there is no randomness in the updates. Consequently, we conclude that the failure of $\LinSPG$ is related to the feature representation and that the approximation error is not a reliable metric for characterizing global convergence even in the bandit setting. Given this result, we aim to answer the following question:
\begin{center}
\textit{For the bandit problems, under what conditions on the features is \\ $\LinSPG$ guaranteed to converge to the globally optimal policy?}    
\end{center}

\vspace{0.3ex}
\paragraph{Contribution 2:} In~\cref{sec:deterministic_linear_bandits}, we consider $\LinSPG$ in the exact bandit setting and identify the feature conditions that guarantee its asymptotic global convergence. Intuitively, we show that a key requirement for guaranteeing global convergence is that linear transformations (computed using the features) can retain the relative ordering of the rewards.  

\vspace{0.3ex}
\paragraph{Contribution 3:} In~\cref{sec:stochastic_linear_bandits}, we consider the standard stochastic bandit setting with unknown noisy rewards and analyze the convergence of $\LinSPG$ with on-policy sampling~\citep{mei2021understanding,mei2022role,mei2023stochastic}. We prove that under the same feature conditions as in the exact setting, $\LinSPG$ with a problem-specific constant learning rate ensures monotonic improvement in the expected reward. We use this property to show that the resulting algorithm achieves almost-sure asymptotic global convergence to the optimal policy. Furthermore, we prove that $\LinSPG$ converges to the optimal policy at an $O(1 / T)$ rate, matching the analogous result in the tabular setting~\citep{mei2023stochastic}.

\vspace{0.3ex}
\paragraph{Contribution 4:} The analysis in~\cref{sec:stochastic_linear_bandits} relies on a carefully chosen small-enough learning rate that helps exploit the objective's smoothness and control the noise in the stochastic policy gradient. One disadvantage of this approach is that the learning rate depends on unknown problem-dependent quantities, limiting the practical utility of the resulting algorithm. Recently,~\citet{mei2024stepsize} proved that tabular Softmax PG with an \emph{arbitrary} constant learning rate can achieve asymptotic global convergence in the stochastic bandit setting. In~\cref{sec:stochastic_linear_bandits_with_arbitrary_learning_rates}, we generalize this result to the linear function approximation setting. Specifically, we prove that under similar feature conditions and with an arbitrary constant learning rate, $\LinSPG$ is guaranteed to converge to the optimal policy. In addition, we prove that the average suboptimality asymptotically decreases at an $O(\ln(T) / T)$ rate.

\section{Problem Formulation}
\label{sec:problem_formulation}

We study the policy optimization problem for $K$-armed stochastic bandits with $K \geq 2$~\citep{lattimore2020bandit} specified by a true mean reward vector $r \in \R^K$. In particular, for each action $a \in [K] \coloneqq \{ 1, 2, \dots, K \}$, $r(a) \coloneq \int_{-R_{\max}}^{R_{\max}} x\,P_a(x) \mu(dx)$, where $R_{\max} > 0$ is the reward range, $\mu$ is a finite measure over $[-R_{\max}, R_{\max}]$, and $P_a(x) \ge 0$ is the probability density function with respect to $\mu$. We define $\mathcal{R}_a$ to be the reward distribution for the action $a$ defined by the density $P_a$ and the base measure $\mu$. In the stochastic setting, conditional on the past history $\mathcal{F}_t$ and the selected action $a_t=a$, the reward $R_t(a_t)$ is an independent draw from $\mathcal{R}_a$. In particular, $\mathbb{E}[R_t(a_t)\mid a_t=a,\mathcal{F}_t]=r(a)$ and $|R_t(a_t)|\le R_{\max}$ almost surely. For simplicity, we also introduce the following assumption.

\begin{assumption}[Unique True Mean Reward] \label{assumption:no_identical_arms}
For all $i, j \in [K]$, if $i \neq j$, $r(i) \neq r(j)$.
\end{assumption}
\cref{assumption:no_identical_arms} ensures that the mean rewards for all actions are distinct, thus guaranteeing a unique optimal action. This assumption has been widely used by existing works~\citep{mei2024ordering,mei2024stepsize} to ensure convergence to strict one-hot policies. Moreover, assuming a unique optimal action simplifies the formulation of subsequent feature-related assumptions. We believe that our results would continue to hold without~\cref{assumption:no_identical_arms}.  

The objective is to find a parametric policy $\pi_\theta$ that maximizes the expected reward:
\begin{align}
\label{eq:expected_reward_maximization}
    \sup_{\theta \in \sR^d} \dpd{\pi_\theta, r} ,
\end{align}
where $\theta \in \sR^d$ is the parameter to be learned  and $\pi_\theta = \softmax(\phi\,\theta)$ is referred to as a log-linear policy~\citep{agarwal2021theory,yuan2022linear}. Specifically, for each action $a \in [K]$, the policy can be represented as
\begin{align}
\label{eq:log_linear_policy}
    \pi_{\theta}(a) = \text{softmax}(\phi\,\theta)(a) = \frac{ \exp( [\phi\,\theta](a) ) }{ \sum_{a^\prime \in [K]}{ \exp( [\phi\,\theta](a^\prime) ) } } = \frac{ \exp( \langle \phi_a, \theta \rangle )}{ \sum_{a^\prime \in [K]}{ \exp( \langle \phi_{a^\prime}, \theta \rangle)  } } ,
\end{align}
where $\phi \in \sR^{K \times d}$ ($d \leq K$) is the given feature matrix, and $\phi_a \in \sR^d$ is the vector that characterizes the feature of action $a$. We also define the logits as $z_\theta \coloneq \phi\,\theta \in \R^K$. With some abuse of notation, the policy can be equivalently expressed in terms of the logits, i.e., $\pi_{\theta} = \pi_{z_\theta} \coloneq \softmax(z_\theta)$.

There are two major difficulties with the policy optimization problem in~\cref{eq:expected_reward_maximization}. First, due to the softmax transform,~\cref{eq:expected_reward_maximization} is a non-concave maximization problem w.r.t. $\theta$~\citep[Proposition 1]{mei2020global}. Second, if $d < K$, both $\pi_\theta$ and $\phi\,\theta$ are restricted to low-dimensional manifolds, implying that some specific policies and rewards can be unrealizable by the linear function approximation. In particular, the parametric log-linear policy $\pi_\theta = \softmax(\phi\,\theta)$ cannot well approximate every policy in the $K$-dimensional probability simplex, and the logit $z_\theta \in \sR^K$ might not well approximate the true mean reward $r \in \sR^K$. \\

\paragraph{Notation.} Without loss of generality, we assume $r(1) > r(2) > \cdots > r(K)$ as ties between distinct actions cannot occur under~\cref{assumption:no_identical_arms}. The optimal action $a^\star$ is the one with the largest true mean reward, i.e., $a^\star \coloneqq \arg\max_{a} r(a)$. Throughout, we use $r(1)$ and $r(a^\star)$ interchangeably, and note that the optimal policy $\pi^\star$ assigns all its probability mass to action $a^\star$, i.e. $\pi^\star(a^\star) = 1$ and $\pi^\star(a) = 0$ for all $a \neq a^\star$. Also, under~\cref{assumption:no_identical_arms}, we can define the non-zero reward gap as $\Delta \coloneq \min_{i, j \in [K], i \neq j} \abs{r(i) - r(j)} > 0$. Besides, we denote $\lambda_{\max} (M)$ as the largest eigenvalue of any square matrix $M$.

\section{Limitations of Approximation Error in Characterizing Convergence}
\label{sec:limitations}

A common first step in characterizing the convergence of PG methods~\citep{agarwal2021theory,mei2020global} is to consider the \textit{exact} setting, where the true rewards are known (\cref{sec:det-spg}). In~\cref{subsec:approx_error_irrelevant}, we show that even for this simple setting, the approximation error is not a useful structural measure to characterize the global convergence of Softmax PG with linear function approximation (referred to as $\LinSPG$).   

\subsection{\texorpdfstring{$\LinSPG$}{LinSPG} in the Exact Setting}
\label{sec:det-spg}
$\LinSPG$ is an instantiation of gradient ascent, which updates the learnable parameter by using the gradient calculated by the chain rule:
\begin{align*}
    \frac{d\,\dpd{\pi_{\theta}, r}}{d \theta} = \frac{d\,(\phi\,\theta)}{d \theta} \left( \frac{d\,\pi_{\theta}}{d\,(\phi\,\theta)} \right)^\top \frac{d\,\dpd{\pi_{\theta}, r}}{d \pi_{\theta}} = \phi^\top\,( \text{diag}{(\pi_{\theta})} - \pi_{\theta}\,\pi_{\theta}^\top )\,r .
\end{align*}

Since the rewards are assumed to be known in the exact setting, the above gradient can be calculated exactly.~\cref{alg:det_spg} gives the pseudo-code for the resulting algorithm.

\begin{algorithm}[H]
    \centering
    \caption{$\LinSPG$ in the Exact Setting}
    \label{alg:det_spg}
    \begin{algorithmic}
        \STATE {\bfseries input:} Initial parameters $\theta_1 \in \R^d$, learning rate $\eta > 0$
        \FOR{$t = 1, 2, \cdots, T$}
            \STATE $\theta_{t+1} = \theta_{t} + \eta\,\phi^\top\,( \text{diag}{(\pi_{\theta_t})} - \pi_{\theta_t}\,\pi_{\theta_t}^\top )\,r$
        \ENDFOR
        \STATE {\bfseries return:} $\pi_{\theta_{T+1}} = \softmax(\phi\,\theta_{T+1})$
    \end{algorithmic}
\end{algorithm}

The convergence results of PG methods with linear function approximation are commonly expressed in terms of the approximation error~\citep{abbasi2019politex,abbasi2019exploration,agarwal2020pc,cayci2021linear,chen2022finite,alfano2022linear}. The approximation error captures the expressivity of the feature matrix and is defined as: 
\begin{align}
\label{eq:approx_error}
\eps_{\text{approx}} \coloneq \min_{w \in \R^d} \norm{\phi\,w - r}_2 = \big\| \phi \left( \phi^\top\,\phi \right)^\dagger \phi^\top\,r - r \big\|_2 ,
\end{align}
where $\left( \phi^\top\,\phi \right)^\dagger$ denotes the Moore--Penrose pseudoinverse of $\phi^\top\,\phi$.
In the special case when $d = K$ and $\phi = \mathbf{I}_K$, we have $\eps_{\text{approx}} = 0$. However, in general, even with linearly realizable rewards (zero approximation error), establishing the global convergence of $\LinSPG$ is an open question~\citep{agarwal2021theory}. One intuitive reason why this is difficult is that, compared to the regression-based updates of natural policy gradient~\citep{kakade2001natural} with linear function approximation~\citep{yuan2022linear,alfano2022linear}, the gradient update in $\LinSPG$ is less directly connected to the concept of approximation error. In the next section, we further demonstrate the effect of approximation error on the global convergence of $\LinSPG$.

\subsection{Approximation Error is Irrelevant to Global Convergence}
\label{subsec:approx_error_irrelevant}

We consider the following four bandit examples, each with $3$ or $4$ actions and $2$-dimensional feature vectors describing each action.

\begin{example}
    \label{eg:first_example}
    Let $K = 4$, $d = 2$, $\phi^\top = \begin{bmatrix}\,0 & -1 & 0 & 2\,\vspace{0.3ex} \\\,-2 & 0 & 1 & 0\,\end{bmatrix}$ and $r = \left(9, 8, 7, 6 \right)^\top$. \\
    Using~\cref{eq:approx_error} to calculate the approximation error, we have $\epsilon_{\text{approx}} = \sqrt{202.6} \approx 14.2338$.
\end{example}
\begin{example}
    \label{eg:second_example}
    Let $K = 4$, $d = 2$, $\phi^\top = \begin{bmatrix}\,0 & 0 & -1 & 2\,\vspace{0.3ex}\\\,-2 & 1 & 0 & 0\,\end{bmatrix}$, and $r = \left(9, 8, 7, 6 \right)^\top$. \\
    Using~\cref{eq:approx_error} to calculate the approximation error, we have $\epsilon_{\text{approx}} = \sqrt{205} \approx 14.3178$.
\end{example}
\begin{example}
    \label{eg:third_example}
    Let $K = 3$, $d = 2$, $\phi^\top = \begin{bmatrix}\ 8 & 4 & 2\,\vspace{0.3ex} \\\,0 & 2 & 0\ \end{bmatrix}$, and $r = \left( 12, 8, 3 \right)^\top$. $\epsilon_{\text{approx}} = 0$.
\end{example}
\begin{example}
    \label{eg:fourth_example}
    Let $K = 3$, $d = 2$, $\phi^\top = \begin{bmatrix}\ 8 & 2 & 2\,\vspace{0.3ex} \\\,0 & 5 & 1\,\end{bmatrix}$, and $r = \left(12, 8, 4 \right)^\top$. $\epsilon_{\text{approx}} = 0$.
\end{example}

\begin{figure}[!ht]
    \centering
    \includegraphics[width=\linewidth]{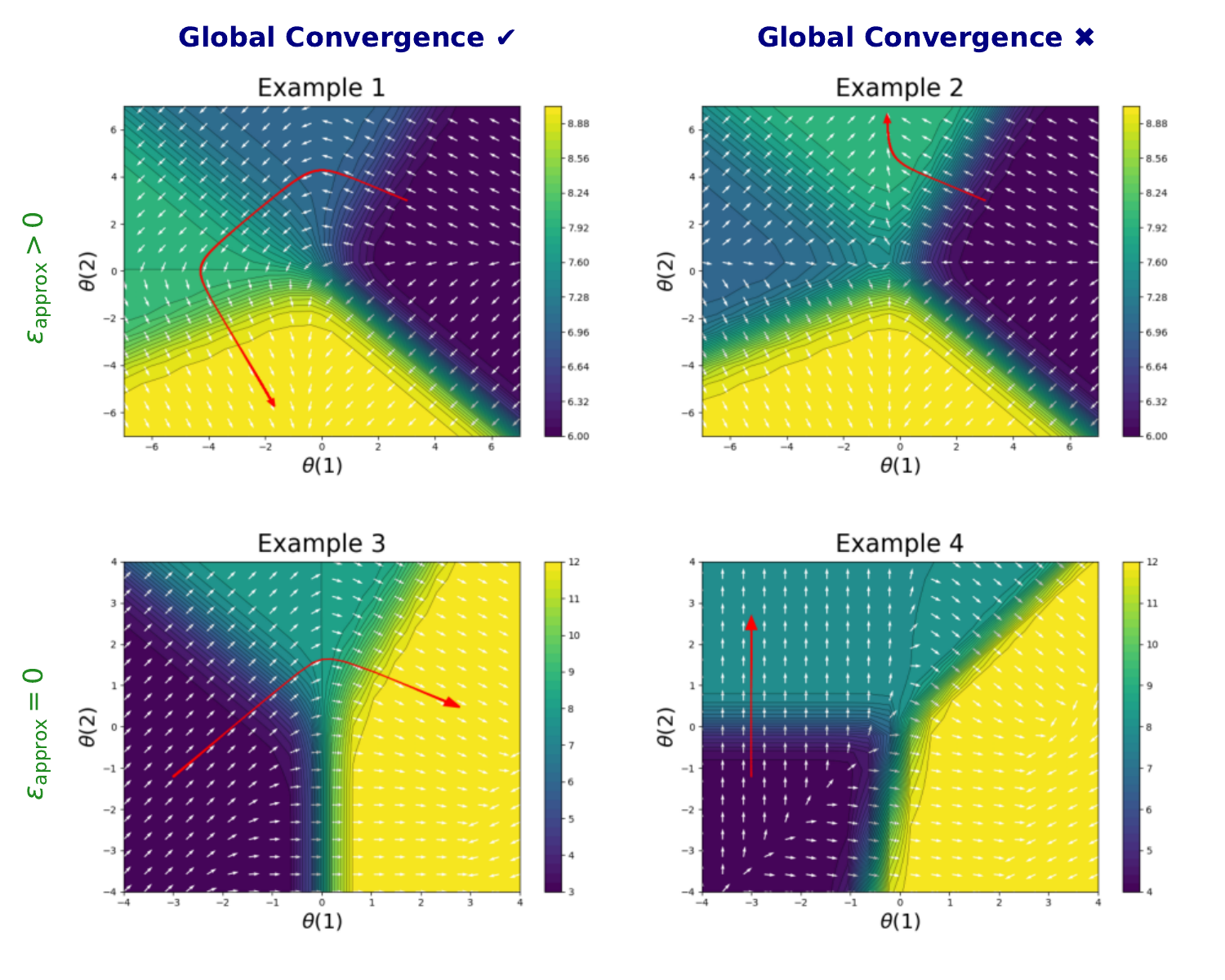}
    \caption{Visualization of the optimization landscape for~\cref{eg:first_example,eg:second_example,eg:third_example,eg:fourth_example}. For each $\theta \in \sR^2$, we calculate the expected reward $\dpd{\pi_{\theta}, r}$ for the log-linear policy $\pi_\theta$ defined in~\cref{eq:log_linear_policy} and color the optimization landscape with respect to its value. The white arrows indicate the gradient directions, whereas the red arrows demonstrate the optimization trajectories of running \cref{alg:det_spg} in these four examples. Specifically, in both~\cref{eg:first_example,eg:second_example}, we set $\theta_1 = (3, 3)^\top$, while in both~\cref{eg:third_example,eg:fourth_example}, we set $\theta_1 = (-3, -1.2)^\top$. The learning rate $\eta$ is set to $0.2$ for all examples.
    Intuitively, for~\cref{eg:first_example,eg:third_example}, all the gradient directions in the landscape demonstrate that~\cref{alg:det_spg} can converge to the optimal policy starting from any initialization. On the other hand, for~\cref{eg:second_example,eg:fourth_example}, we can identify a specific region of the landscape from which~\cref{alg:det_spg} fails to converge globally.
    }
    \label{fig:approx_error}
\end{figure}

Since $d \leq K$, not every policy is realizable using the resulting log-linear policy parametrization. In~\cref{eg:first_example,eg:second_example}, we note that the approximation error is larger than any suboptimality gap, i.e., for any policy $\pi$, $\dpd{\pi^\star - \pi, r} \le 3 < \epsilon_{\text{approx}}$, where $\pi^\star \coloneqq \argmax_{\pi \in \Delta_K} \dpd{\pi, r}$. In~\cref{eg:third_example,eg:fourth_example}, since $\phi\,w = r$ with $w = \left( \frac{3}{2}, 1 \right)^\top$, the approximation error is $0$.

We first prove that~\cref{alg:det_spg} is guaranteed to converge to the optimal policy for any initialization for~\cref{eg:first_example,eg:third_example}. This is despite the fact that~\cref{eg:first_example} has a non-zero approximation error, whereas the approximation error is zero for~\cref{eg:third_example}.
\begin{restatable}{proposition}{convergencewithapproxmiationerror}
\label{prop:softmax_pg_npg_global_convergence_first_example}
    For~\cref{eg:first_example,eg:third_example}, with a specific constant learning rate $\eta > 0$ and any initialization $\theta_1 \in \sR^d$,~\cref{alg:det_spg} guarantees that $\lim_{t \to \infty} \pit(a^\star) = 1$.
\end{restatable}
The complete proof is provided in~\cref{appendix:characterizing_convergence}. Hence, we conclude that the non-zero approximation error does not necessarily imply that~\cref{alg:det_spg} will fail to converge to the optimal policy. On the contrary, for~\cref{eg:second_example,eg:fourth_example}, we can find an initialization from which~\cref{alg:det_spg} fails to converge to the optimal policy. In this comparison,~\cref{eg:second_example} has a non-zero approximation error, whereas the approximation error is zero for~\cref{eg:fourth_example}.

\begin{restatable}{proposition}{nonconvergencewithapproxmiationerror}
\label{prop:linspg_global_convergence_zero_approx_error}
    For~\cref{eg:second_example,eg:fourth_example}, there exists an initialization $\theta_1 \in \sR^d$ such that~\cref{alg:det_spg} with any learning rate $\eta$ guarantees that $\lim_{t \to \infty} \pit(a^\star) \neq 1$.
\end{restatable}
The complete proof is provided in~\cref{appendix:characterizing_convergence}. Therefore, we can further conclude that the zero approximation error does not guarantee the convergence of~\cref{alg:det_spg} to the optimal policy either. The above examples demonstrate the limitations of using the approximation error as an indicator for global convergence. This further motivates the following question: \textit{What feature conditions guarantee the global convergence of $\LinSPG$?}

\section{Global Convergence: Exact Setting}
\label{sec:deterministic_linear_bandits}

In this section, we analyze the conditions under which $\LinSPG$ achieves global convergence to the optimal policy in the exact setting. Specifically, our objective is to characterize the feature and reward structure required to ensure the global convergence of~\cref{alg:det_spg}. 

To gain some intuition, consider~\cref{eg:first_example}, where $\LinSPG$ achieves global convergence. From the optimization landscape shown in~\cref{fig:approx_error}, when following the gradient, there appears to be a monotonic path to the optimal policy from any initialization point. Intuitively, this arises because the rewards of the actions seem to be nicely ``ordered''. 
For example, starting from $\theta_1 = (6,8)^\top$ such that $\dpd{\pi_{\theta_1}, r} \approx 6$, $\LinSPG$ can improve its expected reward eventually to $\dpd{\pi_{\theta_t}, r} \approx 7$ since there exists a suboptimal plateau with higher reward $7$ right beside the lowest plateau with reward $6$. Next, $\LinSPG$ continues to improve its expected reward eventually to $\dpd{\pi_{\theta_t}, r} \approx 8$ by ``climbing'' toward another neighboring plateau with a higher reward. Finally, this process ends with the algorithm successfully reaching the optimal plateau with reward $r(a^\star) = 9$.  In contrast, in \cref{eg:second_example}, as shown in \cref{fig:approx_error}, $\LinSPG$ gets stuck on a bad plateau with a local maximum reward of $8$. Visually, $\LinSPG$ stops improving its expected reward on this suboptimal plateau, because it is ``surrounded'' by two lower plateaus with rewards $6$ and $7$. This breaks the nice ``ordering'' of the expected reward landscape and traps the gradient ascent trajectory on a suboptimal plateau from which there is no monotonic ascent to global optimality.

Based on the above intuition, we conjecture that an ``ordering structure'' between the rewards is a key property behind the global convergence of $\LinSPG$. We instantiate such a conjecture for~\cref{eg:first_example,eg:third_example}. In particular, we determine whether a linear transformation computed using the feature matrix $\phi \in \sR^{K \times d}$ can preserve the same action ordering as the original reward vector $r\in \sR^K$. 

For instance, in~\cref{eg:first_example}, note that with $w = (-1, -1)^\top \in \sR^d$, we have $r^\prime \coloneqq \phi\,w = \left(2, 1, -1, -2 \right)^\top \in \sR^K$,
which preserves the ordering of $r = (9,8,7,6)$, meaning that for all $i, j \in [K]$, $r(i) > r(j)$ if and only if $r^\prime(i) > r^\prime(j)$. In this example, as stated in \cref{prop:softmax_pg_npg_global_convergence_first_example}, $\LinSPG$ can converge to the optimal action. In contrast, for \cref{eg:second_example}, it is impossible to find any $w \in \sR^d$ such that $\phi\,w$ preserves the ordering of the rewards $r$. 
To see why, consider any $w = (w(1), w(2))^\top$ and note that  $r^\prime \coloneqq \phi\,w = \left( -2 \cdot w(2), w(2), - w(1), 2 \cdot w(1) \right)^\top$. To preserve the reward order, we require both $-2 \cdot w(2) > w(2)$ (which would imply $w(2) < 0$) and $- w(1) > 2 \cdot w(1)$ (which would imply $w(1) < 0$).  
Together, these two conditions imply that $w(2) < 0 < - w(1)$, which means that $r^\prime(2) < r^\prime(3)$, and hence this reverses the order of the second and third actions. As shown in~\cref{eg:second_example}, this is an instance where $\LinSPG$ can fail to reach a global optimum. 

We formalize the above intuition and introduce the following assumption.

\begin{assumption}[Reward Ordering Preservation {\citep{mei2024ordering}}] \label{assumption:reward_ordering_preservation}
There exists a $w \in \R^d$ (w.l.o.g., $\norm{w}_2 = 1$) such that $r^\prime = \phi\,w$ preserves the ordering of the reward $r$, i.e., $r^\prime(i) > r^\prime(j)$ if and only if $r(i) > r(j)$.
\end{assumption}

\cref{assumption:reward_ordering_preservation} implies that the feature representation is expressive enough for a linear transformation to retain the relative ordering of the true rewards. This condition is weaker than requiring the exact realization of the true rewards and instead focuses on preserving their relative order. Under the aforementioned assumptions, we can choose a specific learning rate for $\LinSPG$ and establish a monotonic improvement guarantee on the expected reward. The complete proof is provided in \cref{appendix:deterministic_linear_bandit_additional_properties}.

\begin{restatable}{lemma}{monotonicityandonehotpolicy}
\label{lemma:monotonicity_and_onehotpolicy}
Under~\cref{assumption:no_identical_arms,assumption:reward_ordering_preservation}, consider running \cref{alg:det_spg} with the learning rate
\begin{equation}
\label{eq:step_size_for_deterministic_bandits}
0 < \eta < \frac{4}{9\,R_{\max}\,\lambda_{\max}(\phi\,\phi^\top)} .
\end{equation}
Then, it holds that
\begin{enumerate}[nolistsep,leftmargin=*,label=(\roman*)]
    \item For all finite $t \geq 1$, $\dpd{\pitt, r} > \dpd{\pit, r}$.
    \item There exists an action $a \in [K]$ such that $\lim_{t \to \infty} \pit(a) = 1$.
\end{enumerate}
\end{restatable}

\begin{proofsketch}
Since the softmax transform is smooth~\citep{agarwal2021theory,mei2020global} and the feature matrix $\phi$ has bounded values, $\dpd{\pi_{\theta}, r}$ is $L$-smooth with $L=\frac{9\,R_{\max}\,\lambda_{\max}(\phi\,\phi^\top)}{2}$ (\cref{lemma:smoothness_expected_reward_log_linear_policy}). This implies that $\LinSPG$ with a constant learning rate $0 < \eta < 2 / L$ will result in a monotonic increase in the expected reward, i.e.,
\begin{align*}
    \dpd{\pi_{\theta_{t+1}}, r}- \dpd{\pi_{\theta_t}, r} \geq \frac{1}{2 L}\,\norm{ \frac{d\,\dpd{\pit, r}}{d \theta_t} }_2^2 \geq 0 .
    \numberthis \label{eq:monotonic_improvement}
\end{align*}
Note that $\dpd{\pit,r}$ is monotonically increasing and upper bounded by $r(a^\star)$, and hence $\dpd{\pit,r}$ converges. Therefore, $\lim_{t\to\infty}\left\|\frac{d\dpd{\pit,r}}{d\theta_t}\right\|_2=0$. Since the rewards are distinct, we can define $\delta\coloneq\min_{1\leq i<j\leq K}(r^\prime(i)-r^\prime(j))(r(i)-r(j))>0$. By~\cref{lemma:alternative_expression_covariance}, we have $\frac{\delta}{2}\big(1-\sum_{i=1}^K\pit(i)^2\big)\leq\left\|\frac{d\dpd{\pit,r}}{d\theta_t}\right\|_2$, and hence $\lim_{t\to\infty}\big(1-\sum_{i=1}^K\pit(i)^2\big)=0$. Since $\sum_{i=1}^K\pit(i)^2\leq\max_{i\in[K]}\pit(i)$, it follows that $\lim_{t\to\infty}\max_{i\in[K]}\pit(i)=1$. Let $a_t\in\arg\max_{i\in[K]}\pit(i)$. Then, $\lim_{t\to\infty}\pit(a_t)=1$, and $\abs{\dpd{\pit,r}-r(a_t)}\leq2R_{\max}(1-\pit(a_t))$ implies $\lim_{t\to\infty}\abs{\dpd{\pit,r}-r(a_t)}=0$. Since $\dpd{\pit,r}$ converges, $a_t$ must be equal to some fixed action $a\in[K]$ for all sufficiently large $t$. Therefore, $\lim_{t\to\infty}\pit(a)=1$.
\end{proofsketch}

In~\cref{appendix:three_armed_deterministic_linear_bandits_necessity}, we construct an additional example (\cref{eg:necesity_feature_conditions}) to demonstrate that even when~\cref{assumption:no_identical_arms,assumption:reward_ordering_preservation} are satisfied,~\cref{alg:det_spg} is not guaranteed to achieve global convergence. Consequently, we require an additional assumption, which will be introduced in the next section. 

\subsection{Warm-Up: Global Convergence for \texorpdfstring{$K=3$}{K=3}}
\label{sec:deterministic_linear_bandits_warmup}

We begin by examining the three-armed bandit case as an illustrative example. To better understand the role of~\cref{assumption:reward_ordering_preservation}, observe that it requires a direction $w \in \R^d$ such that $r^\prime = \phi\,w$ (i.e., projecting $\phi(a)$ onto this direction $w$ for every action $a$) preserves the ordering of the reward. Note that this assumption can guarantee global convergence when $d = 1$ since, in this case, the one-dimensional features are ordered consistently with the rewards, and the scalar parameter update always moves in the direction that increases the relative logit of the optimal action. However,  when $d > 1$, according to~\cref{prop:necessity_for_three_armed_feature_conditions}, ~\cref{assumption:reward_ordering_preservation} is insufficient to guarantee global convergence.

To identify a sufficient condition for global convergence, we first consider the case when $d=2$. A natural way to strengthen~\cref{assumption:reward_ordering_preservation} is to require the features to preserve the reward ordering along an additional direction. In particular, assume that there exist two orthogonal directions $u$ and $v$ ($u, v \in \R^2$ and $\norm{u} = \norm{v} = 1$) such that $r^u \coloneq \phi\,u$ and $r^v \coloneq \phi\,v$ both preserve the ordering of the rewards. Then, we can rewrite the features as $\phi_i = r^u(i)\,u + r^v(i)\,v$ for all $i \in [3]$. Given this expression, we consider the following feature-dependent quantity:
\begin{align*}
\MoveEqLeft
\dpd{\phi_2 - \phi_3, \phi_1 - \phi_3} \\
=&\,\dpd{(r^u(2) - r^u(3))\,u + (r^v(2) - r^v(3))\,v, (r^u(1) - r^u(3))\,u + (r^v(1) - r^v(3))\,v} \\
=&\,(r^u(2) - r^u(3)) (r^u(1) - r^u(3))\,\normsq{u} + (r^v(2) - r^v(3)) (r^v(1) - r^v(3))\,\normsq{v} \tag{since $\dpd{u, v}$ = 0} \\
>&\,0 .\tag{$r^u$ and $r^v$ preserve the reward ordering}
\end{align*}
Formalizing the above intuition, we state the key feature condition that is required to guarantee global convergence.
\begin{assumption}[Feature Condition ($K=3$)]
\label{assumption:feature_conditions_for_three_armed_linear_bandits}
The feature matrix $\phi$ satisfies that \\ $\dpd{\phi_2 - \phi_3, \phi_1 - \phi_3} > 0$.
\end{assumption}

Our next result shows that in the three-armed bandit setting, the above assumptions are sufficient to ensure convergence to the optimal action. The complete proof can be found in~\cref{appendix:three_armed_deterministic_linear_bandits_sufficiency}.

\begin{restatable}{theorem}{threearmeddeterministiclinearbandits}
\label{theorem:three_armed_deterministic_linear_bandits}
Given a reward vector $r \in \mathbb{R}^{3}$ and a feature matrix $\phi \in \mathbb{R}^{3 \times d}$ such that $d \leq 3$ and \cref{assumption:no_identical_arms,assumption:reward_ordering_preservation,assumption:feature_conditions_for_three_armed_linear_bandits} are satisfied, \cref{alg:det_spg} with a constant learning rate as in \cref{eq:step_size_for_deterministic_bandits} is guaranteed to converge to the optimal policy.  
\end{restatable}

\begin{proofsketch}
Under~\cref{assumption:no_identical_arms,assumption:reward_ordering_preservation}, according to ~\cref{lemma:monotonicity_and_onehotpolicy}, for all finite $t \geq 1$, 
\begin{align*}
\dpd{\pi_{\thtt}, r} > \dpd{\pit, r} ,
\end{align*}
and $\lim_{t \to \infty} \pit(a) = 1$ for some action $a \in \{1, 2, 3\}$.
For any finite initialization $\theta_1$, we have $\dpd{\pi_{\theta_1}, r} > r(3)$ and hence $\dpd{\pit, r} > \dpd{\pi_{\theta_1}, r} > r(3)$. This implies that $\lim_{t \to \infty} \pit(3) \neq 1$. 

Hence, to complete the proof, we need to show that $\lim_{t \to \infty} \pit(2) \neq 1$. We prove this by contradiction. Suppose that $\lim_{t \to \infty} \pit(2) = 1$. In this case, we will show that $\lim_{t \to \infty} \frac{\pit(1)}{\pit(3)} = \infty$, which in turn implies that $\dpd{\pit, r} > r(2)$ for all large enough $t$. This contradicts the assumption that $\lim_{t \to \infty} \pit(2) = 1$. Specifically, by the update in~\cref{alg:det_spg}, we can first derive that, for all $t \geq 1$,
 \begin{align*}
   \frac{\pitt(1)}{\pitt(3)} = \frac{\pit(1)}{\pit(3)}\,\exp\Biggl\{\eta\,\underbrace{\parens*{\sum_{i=1}^3 \dpd{\phi_i, \phi_1 - \phi_3}\,\pit(i)\,(r(i) - \dpd{\pit, r})}}_{\eqcolon P_t}\Biggr\} .
\end{align*}
Under~\cref{assumption:feature_conditions_for_three_armed_linear_bandits}, we can guarantee that $P_t > 0$ for all $t \geq 1$ and hence $\frac{\pit(1)}{\pit(3)}$ monotonically increases with $t$. In particular, using~\cref{assumption:feature_conditions_for_three_armed_linear_bandits} and recursing, we can directly show that,
\begin{align}
\frac{\pit(1)}{\pit(3)} > \frac{\pi_{\theta_1}(1)}{\pi_{\theta_1}(3)} \exp(\eta\,\normsq{\phi_1 - \phi_3}\,(r(1) - r(2))\,\sum_{s=1}^{t-1} \pi_{\theta_s}(1)) .
\label{eq:3-arm-proof-intermediate}
\end{align}
Since $\frac{\pitt(1)}{\pitt(3)} > \frac{\pit(1)}{\pit(3)}$ for all finite $t \geq 1$, we can show that
\begin{align*}
\sum_{s=1}^t (1 - \pis(2)) < \parens*{1 + \frac{\pi_{\theta_1}(3)}{\pi_{\theta_1}(1)}} \sum_{s=1}^t \pis(1) .
\end{align*}
Moreover, under \cref{assumption:reward_ordering_preservation}, ~\cref{lemma:multi_arm_sum_to_infty} shows that $\sum_{s=1}^\infty (1-\pis(2)) = \infty$.  Combining the above relations, we can conclude that $\sum_{s=1}^\infty \pis(1) = \infty$. Combining with~\cref{eq:3-arm-proof-intermediate}, this implies that $\lim_{t \to \infty} \frac{\pit(1)}{\pit(3)} = \infty$. Hence, for all large enough $t$,
\begin{align*}
    r(2) - \dpd{\pit, r} &= \pi_{\theta_t}(1)\,( r(2) - r(1)) + \pi_{\theta_t}(3)\,( r(2) - r(3))  \\
    &= \pi_{\theta_t}(3)\,( r(2) - r(3))\,\bigg[ - \underbrace{\frac{ r(1) - r(2) }{ r(2) - r(3) }}_{> 0}\,\underbrace{\frac{ \pi_{\theta_t}(1) }{ \pi_{\theta_t}(3) }}_{\to \infty} + 1 \bigg] < 0 .
\end{align*}
Therefore, $\dpd{\pit, r} > r(2)$ for all large enough $t$. This contradicts the assumption that $\lim_{t \to \infty} \pit(2) = 1$, which completes the proof.
\end{proofsketch}

In~\cref{appendix:three_armed_deterministic_linear_bandits_necessity}, we construct multiple examples to show that \cref{assumption:reward_ordering_preservation,assumption:feature_conditions_for_three_armed_linear_bandits} are both independently necessary for guaranteeing global convergence. In particular, by ``necessary'', we mean that if any one of the conditions is omitted while the other remains satisfied, then one can construct a problem instance with a specific initialization for which $\LinSPG$ fails to converge to the optimal policy.

In the next section, we consider the general $K$-armed bandit setting ($K \geq 3$) and study the conditions required for global convergence. 

\subsection{Guarantee of Global Convergence for \texorpdfstring{$K \geq 3$}{K ≥ 3}}
\label{sec:deterministic_linear_bandits_K_geq_3}

The most direct way to extend~\cref{assumption:feature_conditions_for_three_armed_linear_bandits} to the general $K$-armed setting with $d \leq K$ is to construct features such that the reward ordering can be preserved when projecting onto $d$ different orthogonal directions, i.e., there exists a set of $d$ orthogonal vectors denoted as $\{u_1, \cdots, u_d\}$ ($u_p \in \R^d$ and $\norm{u_p} = 1$ for all $p \in [d]$) such that $r^p \coloneq \phi\,u_p$ preserves the reward ordering for all $p \in [d]$. Since $\{u_1, \cdots, u_d\}$ are orthogonal unit vectors, we have $\phi_i = \sum_{p=1}^d r^p(i)\,u_p$. In particular, consider an optimal action $a^\star$ and three actions $i$, $j$, and $k$ such that $r(i) > r(j)$ and $r(i) > r(k)$. Consider the following feature-dependent quantity.
\begin{align*}
\dpd{\phi_i - \phi_j, \phi_{a^\star} - \phi_k} &= \dpd{ \sum_{p=1}^d (r^p(i) - r^p(j))\,u_p, \sum_{q=1}^d (r^q(a^\star) - r^q(k))\,u_q} \\
&= \sum_{p=1}^d (r^p(i) - r^p(j)) (r^p(a^\star) - r^p(k))\,\normsq{u_p} \\ \tag{since $\forall p \neq q, \dpd{u_p, u_q} = 0$} \\
&> 0. \tag{$r^p$ preserves the reward ordering for all $p \in [d]$} 
\end{align*}
The above argument provides high-level intuition for identifying the feature conditions for the general $K$-armed setting. In particular, the above argument suggests requiring, for each valid triple $(i,j,k)$, the inner product $\dpd{\phi_i - \phi_j, \phi_{a^\star} - \phi_k}$ to be strictly positive. Additionally, we notice that the strict positivity in the above condition can be relaxed for most triples. With this relaxation, we obtain the following condition on the feature vectors. 

\begin{assumption}[Feature Conditions]\label{assumption:general_feature_conditions}
For any three actions $i$, $j$, and $k$ such that $r(i) > r(j)$ and $r(i) > r(k)$, the feature matrix $\phi$ satisfies that
\begin{align*}
\dpd{\phi_i - \phi_j, \phi_{a^\star} - \phi_k}
\begin{cases}
> 0 &\text{If $i = a^\star$ or $j = k$} \\
\geq 0 &\text{Otherwise}
\end{cases}
\end{align*}
\end{assumption}

\begin{remark}
    For the case when $K=3$,~\cref{assumption:feature_conditions_for_three_armed_linear_bandits} can be derived from~\cref{assumption:general_feature_conditions} by setting $i=2$, $j=3$, and $k=3$. However, compared to~\cref{assumption:feature_conditions_for_three_armed_linear_bandits},~\cref{assumption:general_feature_conditions} is a stronger assumption since it also requires the condition $\dpd{\phi_1 - \phi_2, \phi_1 - \phi_3} > 0$ to hold.
\end{remark}

In the special case when $d = K$ and $\phi = \mathbf{I}_K$, the features for each action correspond to one-hot vectors, and we can recover the standard multi-armed bandit setting with tabular parameterization. In this case, for all $i, j \in [K]$ such that $i \neq j$, we have $\dpd{\phi_i, \phi_i} = 1$ and $\dpd{\phi_i, \phi_j} = 0$. Consequently, this tabular setting satisfies the above feature condition. Hence, the subsequent proof techniques remain applicable for this simplified setting.    

We now show that $\LinSPG$ can achieve global convergence in the exact setting for the general $K$-armed bandit ($K \geq 3$). The complete proof is provided in~\cref{appendix:deterministic_linear_bandits}.

\begin{restatable}{theorem}{deterministiclinearbandits}
\label{theorem:deterministic_linear_bandits}
Given a reward vector $r \in \mathbb{R}^{K}$ and a feature matrix $\phi \in \mathbb{R}^{K \times d}$ such that $d \leq K$ and~\cref{assumption:no_identical_arms,assumption:reward_ordering_preservation,assumption:general_feature_conditions} are satisfied,~\cref{alg:det_spg} with a constant learning rate as in~\cref{eq:step_size_for_deterministic_bandits} converges to the optimal policy.  
\end{restatable}

\begin{proofsketch}
The proof has a similar structure to the one for~\cref{theorem:three_armed_deterministic_linear_bandits}. In particular, using~\cref{lemma:monotonicity_and_onehotpolicy}, we know that $\lim_{t \to \infty} \dpd{\pit, r} = r(a)$ for some action $a \in [K]$. Following the same reasoning as in~\cref{theorem:three_armed_deterministic_linear_bandits}, we know that $a \neq K$. Hence, we only need to show that $a \notin \{2, 3, \cdots, K-1\}$.

We will prove this by contradiction. Assume that $\lim_{t \to \infty} \dpd{\pit, r} = r(a)$ for some $a \in \{2, 3, \dots, K-1\}$. Therefore, there exists a large enough $\tau$ such that for all finite $t \geq \tau$, $r(a) > \dpd{\pit, r} > r(a + 1)$. Consider any action $k \in [a+1, K]$. Similar to the proof of~\cref{theorem:three_armed_deterministic_linear_bandits}, we can establish that
\begin{align*}
    \frac{\pitt(1)}{\pitt(k)} = \frac{\pit(1)}{\pit(k)}\,\exp\Biggl\{\eta\,\underbrace{ \Bigl( \sum_{i=1}^K \dpd{\phi_i, \phi_1 - \phi_{k}}\,\pit(i)\,(r(i) - \dpd{\pit, r}) \Bigr) }_{\eqcolon P_t}\Biggr\} .
\end{align*}
Under~\cref{assumption:general_feature_conditions}, we can guarantee that $P_t > C\,\pit(1) >  0$ for some positive constant $C > 0$. Hence, $\frac{\pit(1)}{\pit(k)}$ monotonically increases with $t$. In particular, using~\cref{assumption:general_feature_conditions} and recursing the above inequality until $\tau$, we have that for all finite $t \geq \tau$,
\begin{align*}
\frac{\pi_{\theta_{t}}(1)}{\pi_{\theta_{t}}(k)} \geq \frac{\pi_{\theta_\tau}(1)}{\pi_{\theta_\tau}(k)}\,\exp\Biggl(\eta\,C\,\sum_{s=\tau}^{t -1} \pi_{\theta_s}(1) \Biggr) > 0 ,
\end{align*}
Moreover, since $\frac{\pi_{\theta_t}(1)}{\pi_{\theta_t}(k)}$ is increasing for every $k\in[a+1,K]$, we have that $\frac{\pi_{\theta_t}(k)}{\pi_{\theta_t}(1)}\leq\frac{\pi_{\theta_\tau}(k)}{\pi_{\theta_\tau}(1)}$ for all finite $t\geq\tau$. 
Therefore, $\sum_{k=a+1}^{K}\pi_{\theta_t}(k)\leq D\pi_{\theta_t}(1)$, where $D:=\sum_{k=a+1}^{K}\frac{\pi_{\theta_\tau}(k)}{\pi_{\theta_\tau}(1)}<\infty$. Combining this with $r(a)>\langle\pi_{\theta_t},r\rangle$, we further obtain that $1-\pi_{\theta_t}(a)\leq D\left(1+\frac{r(a)-r(K)}{r(a-1)-r(a)}\right)\pi_{\theta_t}(1)$. Using~\cref{lemma:multi_arm_sum_to_infty}, we have that $\sum_{s=1}^{\infty}(1-\pi_{\theta_s}(a))=\infty$, and hence $\sum_{s=1}^{\infty}\pi_{\theta_s}(1)=\infty$. This further implies that $\lim_{t\to\infty}\frac{\pi_{\theta_t}(1)}{\pi_{\theta_t}(k)}=\infty$ for all $k\in[a+1,K]$. As a result, there exists a large enough iteration $\tau^\prime > \tau$ such that 
\begin{align*}
r(a) - \dpd{\pi_{\theta_{\tau^\prime}}, r}  < \pi_{\theta_{\tau^\prime}}(1)\,(r(1) - r(a)) \Biggl[ \sum_{i=a+1}^{K} \underbrace{\frac{\pi_{\theta_{\tau^\prime}}(i)}{\pi_{\theta_{\tau^\prime}}(1)} }_{\to 0}\,\underbrace{\frac{r(a) - r(i)}{r(1) - r(a)}}_{>0} - 1 \Biggr]< 0 .
\end{align*}
Therefore, we conclude that $\dpd{\pit, r} > r(a)$ for all $t \geq \tau^\prime$. This contradicts the assumption that $\lim_{t \to \infty} \dpd{\pit, r} = r(a)$. Hence, for all $a \in \{2, 3, \dots, K\}$, $\lim_{t \to \infty} \dpd{\pit, r} \neq r(a)$ and consequently, $\lim_{t \to \infty} \pit(a^\star) = 1$. 
\end{proofsketch}

\begin{remark}
The non-domination condition in \citet[Theorem 1]{mei2024ordering} is not enough to prove convergence to the globally optimal policy from any initialization. The authors of \cite{mei2024ordering} constructed a counterexample (\cref{eg:necesity_feature_conditions} in this paper) for their Theorem 1 and found that \cref{assumption:feature_conditions_for_three_armed_linear_bandits} can be used to fix the proofs when $K = 3$. We simplify their proofs for $K = 3$ and then further generalize the results to handle $K > 3$. 
\end{remark}

In~\cref{subsec:experiments_deterministic_linear_bandits}, we empirically verify the above theorem. In the next section, we analyze the convergence of $\LinSPG$ in the more practical stochastic setting, where the algorithm only has access to noisy estimates of the true mean rewards. 

\section{Global Convergence: Stochastic Setting}
\label{sec:stochastic_linear_bandits}

In~\cref{subsec:stochastic_algorithm}, we introduce the $\LinSPG$ method (\cref{alg:sto_spg}) in the stochastic setting. We first show that under the same assumptions as in~\cref{sec:deterministic_linear_bandits},~\cref{alg:sto_spg} with a suitably chosen constant learning rate guarantees monotonic improvement of the expected reward and convergence to a one-hot policy. In~\cref{subsec:decomposition_of_stochastic_process,subsec:small_step_size_global_convergence}, we use this property to prove that~\cref{alg:sto_spg} achieves almost-sure asymptotic global convergence to the optimal policy. Finally, we characterize the algorithm's convergence rate in~\cref{subsec:small_step_size_global_convergence_rate}. 

\subsection{\texorpdfstring{$\LinSPG$}{LinSPG} in the Stochastic Setting}
\label{subsec:stochastic_algorithm}
In~\cref{alg:det_spg}, we instantiated $\LinSPG$ in the exact setting where the algorithm has access to the true mean rewards. We now focus on the standard stochastic bandit setting~\citep{lattimore2020bandit} and consider the stochastic version of $\LinSPG$ (\cref{alg:sto_spg}). In this setting, at each iteration $t \in [T]$, the algorithm samples an action $a_t \sim \pit$ and receives a noisy reward $R_t(a_t)$ sampled from an unknown distribution $\mathcal{R}_{a_t}$. The reward $R_t(a_t)$ is then used to construct an unbiased gradient estimator using on-policy importance sampling (IS) reward estimates~\citep{sutton2018reinforcement,mei2023stochastic}.

\begin{algorithm}[H]
    \centering
    \caption{$\LinSPG$ in the Stochastic Setting}
    \label{alg:sto_spg}
    \begin{algorithmic}
        \STATE {\bfseries input:} Initial parameters $\theta_1 \in \R^d$, learning rate $\eta > 0$
        \FOR{$t = 1, 2, \cdots, T$}
            \STATE Sample an action $a_t \sim \pi_{\theta_t}(\cdot)$ and observe reward $R_t(a_t)\sim \mathcal{R}_{a_t}$
            \STATE $\theta_{t+1} = \theta_{t} + \eta\,\phi^\top\,(\mathrm{diag}(\pit) - \pit\,\pit^\top)\,\hat{r}_t$, where $\hat{r}_t(a) \coloneq \frac{\indicator{a = a_t}}{\pit(a)}\,R_t(a_t)$ for each $a \in [K]$
        \ENDFOR
        \STATE {\bfseries return:} $\pi_{\theta_{T+1}} = \softmax(\phi\,\theta_{T+1})$
    \end{algorithmic}
\end{algorithm}
Analogous to~\cref{lemma:monotonicity_and_onehotpolicy}, we first prove~\cref{lemma:monotonicity_for_stochastic_bandits}, showing that~\cref{alg:sto_spg} with a constant learning rate guarantees monotonic improvement of the expected reward. The complete proof of this lemma is provided in~\cref{appendix:softmax_pg_sublinear_rate}.

\begin{restatable}{lemma}{monotonicityforstochasticbandits}
\label{lemma:monotonicity_for_stochastic_bandits}
Suppose~\cref{assumption:no_identical_arms,assumption:reward_ordering_preservation} hold. In particular, for~\cref{assumption:reward_ordering_preservation}, suppose there exists a $w \in \mathbb{R}^d$ (w.l.o.g, $\norm{w}_2 = 1$) such that $r^\prime \coloneq \phi\,w$ preserves the ordering of the reward $r$. Let $\Delta \coloneq \min_{i \neq j} \abs{r(i) - r(j)} > 0$, $\Delta^\prime \coloneq \min_{i \neq j} \abs{r^\prime(i) - r^\prime(j)} > 0$, $\rho \coloneq \frac{8\,R^3_{\max}\,K^{3/2}}{\Delta^2}$, and $\gamma \coloneq \frac{( \Delta\,\Delta^\prime )^2}{16\,R_{\max}^2}$. Consider running~\cref{alg:sto_spg} with the learning rate
\begin{align}
0<\eta\leq\min\left\{\frac{\sqrt{\gamma}}{6\sqrt{2}\left[\lambda_{\max}(\phi\phi^\top)\right]^{3/2}R_{\max}},\frac{\gamma^{3/2}}{6\rho\left[\lambda_{\max}(\phi\phi^\top)\right]^{5/2}}\right\} .
\label{eq:step_size_for_stochastic_bandits}
\end{align}
Then, it holds that
\begin{enumerate}[nolistsep,leftmargin=*,label=(\roman*)]
\item For all $t \geq 1$, $\mathbb{E}_t[\langle\pi_{\theta_{t+1}},r\rangle]-\langle\pi_{\theta_t},r\rangle \geq \frac{\eta\,\gamma}{2}\left\|\frac{d\langle\pi_{\theta_t},r\rangle}{d(\phi\,\theta_t)}\right\|_2^2.$, where $\E_t[\cdot]$ denotes the conditional expectation with respect to the randomness in iteration $t$. 
\item There exists a (possibly random) action $a \in [K]$ such that $\lim_{t \to \infty} \pit(a) = 1$.
\end{enumerate}
\end{restatable}

\begin{proofsketch}
The proof relies on the following properties of the stochastic gradient estimates. First, according to~\cref{lemma:unbiased_gradient}, the stochastic gradient is unbiased, i.e. $\mathbb{E}_t\left[\frac{d\langle \pi_{\theta_t},\hat{r}_t\rangle}{d\theta_t}\right] = \frac{d\langle \pi_{\theta_t},r\rangle}{d\theta_t}$. Secondly, according to~\cref{lemma:sgc}, the stochastic gradients satisfy a variant of the \textit{strong growth condition} (SGC)~\citep{mei2023stochastic,schmidt2013fast,vaswani2019fast}:
\begin{align*}
    \E_t \left[ \normsq{\frac{d \dpd{\pit, \hat{r}_t}}{d \tht}} \right] \leq \rho\,\lambda_{\max}(\phi\,\phi^\top) \norm{\frac{d\dpd{\pit, r}}{d (\phi\,\tht)}}_2 .
\end{align*}
The above inequality implies that the variance in the stochastic gradients decreases as the algorithm approaches a stationary point. Additionally, combining this relation with~\cref{lemma:lpg_ns}, the objective also satisfies a \textit{non-uniform smoothness} property: 
\begin{align*}
    \left\|\frac{d^2\langle\pi_\theta,r\rangle}{d\theta^2}\right\|_2 \leq \frac{3\lambda_{\max}(\phi\phi^\top)}{\sqrt{\gamma}}\left\|\frac{d\langle\pi_\theta,r\rangle}{d\theta}\right\|_2 .
\end{align*}
The non-uniform smoothness property suggests that the optimization landscape becomes flatter as it gets closer to any stationary point. Using the above properties and following a proof similar to that in the tabular setting~\citep[Lemma 4.6]{mei2023stochastic}, we can prove that $\LinSPG$ can use a constant learning rate and ensure nonnegative conditional improvement of the expected reward, i.e., 
\begin{align*}
    \E_t[\dpd{\pitt, r}] - \dpd{\pit, r} \geq \frac{\eta\,\gamma}{2}\,\norm{\frac{d\,\dpd{\pit, r}}{d\,(\phi\,\tht)}}_2^2 .
\end{align*}
The above inequality implies that $\E_t[\dpd{\pitt, r}] \geq \dpd{\pit, r}$ for all finite $t \geq 1$. 
Hence, the sequence $\{\dpd{\pit, r}\}_{t=1}^\infty$ satisfies the condition of a sub-martingale. 
Using Doob’s martingale theorem (\cref{theorem:doobs}), we know that $\lim_{t \to \infty} \dpd{\pit, r}$ exists and is finite. This, along with the special covariance structure of $\LinSPG$ (\cref{lemma:alternative_expression_covariance}), further implies that $\pi_{\theta_t}$ approaches a one-hot policy as $t \to \infty$.
\end{proofsketch}

Given that the expected reward improves in expectation and the policy is guaranteed to approach a one-hot policy asymptotically, we now need to make sure that the policy does not converge to any suboptimal action. In order to show this, in the next section, we analyze the stochastic process corresponding to~\cref{alg:sto_spg} and handle the randomness arising from the sampling of actions and the noise in the rewards.

\subsection{Decomposition of Stochastic Process}
\label{subsec:decomposition_of_stochastic_process}

To prove the global convergence of~\cref{alg:sto_spg}, we need to show that $\lim_{t \to \infty} \pit(a^\star) = 1$ almost surely. For this, we will prove that $\lim_{t\to\infty}\left(z_t(a^\star)-z_t(a)\right)=\infty$ for all suboptimal actions $a\neq a^\star$.
To establish this, we consider the stochastic process corresponding to the logit $z_t(a)$ for an action $a \in [K]$. In particular, we define $\gF_t$ as the $\sigma$-algebra generated by $\{\theta_1, a_1$, $R_1(a_1)$, $\cdots$, $a_{t-1}$, $R_{t-1}(a_{t-1})\}$:
\begin{align*}
    \gF_t = \sigma( \{ \theta_1, a_1, R_1(a_1), \cdots, a_{t-1}, R_{t-1}(a_{t-1}) \} ) ,
\end{align*}
where $\theta_1\in \R^d$ is the fixed finite policy parameter at initialization. Note that for all finite $t \geq 1$, $\theta_{t}$ and $z_t$ are $\gF_t$-measurable, and $\hat{r}_t$ is $\gF_{t+1}$-measurable. We use $\mathbb{E}_t$ to denote the conditional expectation with respect to $\gF_t$, i.e., for any random variable $Z$, $\mathbb{E}_t[Z] \coloneq \mathbb{E}[Z|\gF_t]$. Under the stochastic bandit model specified in~\cref{sec:problem_formulation}, for every $t \geq 1$ and $a \in [K]$, $\mathbb{E}[R_t(a_t) \mid a_t = a, \mathcal{F}_t] = r(a)$ and $|R_t(a_t)| \leq R_{\max}$ almost surely. Based on the above $\sigma$-algebra, we decompose the difference between $z_{t+1}(a)$ and $z_t(a)$ into two components, the ``progress'' and ``noise'':
\begin{align*}
P_t(a) &\coloneq \E_{t}[z_{t+1}(a)] - z_t(a) , \tag{progress} \\
W_{t+1}(a) &\coloneq z_{t+1}(a) - \E_t{[ z_{t+1}(a)]} . \tag{noise}
\end{align*}
For any action $a \in [K]$ and $t > 1$, we can decompose the stochastic process for $\z_t(a)$:
\begin{align*}
    z_t(a) = W_t(a) + P_{t-1}(a) + z_{t-1}(a) .
\end{align*}
Note that the progress term has no additional randomness conditional on $\mathcal{F}_t$, and its subsequent analysis will be similar to the exact setting in~\cref{sec:deterministic_linear_bandits}. Consider a specific iteration $\tau \geq 1$. For any finite $t > \tau$, using the above decomposition and recursing it from $s=\tau$ to $t$, we can obtain that
\begin{align*}
z_t(a) = z_\tau(a) + \underbrace{\sum_{s=\tau}^{t-1} P_s(a)}_{\text{cumulative progress}} + \underbrace{\sum_{s=\tau+1}^{t} W_{s}(a)}_{\text{cumulative noise}} . \numberthis \label{eq:stochastic_logit_decomposition}
\end{align*}
Furthermore, for any two distinct actions $a_1, a_2 \in [K]$, using~\cref{eq:stochastic_logit_decomposition}, for all finite $t > \tau$,
\begin{align}
z_t(a_1) - z_t(a_2) = z_\tau(a_1) - z_\tau(a_2) + \underbrace{\sum_{s=\tau}^{t-1} \Bigl[P_s(a_1) - P_s(a_2)\Bigr]}_{\rm \text{(i)}} + \underbrace{\sum_{s=\tau+1}^{t} \Bigl[ W_{s}(a_1) - W_{s}(a_2)\Bigr]}_{\rm \text{(ii)}} . \label{eq:stochastic_logit_difference_decomposition}
\end{align}

This decomposition sets up a framework for most convergence analyses of Softmax PG in previous works~\cite{mei2021understanding,mei2023stochastic,mei2024stepsize}, and will be subsequently used to prove guarantees for $\LinSPG$ in the stochastic setting.  

\subsection{Guarantee of Global Convergence}\label{subsec:small_step_size_global_convergence}

Using the decomposition of the stochastic process in~\cref{subsec:decomposition_of_stochastic_process}, we now proceed to show that~\cref{alg:sto_spg} is guaranteed to achieve almost-sure global convergence to the optimal policy. 
The complete proof is provided in \cref{subsec:stochastic_linear_bandits}.
\begin{restatable}{theorem}{stochasticlinearbandits}
\label{theorem:stochastic_linear_bandits}
Given a reward vector $r \in \mathbb{R}^{K}$ and a feature matrix $\phi \in \mathbb{R}^{K \times d}$ such that $d \leq K$ and \cref{assumption:no_identical_arms,assumption:reward_ordering_preservation,assumption:general_feature_conditions} are satisfied,~\cref{alg:sto_spg} with the constant learning rate as in \cref{eq:step_size_for_stochastic_bandits} almost surely converges to the optimal policy.
\end{restatable}

\begin{proofsketch}
According to~\cref{lemma:monotonicity_for_stochastic_bandits}, $\{\dpd{\pit,r}\}_{t\geq1}$ is a bounded submartingale and therefore converges almost surely. Moreover, it guarantees that there exists a possibly random action $a\in[K]$ such that $\lim_{t\to\infty}\pit(a)=1$ almost surely. Consequently, $\lim_{t\to\infty}\dpd{\pit,r}=r(a)$ almost surely. Hence, it remains to show that the limiting action cannot be suboptimal.

We prove this by contradiction. Suppose that $\lim_{t\to\infty}\pit(k)=1$ for some suboptimal action $k$. Then, $\lim_{t\to\infty}\dpd{\pit,r}=r(k)$. Let $\Delta\coloneq\min_{1\leq i\leq K-1}(r(i)-r(i+1))>0$. For a sufficiently large finite iteration $\tau$ and all finite $t\geq\tau$,
\begin{align*}
\forall i<k:\,r(i)-\dpd{\pit,r}>\frac{\Delta}{2};\quad \forall i>k:\,\dpd{\pit,r}-r(i)>\frac{\Delta}{2}. \numberthis \label{eq:reward_gap_around_k}
\end{align*}
Thus, the convergence in~\cref{lemma:monotonicity_for_stochastic_bandits} provides a uniform reward gap of at least $\Delta/2$ between $\dpd{\pit,r}$ and every action lying above or below $k$ in the reward ordering. These inequalities will be used below to lower-bound the progress terms.

Using the above inequalities, we first establish that the probability mass outside action $k$ cannot be summable, i.e., $\sum_{t=1}^{\infty}(1-\pit(k))=\infty$.

Next, we establish a progress-versus-noise argument for comparing the optimal action with the relevant suboptimal actions through their logit differences. In particular, for any $a\neq a^\star$, the softmax parameterization implies that
\begin{align*}
\frac{\pit(a^\star)}{\pit(a)}=\exp(z_t(a^\star)-z_t(a)).
\end{align*}
By~\cref{eq:stochastic_logit_difference_decomposition}, the logit difference is the sum of the cumulative progress and cumulative noise. We use a progress-versus-noise argument to compare these two terms. We first define that
\begin{align*}
\mathcal{X}(k,a)\coloneq \left\{i\in[K]\mid \left|\dpd{\phi_i-\phi_k,\phi_{a^\star}-\phi_a}\right|>0 \right\} \quad \text{and} \quad S_t(k,a)\coloneq\sum_{s=\tau}^{t-1}\sum_{i\in\mathcal{X}(k,a)}\pis(i).
\end{align*}
The progress term can be written as
\begin{align*}
P_s(a^\star)-P_s(a)=\eta\,\sum_{i\neq k}\dpd{\phi_i-\phi_k,\phi_{a^\star}-\phi_a}\,\pis(i)\left(r(i)-\dpd{\pis,r}\right).
\end{align*}
Therefore,~\cref{eq:reward_gap_around_k} controls the reward factors in the above expression uniformly after iteration $\tau$, while~\cref{assumption:general_feature_conditions} controls the signs of the corresponding feature interactions. Since there are finitely many actions, the nonzero feature interactions appearing in the relevant action pairs are also bounded away from zero.

We then establish the following implication. Suppose there exist constants $C_1>0$ and $C_2\geq0$ such that the cumulative progress can be lower bounded as follows.
\begin{align}
\sum_{s=\tau}^{t-1}[P_s(a^\star)-P_s(a)]\geq \eta\,C_1\,S_t(k,a)-C_2. \label{eq:progress_condition_proofsketch}
\end{align}
Then, using the martingale concentration lemma (\cref{lemma:martingale}), the cumulative noise satisfies that, with probability at least $1-\delta$,
\begin{align*}
\sum_{s=\tau+1}^{t}[W_s(a^\star)-W_s(a)]\geq-\mathcal{O}\left(\sqrt{S_t(k,a)\,\log(S_t(k,a)/\delta)}\right).
\end{align*}
Therefore, whenever~\cref{eq:progress_condition_proofsketch} holds and $\lim_{t\to\infty}S_t(k,a)=\infty$, the cumulative progress asymptotically dominates the cumulative noise, and it follows almost surely that
\begin{align*}
\lim_{t\to\infty}[z_t(a^\star)-z_t(a)]=\infty \implies \lim_{t\to\infty}\frac{\pit(a^\star)}{\pit(a)}=\infty.
\end{align*}

We first prove that $\sP[\gC_K]=0$. On the event $\gC_K$,~\cref{assumption:general_feature_conditions} gives $\dpd{\phi_i-\phi_K,\phi_{a^\star}-\phi_K}>0$ for every $i<K$, and hence $\mathcal{X}(K,K)=[K-1]$. Setting $(k,a)=(K,K)$ in the progress expression above, all the terms correspond to actions $i<K$. By~\cref{eq:reward_gap_around_k}, $r(i)-\dpd{\pis,r}>\Delta/2$ for every such action, while~\cref{assumption:general_feature_conditions} makes every corresponding feature interaction strictly positive. Therefore, summing the progress over $s$ verifies~\cref{eq:progress_condition_proofsketch} with $C_2=0$ for a suitable $C_1>0$. Moreover, $S_t(K,K)=\sum_{s=\tau}^{t-1}\sum_{i=1}^{K-1}\pis(i)=\sum_{s=\tau}^{t-1}(1-\pis(K))$. Since $\sum_{s=\tau}^{\infty}(1-\pis(K))=\infty$, it follows that $\lim_{t\to\infty}S_t(K,K)=\infty$. Applying the progress-versus-noise argument then gives $\lim_{t\to\infty}\pit(a^\star)/\pit(K)=\infty$, which contradicts $\lim_{t\to\infty}\pit(K)=1$. Therefore, $\sP[\gC_K]=0$.

Next, we proceed to prove that $\sP[\gC_k]=0$ for any fixed $k\in[2,K-1]$. We first establish that there exists an action $a_0>k$ such that $\sum_{s=\tau}^{\infty}\pis(a_0)=\infty$. We next show that $\sum_{s=\tau}^{\infty}\pis(a^\star)=\infty$. For any $a>k$, the progress expression can be rewritten as
\begin{align*}
P_s(a^\star)-P_s(a)&=\eta\,\sum_{i=1}^{k-1}\underbrace{\dpd{\phi_i-\phi_k,\phi_{a^\star}-\phi_a}}_{\geq0}\pis(i)\underbrace{(r(i)-\dpd{\pis,r})}_{>\Delta/2}\\
&\quad+\eta\,\sum_{i=k+1}^{K}\underbrace{\dpd{\phi_k-\phi_i,\phi_{a^\star}-\phi_a}}_{\geq0}\pis(i)\underbrace{(\dpd{\pis,r}-r(i))}_{>\Delta/2}.
\end{align*}
By~\cref{eq:reward_gap_around_k}, both reward factors in the above expression are larger than $\Delta/2$. Moreover,~\cref{assumption:general_feature_conditions} makes both feature interactions nonnegative, and every nonzero interaction is bounded below by a positive constant. Hence, for every $a>k$, summing over $s$ verifies~\cref{eq:progress_condition_proofsketch} with $C_2=0$ for a suitable $C_1>0$. In particular, setting $a=a_0$,~\cref{assumption:general_feature_conditions} gives $\dpd{\phi_k-\phi_{a_0},\phi_{a^\star}-\phi_{a_0}}>0$, and hence $a_0\in\mathcal{X}(k,a_0)$. Therefore, $\lim_{t\to\infty}S_t(k,a_0)=\infty$, and using the above progress-versus-noise argument, we obtain $\lim_{t\to\infty}\pit(a^\star)/\pit(a_0)=\infty$. Since $\sum_{s=\tau}^{\infty}\pis(a_0)=\infty$, this further implies $\sum_{s=\tau}^{\infty}\pis(a^\star)=\infty$.

We now show that, for every action $a>k$, $\lim_{t\to\infty}\pit(a^\star)/\pit(a)=\infty$. By~\cref{assumption:general_feature_conditions}, $a^\star\in\mathcal{X}(k,a)$, so $S_t(k,a)\geq\sum_{s=\tau}^{t-1}\pis(a^\star)$ and $\lim_{t\to\infty}S_t(k,a)=\infty$. Applying~\cref{eq:progress_condition_proofsketch} and the same progress-versus-noise argument then yields that
\begin{align*}
\forall a > k:\,\lim_{t\to\infty}\frac{\pit(a^\star)}{\pit(a)}=\infty.
\end{align*}
Since there are finitely many actions, the above ratios imply that $\lim_{t\to\infty}\frac{\sum_{i=k+1}^{K}\pit(i)}{\sum_{i=1}^{k-1}\pit(i)}=0$. For the pair $(a^\star,k)$, setting $a=k$ in the progress expression shows that the terms with $i<k$ are favorable: by~\cref{eq:reward_gap_around_k}, their reward factors are larger than $\Delta/2$, and~\cref{assumption:general_feature_conditions} makes their feature interactions strictly positive. The remaining terms with $i>k$ have total magnitude bounded by a finite constant times $\sum_{i=k+1}^{K}\pis(i)$. Hence, the above ratio implies that the favorable contribution from actions $i\in[1,k-1]$ eventually dominates the adverse contribution from actions $i\in[k+1,K]$. Therefore, for a sufficiently large finite iteration $\tau$, there exist constants $C_3>0$ and $C_4\geq0$ such that
\begin{align*}
\sum_{s=\tau}^{t-1}[P_s(a^\star)-P_s(k)]\geq \eta\,C_3\,S_t(k,k)-C_4.
\end{align*}
Moreover, by~\cref{assumption:general_feature_conditions}, $a^\star\in\mathcal{X}(k,k)$, so $S_t(k,k)\geq\sum_{s=\tau}^{t-1}\pis(a^\star)$. Since $\sum_{s=\tau}^{\infty}\pis(a^\star)=\infty$, it follows that $\lim_{t\to\infty}S_t(k,k)=\infty$, and the same progress-versus-noise argument yields
\begin{align*}
\lim_{t\to\infty}\frac{\pit(a^\star)}{\pit(k)}=\infty.
\end{align*}

However, the assumption $\lim_{t\to\infty}\pit(k)=1$ implies $\lim_{t\to\infty}\pit(a^\star)=0$ and hence $\lim_{t\to\infty}\frac{\pit(a^\star)}{\pit(k)}=0$. This forms a contradiction with $\lim_{t\to\infty}\frac{\pit(a^\star)}{\pit(k)}=\infty$. Therefore, the policy cannot converge to any suboptimal action. Since~\cref{lemma:monotonicity_for_stochastic_bandits} guarantees asymptotic convergence to some one-hot policy, $\lim_{t\to\infty}\pit(a^\star)=1$ almost surely.
\end{proofsketch}

When using the above result for the tabular parameterization (i.e., setting $d=K$ and $\phi = \mathbf{I}_d$), we can recover the asymptotic convergence guarantee in~\cite{mei2023stochastic}. In particular,~\citet{mei2023stochastic} consider the multi-armed bandit setting and show that Softmax PG with the tabular parameterization converges to the optimal action. A\,t a high level, their proof technique analyzes the dynamics for each action, while our proof considers pairs of actions and analyzes the difference in the logits for the corresponding pair. This results in a shorter and arguably more elegant proof. 

Next, we characterize the rate at which~\cref{alg:sto_spg} converges to the optimal action.

\subsection{Rate of Convergence}
\label{subsec:small_step_size_global_convergence_rate}

The following theorem establishes a conditional expected pseudo-rate for~\cref{alg:sto_spg} in stochastic bandits. The complete proof is provided in~\cref{appendix:softmax_pg_sublinear_rate}.

\begin{restatable}{theorem}{softmaxpgsublinearrate}
\label{theorem:softmax_pg_sublinear_rate}
Given a reward vector $r \in \mathbb{R}^{K}$ and a feature matrix $\phi \in \mathbb{R}^{K \times d}$ such that $d \leq K$ and \cref{assumption:no_identical_arms,assumption:reward_ordering_preservation,assumption:general_feature_conditions} are satisfied, \cref{alg:sto_spg} with the constant learning rate $\eta$ as in~\cref{eq:step_size_for_stochastic_bandits}
results in the following convergence rate that if $\mu \coloneq \left(\mathbb{E}\left[\left(\inf_{t\geq 1}\pi_{\theta_t}(a^\star)\right)^{-2}\right]\right)^{-1} > 0$, then it holds that
\begin{align*}
   \mathbb{E}\left[\langle\pi^\star,r\rangle-\langle\pi_{\theta_{T+1}},r\rangle\right]\leq\frac{2}{\eta\,\gamma\,\mu\,T} .
\end{align*}
\end{restatable}

\begin{proofsketch}
Under \cref{assumption:no_identical_arms,assumption:reward_ordering_preservation}, according to \cref{lemma:monotonicity_for_stochastic_bandits}, for all finite $t \geq 1$,
\begin{align*}
    \mathbb{E}_t\!\left[\langle \pi^\star,r\rangle-\langle \pi_{\theta_{t+1}},r\rangle\right]\leq \langle \pi^\star,r\rangle-\langle \pi_{\theta_t},r\rangle-\frac{\eta\,\gamma}{2}\left\|\frac{d\langle \pi_{\theta_t},r\rangle}{d(\phi\,\theta_t)}\right\|_2^2 .
\end{align*}
To show convergence to the optimal policy $\pi^{\star}$, we can rewrite the above inequality as
\begin{align*}
\E_t[\dpd{\pi^{\star}, r} - \dpd{\pitt, r}] \leq \E_t[\dpd{\pi^{\star}, r} - \dpd{\pit, r}] - \frac{\eta\,\gamma}{2} \normsq{\frac{d \dpd{\pit, r}}{d (\phi\,\tht)}}.
\end{align*}
By~\cref{lemma:non_uniform_l}, $\dpd{\pitheta, r}$ satisfies the non-uniform Polyak–\L ojasiewicz condition with $\xi = 0$ and $C(\theta) = \pitheta(a^\star)$. Using this property, we have that
\begin{align*}
    \mathbb{E}_t\!\left[\langle \pi^\star,r\rangle-\langle \pi_{\theta_{t+1}},r\rangle\right]\leq \langle \pi^\star,r\rangle-\langle \pi_{\theta_t},r\rangle-\frac{\eta\,\gamma}{2}\left[\pi_{\theta_t}(a^\star)\right]^2\left(\langle \pi^\star,r\rangle-\langle \pi_{\theta_t},r\rangle\right)^2 ,
\end{align*}
Under the stated assumption $\mu:=\left(\mathbb{E}\left[\left(\inf_{s\geq1}\pi_{\theta_s}(a^\star)\right)^{-2}\right]\right)^{-1}>0$, taking expectation with respect to all previous iterations and applying Cauchy--Schwarz yields that
\begin{align*}
    \mathbb{E}[\langle\pi^\star,r\rangle-\langle\pi_{\theta_{t+1}},r\rangle]\leq\mathbb{E}[\langle\pi^\star,r\rangle-\langle\pi_{\theta_t},r\rangle]-\frac{\eta\,\gamma\,\mu}{2}\left(\mathbb{E}[\langle\pi^\star,r\rangle-\langle\pi_{\theta_t},r\rangle]\right)^2 .
\end{align*}
Solving the above recursive inequality, we can finally obtain:
\begin{align*}
    \mathbb{E}[\langle\pi^\star,r\rangle-\langle\pi_{\theta_{T+1}},r\rangle]\leq\frac{2}{\eta\,\gamma\,\mu\,T} .
\end{align*}
\end{proofsketch}

We note that the above convergence rate matches that of Softmax PG with the tabular parameterization~\citep[Theorem 5.5]{mei2023stochastic}. Similar to prior analyses of softmax policy gradient methods~\citep{mei2020global,lu2024towards}, the convergence rate presented in~\cref{theorem:softmax_pg_sublinear_rate} is a pseudo-rate and comes with an important caveat. The term $\mu^{-1}$ could scale in a prohibitive manner, and in the worst case, softmax policy gradient methods could take exponential time to converge to the global optimal policy (see~\citet{li2021softmax} for the corresponding lower bound). However, since the global convergence is guaranteed asymptotically (\cref{theorem:stochastic_linear_bandits}), for any $C \in (0, 1)$, along almost every trajectory there exists a sufficiently large finite iteration $\tau$ such that for all $t \geq \tau$, $\pi_t(a^\star) \geq C$. Consequently, along almost every trajectory, it holds that $\inf_{t\geq\tau}\pi_{\theta_t}(a^\star)\geq C$, and therefore $\left(\inf_{t\geq\tau}\pi_{\theta_t}(a^\star)\right)^{-2}\leq C^{-2}$. Therefore, along almost every trajectory, the stated $O(1/T)$ rate always holds asymptotically.

The convergence result in~\cref{theorem:stochastic_linear_bandits,theorem:softmax_pg_sublinear_rate} relies on carefully chosen learning rates that depend on unknown quantities such as the true mean reward gap. This limits the practical utility of the resulting algorithm. Consequently, in the next section, we leverage a recent result by~\citet{mei2024stepsize} and develop a different proof technique to show the asymptotic global convergence of $\LinSPG$ with \textit{any} constant learning rate.

\section{Global Convergence for Arbitrary Learning Rates}
\label{sec:stochastic_linear_bandits_with_arbitrary_learning_rates}
Recently,~\citet{mei2024stepsize} proved that in the bandit setting, stochastic Softmax PG with a tabular parameterization and any \textit{arbitrary} large constant learning rate is guaranteed to asymptotically converge to the optimal policy. In~\cref{subsec:large_step_size_global_convergence}, we first generalize this result to $\LinSPG$ and prove the asymptotic convergence of~\cref{alg:sto_spg} with arbitrary constant learning rates. Subsequently, we characterize the algorithm's asymptotic rate of convergence in~\cref{subsec:large_step_size_rate}. Finally, in~\cref{subsec:experiments_stochastic_linear_bandits}, we empirically evaluate~\cref{alg:sto_spg} with different learning rates. 

\subsection{Guarantee of Global Convergence}
\label{subsec:large_step_size_global_convergence}
The proof of~\cref{theorem:stochastic_linear_bandits} heavily relied on using a learning rate that guarantees monotonic improvement in the expected reward. Since~\cref{alg:sto_spg} with an arbitrary constant learning rate does not have such a guarantee, we use a different proof technique to show the algorithm's asymptotic global convergence. The complete proof is provided in~\cref{appendix:stochastic_linear_bandits_with_arbitrary_learning_rates_proof}.

\begin{restatable}{theorem}{stochasticlinearbanditswitharbitrarylearningrates}
\label{theorem:stochastic_linear_bandits_with_arbitrary_learning_rates}
Given a reward vector $r \in \mathbb{R}^{K}$ and a feature matrix $\phi \in \mathbb{R}^{K \times d}$ such that $d \leq K$ and \cref{assumption:no_identical_arms,assumption:general_feature_conditions} are satisfied, \cref{alg:sto_spg} with an arbitrary but constant learning rate converges to the optimal policy almost surely.
\end{restatable}

\begin{proofsketch}
We first introduce the following definitions. We define $N_t(a)$ as the number of times action $a$ has been sampled until iteration $t$ and $N_\infty(a) \coloneq \lim_{t \to \infty} N_t (a)$. We further define $\mathcal{A}_\infty$ as the subset of actions that are sampled infinitely many times as $t \to \infty$, i.e.,
\begin{align*}
    \mathcal{A}_\infty \coloneq \{ a \in [K] \mid N_\infty(a) = \infty \} .
\end{align*}
According to~\cref{lemma:2_arms_pulled_infinitely_many_times,lemma:a_star_in_A_infty}, we first establish that $K_\infty \coloneq \abs{\mathcal{A}_\infty} \geq 2$ and $a^\star \in \mathcal{A}_\infty$. We also sort the action indices in $\mathcal{A}_\infty$ such that $r(a^\star) = r(i_1) > r(i_2) > \cdots > r(i_{K_\infty})$.
Then, we introduce the following claim on this subset of actions, $\mathcal{A}_\infty$.\\

\paragraph{Claim:} If $m = K_\infty$ or for all $p \in [m+1, K_\infty]$, $\lim_{t \to \infty} \frac{\pit(a^\star)}{\pit(i_p)} = \infty$ almost surely, then it holds that almost surely
\begin{align*}
    \lim_{t \to \infty} \frac{\pit(a^\star)}{\pit(i_{m})} = \infty .
\end{align*}
The formal version of this claim is stated in~\cref{lemma:induction_on_A_infty}, and its proof is provided in~\cref{subsubsec:proofs_lemmas_arbitrary_learning_rates}. Using this claim, we can prove via strong backward induction (from $m = K_\infty$ to $m = 2$) that
\begin{align*}
    \forall i_m \in \mathcal{A}_\infty \backslash \{a^\star\}: \; \lim_{t \to \infty} \frac{\pit(a^\star)}{\pit(i_{m})} = \infty .
\end{align*}
Moreover, according to the Borel--Cantelli lemma (\cref{lemma:extended_borel_cantelli}), we have that almost surely,
\begin{align*}
    \forall a \notin \mathcal{A}_\infty: \; \sum_{t=1}^\infty \pit(a) < \infty .
\end{align*}
This implies that almost surely,
\begin{align*}
    \forall a \in [K] \backslash \mathcal{A}_\infty: \;  \lim_{t \to \infty} \pit(a) = 0 \implies \lim_{t \to \infty} \sum_{a \in [K] \backslash \mathcal{A}_\infty} \pit(a) = 0 .
\end{align*}
We also have that
\begin{align*}
    \pit (a^\star) = \frac{1 - \sum_{a \in [K] \backslash \mathcal{A}_\infty} \pit (a)}{1 + \sum_{a \in \mathcal{A}_\infty \backslash \{a^\star\} } \frac{\pit (a)}{\pit (a^\star)}} .
\end{align*}
Putting everything together, we have that almost surely,
\begin{align*}
    \lim_{t \to \infty} \pit (a^\star) = \lim_{t \to \infty} \frac{1 - \sum_{a \in [K] \backslash \mathcal{A}_\infty} \pit (a)}{1 + \sum_{a \in \mathcal{A}_\infty \backslash \{a^\star\}} \frac{\pit (a)}{\pit (a^\star)}} = 1 .
\end{align*}
This completes the proof.
\end{proofsketch}

In the special case of the tabular parameterization (i.e., setting $d=K$ and $\phi = \mathbf{I}_d$), the above result recovers the asymptotic convergence guarantee in~\citet{mei2024stepsize}. Compared to~\citet[Theorem~2]{mei2024stepsize},~\cref{theorem:stochastic_linear_bandits_with_arbitrary_learning_rates} follows a similar progress-versus-noise decomposition, but uses a different induction argument. Specifically, in the tabular setting, each action has an independent logit and the expected reward $\langle\pi_{\theta_t},r\rangle$ directly determines the sign of each logit's drift, i.e., actions with reward above it have positive progress and those below it have negative progress. Consequently, the induction proves that the policy's expected reward eventually exceeds the reward of a given suboptimal action and then uses the resulting negative drift to eliminate the action. On the other hand, in the linear case, the logits are coupled through the features, and an update to the parameter $\theta$ simultaneously changes every logit. Hence, we cannot analyze the logits individually. Instead, we consider pairwise logit gaps and use a strong induction argument over the actions that are sampled infinitely often. More specifically, we require the feature conditions (\cref{assumption:general_feature_conditions}) to ensure that the favorable pairwise progress terms eventually dominate the adverse pairwise interactions and the martingale noise.

\subsection{Rate of Convergence}
\label{subsec:large_step_size_rate}

Although using arbitrary constant learning rates can guarantee asymptotic global convergence to the optimal action, the resulting algorithm does not share the same pseudo-rate of convergence as in~\cref{theorem:softmax_pg_sublinear_rate}. In this case, characterizing the convergence rate is more difficult because the expected reward is not guaranteed to increase monotonically, but can oscillate or get stuck on plateaus (see the experiments in~\cref{fig:arbitrary_learning_rates} for examples of such behavior). However, we can still establish an \textit{asymptotic convergence rate} on the average suboptimality. In particular,~\cref{theorem:asymptotic_convergence_rate} shows that asymptotically, the average suboptimality converges at an $O(\ln(T)/T)$ rate. The complete proof is provided in~\cref{subsec:asymptotic_convergence_rate}.

\begin{restatable}{theorem}{asymptoticconvergencerate}
\label{theorem:asymptotic_convergence_rate}
Suppose \cref{assumption:no_identical_arms,assumption:general_feature_conditions} hold. Using~\cref{alg:sto_spg} with any constant learning rate $\eta>0$, almost surely, there exist a sufficiently large $\tau\geq1$ and constants $C>0$ and $C'\geq0$ such that for all $T>\tau$, it holds that
\begin{align*}
\frac{\sum_{s=\tau}^T r(a^\star) - \dpd{\pis, r}}{T - \tau + 1} \leq \frac{2 R_{\max} \left[ \frac{K}{C} \ln \left( C\,T + e^C  \right) + \frac{\pi^2\,K}{6\,C} + K\,C^\prime \right] }{T - \tau + 1} .
\end{align*}
\end{restatable}

\begin{proofsketch}
Following the proof of~\cref{lemma:induction_on_A_infty} (\cref{subsubsec:proofs_lemmas_arbitrary_learning_rates}), we can prove that there exists a large enough $\tau > 0$ and $C > 0$ such that the cumulative progress term in~\cref{eq:stochastic_logit_decomposition} dominates the cumulative noise. Consequently, for any action $i_m \in \mathcal{A}_\infty \backslash \{a^\star\}$ and all $t \geq \tau$, we have that almost surely,
\begin{align*}
&\pit(i_m) < \exp\parens*{-C\,\sum_{s=\tau}^{t-1} \pis(i_m)} \implies \sum_{s=\tau}^t \pis(i_m) - \sum_{s=\tau}^{t-1} \pis(i_m) < \exp(- C\,\sum_{s=\tau}^{t-1} \pis(i_m)) .
\end{align*}
Using \cref{lemma:sequence_equality,lemma:sequence_inequality} to solve the above recursive inequality, we have that
\begin{align}
\sum_{s=\tau}^t \pis(i_m) \leq \frac{1}{C} \ln \left( C\,t + e^C  \right) + \frac{\pi^2}{6\,C} . \nonumber
\end{align}
According to the Borel--Cantelli lemma (\cref{lemma:extended_borel_cantelli}), we also know that
\begin{align*}
    \forall k \in [K] \backslash \mathcal{A}_\infty: \; \sum_{s=\tau}^\infty \pis(k) < \infty .
\end{align*}
The above two inequalities yield that there exists a constant $C^\prime \geq 0$ such that
\begin{align*}
    \sum_{s=\tau}^t (1 - \pis(a^\star)) &= \sum_{k \in \mathcal{A}_\infty \backslash \{a^\star\}} \sum_{s=\tau}^t \pis(k) + \sum_{k \in [K] \backslash \mathcal{A}_\infty} \sum_{s=\tau}^t \pis(k) \\
    &\leq \frac{K}{C} \ln \left( C\,t + e^C  \right) + \frac{\pi^2\,K}{6\,C} + K\,C^\prime .
\end{align*}
Furthermore, we also note that the suboptimality gap can be written as follows.
\begin{align*}
    r(a^\star) - \dpd{\pis, r} &= \sum_{a \neq a^\star} \pis(a)\,(r(a^\star) - r(a)) \leq 2 R_{\max}\,(1 - \pis(a^\star)) .
\end{align*}
Averaging the suboptimality gap from $s = \tau$ to $T$ yields that
\begin{align*}
    \frac{\sum_{s=\tau}^{T} r(a^\star) - \dpd{\pis, r}}{T - \tau + 1} &\leq \frac{2 R_{\max}\,\sum_{s=\tau}^{T} (1 - \pis(a^\star)) }{T - \tau + 1} \leq \frac{2 R_{\max} \left[ \frac{K}{C} \ln \left( C\,T + e^C  \right) + \frac{\pi^2\,K}{6\,C} + K\,C^\prime \right] }{T - \tau + 1} .
\end{align*}
This completes the proof.
\end{proofsketch}

In the special case of the tabular parameterization (i.e., setting $d=K$ and $\phi = \mathbf{I}_d$), the above result recovers the asymptotic convergence guarantee in~\citet{mei2024stepsize}.

\section{Conclusions and Future Work}

Although the approximation error has been commonly used in analyses of PG methods, we show that it is not a reliable metric for characterizing the global convergence of $\LinSPG$. Therefore, we focus on the simple multi-armed bandit setting and identify the conditions on the feature representation under which $\LinSPG$ is guaranteed to converge to the optimal policy. Furthermore, we characterize the convergence rates of $\LinSPG$ when using either a sufficiently small, problem-specific constant learning rate or an arbitrary constant learning rate. Our work has made great progress towards understanding the global convergence of PG methods with linear function approximation, going well beyond the conventional approximation error-based analyses.

In the future, extending the results and techniques to general Markov decision processes is an important and challenging next step. Additionally, investigating whether our feature conditions can be used for better representation learning is an interesting question. Finally, another ambitious goal is to generalize the proof techniques to handle nonlinear complex function approximation.

\bibliography{references}

\clearpage
\appendix

\addcontentsline{toc}{section}{Appendix}
\part{Appendix}
\parttoc

\clearpage

\section{Proofs of \texorpdfstring{\cref{sec:limitations}}{} }
\label{appendix:characterizing_convergence}

\convergencewithapproxmiationerror*
\begin{proof} We use~\cref{theorem:deterministic_linear_bandits} to prove the convergence of both examples.
\paragraph{Convergence in \cref{eg:first_example}}
Let $w = (-1, -1)^\top \in \sR^d$. We have that $r^\prime \coloneqq \phi\,w = \left(2, 1, -1, -2 \right)^\top$ preserves the ordering of $r \in \sR^K$, such that for all $i, j \in [K]$, $r(i) > r(j)$ if and only if $r^\prime(i) > r^\prime(j)$. This suggests that \cref{eg:first_example} satisfies the \cref{assumption:reward_ordering_preservation}. Moreover, we can verify that \cref{eg:first_example} also satisfies \cref{assumption:general_feature_conditions}. Hence, \cref{theorem:deterministic_linear_bandits} shows that global convergence is guaranteed in \cref{eg:first_example}.

\paragraph{Convergence in \cref{eg:third_example}}
Since $r = \phi\,w$ where $w = \left( \frac{3}{2}, 1 \right)^\top$, \cref{eg:third_example} satisfies \cref{assumption:reward_ordering_preservation}. Moreover, we can verify that \cref{eg:third_example} also satisfies \cref{assumption:general_feature_conditions}. Hence, \cref{theorem:deterministic_linear_bandits} shows that global convergence is guaranteed in \cref{eg:third_example}.
\end{proof}

\nonconvergencewithapproxmiationerror*
\begin{proof}
We first start by proving the possible non-convergence of~\cref{eg:fourth_example}.
\paragraph{Non-convergence in \cref{eg:fourth_example}}
Recall that $\phi^\top = \begin{bmatrix}\ 8 & 2 & 2\,\vspace{0.3ex} \\\,0 & 5 & 1\,\end{bmatrix}$. Then, we have that $\dpd{\phi_2 - \phi_3, \phi_1 - \phi_3} = -4 < 0$, suggesting that~\cref{eg:fourth_example} violates~\cref{assumption:feature_conditions_for_three_armed_linear_bandits}. According to~\cref{prop:necessity_for_three_armed_feature_conditions}, there exists an initialization $\theta_1 = C\,(\phi_3 - \phi_1)$ for a sufficiently large $C$ such that the algorithm fails to converge to the optimal policy.

Finally, we proceed to prove the possible non-convergence of~\cref{eg:second_example}.
\paragraph{Non-convergence in \cref{eg:second_example}} Recall that $\phi^\top = \begin{bmatrix}\,0 & 0 & -1 & 2\,\vspace{0.3ex}\\\,-2 & 1 & 0 & 0\,\end{bmatrix}$ and $r = \left(9, 8, 7, 6 \right)^\top$.
We will prove the non-convergence of $\LinSPG$ by contradiction. Assume that $\LinSPG$ converges to the optimal policy under a certain initialization $\theta_1$. It thus implies that
\begin{align*}
    \forall a \neq 1: \lim_{t \to \infty} \frac{\pi_t(1)}{\pi_t(a)} = \infty .
\end{align*}
We will then show that $\lim_{t \to \infty} \frac{\pi_t(1)}{\pi_t(2)} < \infty$, hence contradicting the above assumption. \\ 
For a specific iteration $t$, we denote the parameter as $\theta_t = (\theta_t(1), \theta_t(2))^{\top}$. Subsequently, we can calculate the logits as $z_t = (-2 \theta_t(2), \theta_t(2), -\theta_t(1), 2\theta_t(1) )^\top$.
We can then calculate the gradient as: 
\begin{align*}
    \frac{d\,\dpd{\pi_{\theta}, r}}{d \theta} = \phi^\top\,( \text{diag}{(\pi_{\theta})} - \pi_{\theta}\,\pi_{\theta}^\top )\,r = \sum_{a \in [K]} \phi_a\,\pi_t(a)\,(r(a) - \dpd{\pi_t, r}) .
\end{align*}
Consequently, in this instance, the partial gradient with respect to $\theta_t(2)$ can be written as:
\begin{align*}
    \frac{\partial \dpd{\pi_{\theta_t}, r}}{\partial \theta_t(2)} &= -2\,\pi_t(1)\,(9 - \langle \pi_t, r \rangle) + \pi_t(2)\,(8 - \langle \pi_t, r \rangle) .
\end{align*}
Since $\langle \pi_t, r \rangle = 9\,\pi_t(1) + 8\,\pi_t(2) + 7\,\pi_t(3) + 6\,\pi_t(4)$ and $\sum_{a=1}^4 \pi_t(a) = 1$, we have that
\begin{align*}
    9 - \langle \pi_t, r \rangle &= \pi_t(2) + 2\pi_t(3) + 3\pi_t(4) , \\
    8 - \langle \pi_t, r \rangle &= -\pi_t(1) + \pi_t(3) + 2\pi_t(4) .
\end{align*}
Using the above relations to simplify the expression for $\frac{\partial \dpd{\pi_{\theta_t}, r}}{\partial \theta_t(2)}$, if $Z_t \coloneqq \sum_{a=1}^4 \exp(z_t(a))$,
\begin{align*}
    \frac{\partial \dpd{\pi_{\theta_t}, r}}{\partial \theta_t(2)} 
    &= -2\pi_t(1)(\pi_t(2) + 2\pi_t(3) + 3\pi_t(4)) + \pi_t(2)(-\pi_t(1) + \pi_t(3) + 2\pi_t(4)) \\
    &= -3\pi_t(1)\,\pi_t(2) - 4\pi_t(1)\,\pi_t(3) - 6\pi_t(1)\,\pi_t(4) + \pi_t(2)\,\pi_t(3) + 2\pi_t(2)\,\pi_t(4) \\
    &= \frac{-3\exp(z_t(1) + z_t(2))}{Z_t^2} - \frac{4\exp(z_t(1) + z_t(3))}{Z_t^2} - \frac{6\exp(z_t(1) + z_t(4))}{Z_t^2} \\
    &\quad+ \frac{\exp(z_t(2) + z_t(3))}{Z_t^2} + \frac{2\exp(z_t(2) + z_t(4))}{Z_t^2},
\end{align*}
Recall that $z_t = (-2 \theta_t(2), \theta_t(2), -\theta_t(1), 2\theta_t(1) )^\top$. Multiplying each side by $Z_t^2$ yields that,
\begin{align*}
    Z_t^2\,\frac{\partial \dpd{\pi_{\theta_t}, r}}{\partial \theta_t(2)} &= -3e^{-\theta_t(2)} - 4e^{-2\theta_t(2) - \theta_t(1)} - 6e^{-2\theta_t(2) + 2\theta_t(1)} + e^{\theta_t(2) - \theta_t(1)} + 2e^{\theta_t(2) + 2\theta_t(1)} \\
    &= e^{\theta_t(2)}(e^{-\theta_t(1)} + 2e^{2\theta_t(1)}) - 3e^{-\theta_t(2)} - e^{-2\theta_t(2)}(4e^{-\theta_t(1)} + 6e^{2\theta_t(1)}) \\
    &\geq e^{\theta_t(2)}(e^{-\theta_t(1)} + 2e^{2\theta_t(1)}) - 3e^{-\theta_t(2)} - e^{-2\theta_t(2)}(4 e^{-\theta_t(1)} + 8e^{2\theta_t(1)}) \\
    &= (e^{\theta_t(2)} - 4 e^{-2\theta_t(2)})(e^{-\theta_t(1)} + 2e^{2\theta_t(1)}) - 3e^{-\theta_t(2)}
\end{align*}
We then proceed to analyze $e^{-\theta_t(1)} + 2e^{2\theta_t(1)}$. Let $g(x) = \exp(-x) + 2 \exp(2 x)$. Its first derivative is $g^\prime(x) = -\exp(-x) + 4\exp(2x)$. Setting it to $0$ yields that $x = -\frac{\ln(4)}{3}$. By analyzing the second derivative, we know that it is the global minimum. Therefore, $g(x) \geq g\left(-\frac{\ln(4)}{3}\right) = \frac{3}{2^{1/3}} \approx 2.381$. Hence, we have that
\begin{align*}
    Z_t^2 \frac{\partial \dpd{\pi_{\theta}, r}}{\partial \theta_t(2)} \geq 2.381 (e^{\theta_t(2)} - 4 e^{-2\theta_t(2)}) - 3e^{-\theta_t(2)}.
\end{align*}
Note that the RHS monotonically increases with respect to $\theta_t(2)$. Consequently, for $\theta_t(2) > 1$, we know that $Z_t^2 \frac{\partial \dpd{\pi_{\theta}, r}}{\partial \theta_t(2)} > 2.381 \times (e^1 - 4 e^{-2}) - 3 e ^{-1} \approx 4.080 > 0$. Thus, $\frac{\partial \dpd{\pi_{\theta}, r}}{\partial \theta_t(2)} > 0$. Using the $\LinSPG$ update, if $\theta_t(2) > 1$, then,  
\begin{align*}
    \theta_{t+1}(2) = \theta_t(2) + \eta \frac{\partial \langle \pi_\theta, r \rangle}{\partial \theta_t(2)} > \theta_t(2) > 1 .
\end{align*}
Therefore, if we can construct an initialization $\theta_1$ such that $\theta_1(2) > 1$, then, $\theta_t(2)$ will monotonically increase as $t \to \infty$. Consequently, we have that
\begin{align*}
    \lim_{t \to \infty} \frac{\pi_t(1)}{\pi_t(2)} = \lim_{t \to \infty} \frac{\exp(-2\theta_t(2))}{\exp(\theta_t(2))} = \lim_{t \to \infty} \exp(-3\theta_t(2)) < e^{-3} < \infty.
\end{align*}
This contradicts the implication that $\lim_{t \to \infty} \frac{\pi_t(1)}{\pi_t(a)} = \infty$ for all $a \neq 1$ and thus the assumption that $\LinSPG$ converges to the optimal policy under the initialization $\theta_1$. This suggests that there exists an initialization under which $\LinSPG$ cannot achieve global convergence.\\
In particular, in~\cref{fig:approx_error}, we set $\theta_1 = (3, 3)^\top$ (satisfying $\theta_1(2) > 1$) and demonstrate the non-convergence behavior of $\LinSPG$ in~\cref{eg:second_example}.
\end{proof}

\section{Proofs of \texorpdfstring{\cref{sec:deterministic_linear_bandits}}{}}
\label{appendix:deterministic_linear_bandits}

\subsection{Warm-Up: Global Convergence for \texorpdfstring{$K = 3$}{K = 3}}
\label{appendix:three_armed_deterministic_linear_bandits_sufficiency}

\subsubsection{Sufficiency}

\threearmeddeterministiclinearbandits*
\begin{proof}
Under~\cref{assumption:no_identical_arms,assumption:reward_ordering_preservation}, according to~\cref{lemma:monotonicity_and_onehotpolicy}, for all finite $t \geq 1$,
\begin{align*}
\dpd{\pi_{\thtt}, r} > \dpd{\pit, r}, \numberthis \label{eq:monotonicity_det}
\end{align*}
and $\lim_{t \to \infty} \pit(a) = 1$ for some action $a \in \{1, 2, 3\}$.
We will prove $\lim_{t \to \infty} \pit(1) = 1$ by showing that $\lim_{t \to \infty} \pit(2) \neq 1$ and $\lim_{t \to \infty} \pit(3) \neq 1$.

For any bounded initialization $\theta_1$, we have $\dpd{\pi_{\theta_1}, r} > r(3)$. From~\cref{eq:monotonicity_det}, we know that for all finite $t \geq 1$,
\begin{align*}
\dpd{\pit, r} > \dpd{\pi_{\theta_1}, r} > r(3).
\end{align*} 
Therefore, we know that $\lim_{t \to \infty} \pit(3) \neq 1$.

Suppose that $\lim_{t \to \infty} \pit(2) = 1$. Given this assumption and~\cref{eq:monotonicity_det}, we know that for all finite $t \geq 1$, $\dpd{\pit, r} < r(2)$. In this case, we will show that
\begin{align*}
    \lim_{t \to \infty} \frac{\pit(1)}{\pit(3)} = \infty,
\end{align*}
and prove that this implies that for all large enough $t$, $\dpd{\pit, r} > r(2)$.
Hence, this results in a contradiction proving that $\lim_{t \to \infty} \pit(2) \neq 1$. 
To start, we consider the following ratio:
\begin{align*}
   \frac{\pitt(1)}{\pitt(3)}  &= \exp\parens*{[\phi\,\thtt](1) - [\phi\,\thtt](3)}  \\
   &= \exp\parens*{[\phi\,\tht](1) - [\phi\,\tht](3) + \eta\,\parens*{\sum_{i=1}^3 \dpd{\phi_i, \phi_1 - \phi_3}\,\pit(i)\,(r(i) - \dpd{\pit, r})}}  \tag{by the update in~\cref{alg:det_spg}} \\
   &= \frac{\pit(1)}{\pit(3)}\,\exp\parens*{\eta\,\underbrace{\parens*{\sum_{i=1}^3 \dpd{\phi_i, \phi_1 - \phi_3}\,\pit(i)\,(r(i) - \dpd{\pit, r})}}_{\eqcolon P_t}}, \numberthis \label{eq:ratio_for_three_armed_deterministic_bandits}
\end{align*}
and the sign of $P_t$ will dictate whether $\frac{\pi_{\tht}(1)}{\pi_{\tht}(3)}$ will increase or decrease. Then, we will further look into $P_t$. For all finite $t \geq 1$, we have that
\begin{align*}
    P_t &= \sum_{i=1}^{3}{ \dpd{\phi_i, \phi_1 - \phi_3}\,\pi_{\theta_t}(i)\,(r(i) - \dpd{\pi_{\theta_t}, r}) }  \\
    &= \dpd{\phi_1 - \phi_3, \phi_1 - \phi_3}\,\pi_{\theta_t}(1)\,(r(1) - \dpd{\pi_{\theta_t}, r}) + \dpd{\phi_2 - \phi_3, \phi_1 - \phi_3}\,\pi_{\theta_t}(2)\,(r(2) - \dpd{\pi_{\theta_t}, r}) \tag{since $\sum_{i=1}^3 \dpd{\phi_3, \phi_1 - \phi_3}\,\pit(i)\,(r(i) - \dpd{\pit, r}) = 0$}  \\
    & > \dpd{\phi_1 - \phi_3, \phi_1 - \phi_3}\,\pi_{\theta_t}(1)\,(r(1) - r(2)) \tag{under~\cref{assumption:feature_conditions_for_three_armed_linear_bandits}, $\dpd{\phi_2 - \phi_3, \phi_1 - \phi_3} > 0$ and for all finite $t \geq 1$, $r(2) > \dpd{\pit, r}$}  \\
    &= \normsq{\phi_1 - \phi_3}\,\pi_{\theta_t}(1)\,(r(1) - r(2)) > 0 \numberthis \label{eq:P_t_for_three_armed_deterministic_bandits_lower_bound}.
\end{align*}
By recursing~\cref{eq:ratio_for_three_armed_deterministic_bandits}, we get that,  \begin{align*}
\frac{\pit(1)}{\pit(3)} &= \frac{\pi_{\theta_1}(1)}{\pi_{\theta_1}(3)}\,\exp (\eta\,\sum_{s=1}^{t-1} P_s)  \\
&> \frac{\pi_{\theta_1}(1)}{\pi_{\theta_1}(3)}\,\exp(\eta\,\normsq{\phi_1 - \phi_3}\,(r(1) - r(2))\,\sum_{s=1}^{t-1} \pi_{\theta_s}(1))  \tag{by \cref{{eq:P_t_for_three_armed_deterministic_bandits_lower_bound}}}
\end{align*}
Next, we will prove $\sum_{s = 1}^{\infty} { \pi_{\theta_s}(1) } = \infty$. Since $P_t > 0$, $\frac{\pit(1)}{\pit(3)}$ is monotonically increasing. Hence, we have that $\frac{\pitt(3)}{\pitt(1)} < \frac{\pit(3)}{\pit(1)}$ for all finite $t \geq 1$. As a result,
\begin{align*}
    \sum_{s = 1}^{t}{ ( 1 - \pi_{\theta_s}(2) ) } &= \sum_{s = 1}^{t}{ \big( \pi_{\theta_s}(1) + \pi_{\theta_s}(3) \big) }  \\
    &= \sum_{s = 1}^{t}{ \bigg( \pi_{\theta_s}(1) + \pi_{\theta_s}(1)\,\frac{ \pi_{\theta_s}(3) }{ \pi_{\theta_s}(1) }  \bigg) }  \\
    &<  \sum_{s = 1}^{t}{ \bigg( \pi_{\theta_s}(1) + \pi_{\theta_s}(1)\,\frac{ \pi_{\theta_1}(3) }{ \pi_{\theta_1}(1) } \bigg) }  \\
    &= \bigg( 1 + \frac{ \pi_{\theta_1}(3) }{ \pi_{\theta_1}(1) } \bigg)\,\sum_{s = 1}^{t}{ \pi_{\theta_s}(1) } .
\end{align*}
For the LHS,~\cref{lemma:multi_arm_sum_to_infty} shows that $\sum_{s=1}^{\infty} (1 - \pi_{\theta_s}(2)) = \infty$. Therefore, $\sum_{s = 1}^{\infty} { \pi_{\theta_s}(1) } = \infty$. 
Using the equation above, we have that $\lim_{t \to \infty} \frac{\pit(1)}{\pit(3)} = \infty$. Moreover, for large enough $t$,
\begin{align*}
    r(2) - \dpd{\pit, r} &= \pi_{\theta_t}(1)\,( r(2) - r(1)) + \pi_{\theta_t}(3)\,( r(2) - r(3))  \\
    &= \pi_{\theta_t}(3)\,( r(2) - r(3))\,\bigg[ - \underbrace{\frac{ r(1) - r(2) }{ r(2) - r(3) }}_{> 0}\,\underbrace{\frac{ \pi_{\theta_t}(1) }{ \pi_{\theta_t}(3) }}_{\to \infty} + 1 \bigg] \\
    & < 0 . \tag{$t$ is large enough}
\end{align*}

Therefore, we know that $\dpd{\pit, r} > r(2)$ for all large enough $t$. This, combined with~\cref{eq:monotonicity_det}, contradicts our assumption that $\lim_{t \to \infty} \pit(2) = 1$. Putting everything together, we can draw the conclusion that $\lim_{t \to \infty} \pi_{\theta_t}(1) = 1$.
\end{proof}

\subsubsection{Necessity}
\label{appendix:three_armed_deterministic_linear_bandits_necessity}

Given~\cref{assumption:reward_ordering_preservation,assumption:feature_conditions_for_three_armed_linear_bandits}, we next investigate if these assumptions are required for global convergence. 
The following is an ideal example where all assumptions are satisfied.
\begin{example}
    \label{eg:sufficient_feature_conditions}
    Let $K = 3$ $d = 2$, $\phi^\top = \begin{bmatrix}\,0 & -0.3 & 1\,\vspace{0.3ex}\\\,-1 & 0.6 & 0\,\end{bmatrix}$, and $r = (1, 0.5, 0)^\top$. Since $r^{\prime} = \phi\,w = (1, 0, -2)^\top$ with $w = (-2, -1)^\top$,~\cref{assumption:reward_ordering_preservation} is satisfied.~\cref{assumption:feature_conditions_for_three_armed_linear_bandits} is satisfied because $\dpd{\phi_2 - \phi_3, \phi_1 - \phi_3} = 0.7 > 0$.
\end{example}

In the above example,~\cref{alg:det_spg} is guaranteed to converge to the optimal policy for any initialization (as illustrated in~\cref{fig:feature_conditions} (left)). Furthermore, we will prove that \cref{assumption:feature_conditions_for_three_armed_linear_bandits} is a necessary condition for global convergence in 3-armed bandits. By ``necessary'', we do not claim that a violation of this condition guarantees failure of the algorithm in all cases. Rather, we assert that if this condition is omitted while the others are satisfied, it is always possible to construct a specific counterexample on which the algorithm fails to converge. In other words, each condition is essential in the sense that leaving any one of them out allows for the existence of a problem instance that breaks global convergence.

\begin{figure}
    \centering
    \begin{minipage}{0.48\linewidth}
        \includegraphics[width=\linewidth]{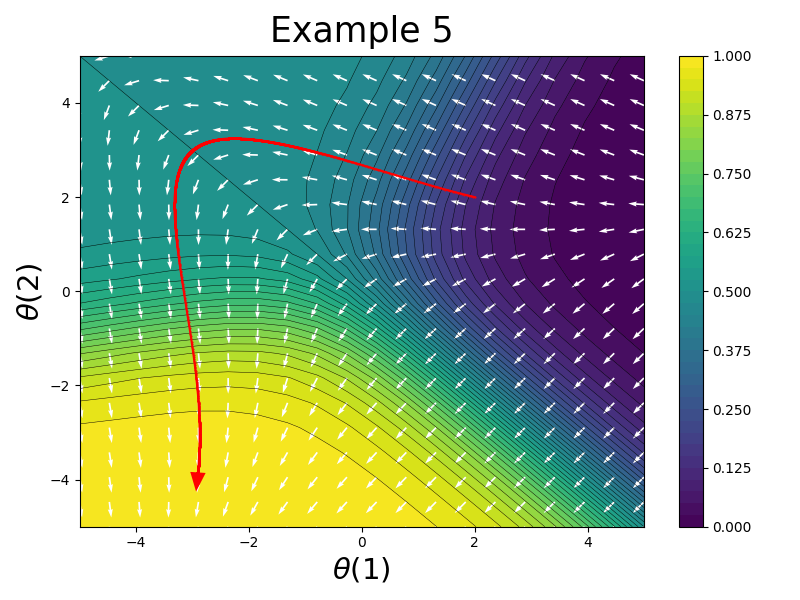} 
    \end{minipage}
    \begin{minipage}{0.48\linewidth}
        \includegraphics[width=\linewidth]{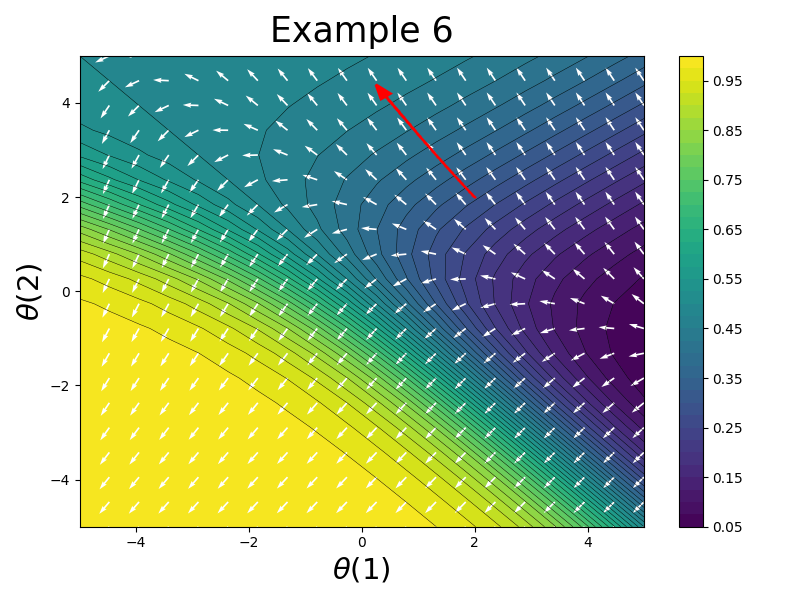}    
    \end{minipage}
    \caption{The effect of feature conditions on the global convergence.}
    \label{fig:feature_conditions}
\end{figure}

We now show that for the three-armed bandit setting, ~\cref{assumption:feature_conditions_for_three_armed_linear_bandits} is necessary for achieving global convergence. Specifically, the following proposition allows the construction of examples where only~\cref{assumption:no_identical_arms,assumption:reward_ordering_preservation} are satisfied while~\cref{alg:det_spg} fails to converge to the optimal policy.

\begin{proposition}
\label{prop:necessity_for_three_armed_feature_conditions}
Given a reward vector $r \in \mathbb{R}^{3}$ and a feature matrix $\phi \in \mathbb{R}^{3 \times d}$ such that~\cref{assumption:no_identical_arms,assumption:reward_ordering_preservation} are satisfied, and $\dpd{\phi_2 - \phi_3, \phi_1 - \phi_3} < 0$ (implying that \cref{assumption:feature_conditions_for_three_armed_linear_bandits} is not satisfied). Let $\theta_1 = C\,(\phi_3 - \phi_1)$, where $C > 0$ is sufficiently large. Then,~\cref{alg:det_spg} with any constant learning rate $\eta > 0$ fails to converge to the optimal policy.
\end{proposition}

\begin{proof}
Since $\dpd{\phi_2 - \phi_3, \phi_1 - \phi_3} < 0$, we have that $\dpd{\phi_3 - \phi_2,\phi_1 - \phi_3} > 0$, implying that
\begin{align*}
    \dpd{\phi_1 - \phi_2, \phi_1 - \phi_3} &= \dpd{\phi_1 - \phi_3, \phi_1 - \phi_3} + \dpd{\phi_3 - \phi_2, \phi_1 - \phi_3} \ge \dpd{\phi_3 - \phi_2, \phi_1 - \phi_3} > 0 \\
    \langle\phi_1-\phi_2,\phi_2-\phi_3\rangle &= \langle\phi_1-\phi_3,\phi_2-\phi_3\rangle-\|\phi_2-\phi_3\|_2^2<0.
\end{align*}
Based on~\cref{alg:det_spg}, we have that
\begin{align*}
    \phi\,\theta_{t+1} &= \phi\,\theta_t + \eta\,\phi\,\phi^\top \left( \text{diag}( \pi_{\theta_t} ) - \pi_{\theta_t}\,\pi_{\theta_t}^\top \right) r .
\end{align*}
To start, we first have that
\begin{align*}
    \frac{ \pi_{\theta_1}(1) }{ \pi_{\theta_1}(3) } = &\exp( [\phi\,\theta_1](1) - [\phi\,\theta_1](3)) \\ 
    =\,&\exp(\dpd{\phi_1 - \phi_3, \theta_1}) \\
    =\,&\exp( - C\,\normsq{\phi_3 - \phi_1}).  \tag{since $\theta_1 = C\,(\phi_3 - \phi_1)$}
\end{align*}
Moreover, we have that
\begin{align*}
\frac{\pi_{\theta_1}(2)}{\pi_{\theta_1}(3)} &= \exp\left([\phi\,\theta_1](2)-[\phi\,\theta_1](3)\right) \\
&= \exp\left(\langle \phi_2-\phi_3,\theta_1\rangle\right) \\
&= \exp\left(-C\langle \phi_2-\phi_3,\phi_1-\phi_3\rangle\right) . \tag{since $\theta_1 = C\,(\phi_3 - \phi_1)$}
\end{align*}
Since $\langle \phi_2-\phi_3,\phi_1-\phi_3\rangle<0$, as $C\to\infty$, $\frac{\pi_{\theta_1}(1)}{\pi_{\theta_1}(3)}\to 0$ and $\frac{\pi_{\theta_1}(2)}{\pi_{\theta_1}(3)}\to\infty$. Therefore, for a sufficiently large $C > 0$, $\frac{\pi_{\theta_1}(1)}{\pi_{\theta_1}(3)}$ can be as small as needed and $\frac{\pi_{\theta_1}(2)}{\pi_{\theta_1}(3)}$ can be as large as needed.

We now prove by induction that, for every finite $t \ge 1$,
\begin{align} \label{eq:ratio_monotonicity}
\frac{\pi_{\theta_t}(1)}{\pi_{\theta_t}(3)} &\le \frac{\pi_{\theta_1}(1)}{\pi_{\theta_1}(3)} ,
\qquad
\frac{\pi_{\theta_t}(2)}{\pi_{\theta_t}(3)} \ge \frac{\pi_{\theta_1}(2)}{\pi_{\theta_1}(3)} .
\end{align}
The base case $t=1$ holds with equality. For the induction step, fix a finite $t \ge 1$ and suppose that~\cref{eq:ratio_monotonicity} holds at iteration $t$. We show that it also holds at iteration $t+1$.\\

\paragraph{Part I:} First, we consider the following ratio:
\begin{align*}
    \frac{ \pi_{\theta_{t+1}}(1) }{ \pi_{\theta_{t+1}}(3) } = \exp( [\phi\,\theta_{t+1}](1) - [\phi\,\theta_{t+1}](3)).
\end{align*}
Using the update of~\cref{alg:det_spg},
\begin{align*}
    [\phi\,\theta_{t+1}](1) - [\phi\,\theta_{t+1}](3) &= [\phi\,\theta_t](1) - [\phi\,\theta_t](3) + \eta\,\sum_{i=1}^{3}{ \dpd{\phi_i, \phi_1 - \phi_3}\,\pi_{\theta_t}(i) \cdot (r(i) - \dpd{\pi_{\theta_t}, r}) }.
\end{align*}
We then analyze the term $\sum_{i=1}^{3}{ \dpd{\phi_i, \phi_1 - \phi_3}\,\pi_{\theta_t}(i) \cdot (r(i) - \dpd{\pi_{\theta_t}, r}) }$.
\begin{align*}
\MoveEqLeft
\sum_{i=1}^{3}\langle \phi_i,\phi_1-\phi_3\rangle\pi_{\theta_t}(i)\bigl(r(i)-\langle\pi_{\theta_t},r\rangle\bigr) = \sum_{i=1}^{3}\langle \phi_i-\phi_3,\phi_1-\phi_3\rangle\pi_{\theta_t}(i)\bigl(r(i)-\langle\pi_{\theta_t},r\rangle\bigr) \\
\tag{since $\sum_{i=1}^{3}\pi_{\theta_t}(i)\bigl(r(i)-\langle \pi_{\theta_t},r\rangle\bigr)=0$} \\
&= \|\phi_1-\phi_3\|_2^2\pi_{\theta_t}(1)\bigl(r(1)-\langle\pi_{\theta_t},r\rangle\bigr)+\langle\phi_2-\phi_3,\phi_1-\phi_3\rangle\pi_{\theta_t}(2)\bigl(r(2)-\langle\pi_{\theta_t},r\rangle\bigr) \\
&= \|\phi_1-\phi_3\|_2^2\pi_{\theta_t}(1)\Bigl[\pi_{\theta_t}(2)\,(r(1)-r(2))+\pi_{\theta_t}(3)\,(r(1)-r(3))\Bigr] \\
&\quad+\langle\phi_2-\phi_3,\phi_1-\phi_3\rangle\pi_{\theta_t}(2)\Bigl[-\pi_{\theta_t}(1)\,(r(1)-r(2))+\pi_{\theta_t}(3)\,(r(2)-r(3))\Bigr] \\
&= \pi_{\theta_t}(1)\,\pi_{\theta_t}(2)\,(r(1)-r(2))\Bigl[\|\phi_1-\phi_3\|_2^2-\langle\phi_2-\phi_3,\phi_1-\phi_3\rangle\Bigr] \\
&\quad+\pi_{\theta_t}(2)\,\pi_{\theta_t}(3)\langle\phi_2-\phi_3,\phi_1-\phi_3\rangle(r(2)-r(3))+\pi_{\theta_t}(1)\,\pi_{\theta_t}(3)\,(r(1)-r(3))\|\phi_1-\phi_3\|_2^2 \\
&= \pi_{\theta_t}(1)\,\pi_{\theta_t}(2)\langle\phi_1-\phi_2,\phi_1-\phi_3\rangle(r(1)-r(2)) +\pi_{\theta_t}(2)\,\pi_{\theta_t}(3)\langle\phi_2-\phi_3,\phi_1-\phi_3\rangle(r(2)-r(3)) \\
&\quad+\pi_{\theta_t}(1)\,\pi_{\theta_t}(3)\,(r(1)-r(3))\|\phi_1-\phi_3\|_2^2 \\
&= \pi_{\theta_t}(3)^2\Biggl[\frac{\pi_{\theta_t}(2)}{\pi_{\theta_t}(3)}\left(\underbrace{\langle\phi_1-\phi_2,\phi_1-\phi_3\rangle}_{>0} (r(1)-r(2))\frac{\pi_{\theta_t}(1)}{\pi_{\theta_t}(3)}+\underbrace{\langle\phi_2-\phi_3,\phi_1-\phi_3\rangle}_{<0} (r(2)-r(3))\right) \\
&\qquad\qquad\qquad+\frac{\pi_{\theta_t}(1)}{\pi_{\theta_t}(3)}(r(1)-r(3))\|\phi_1-\phi_3\|_2^2\Biggr]  \\
&\leq \pi_{\theta_t}(3)^2\Biggl[\frac{\pi_{\theta_t}(2)}{\pi_{\theta_t}(3)}\left( \underbrace{\langle\phi_1-\phi_2,\phi_1-\phi_3\rangle}_{>0} (r(1)-r(2))\frac{\pi_{\theta_1}(1)}{\pi_{\theta_1}(3)}+\underbrace{\langle\phi_2-\phi_3,\phi_1-\phi_3\rangle}_{<0}(r(2)-r(3))\right) \\
&\qquad\qquad\qquad+\frac{\pi_{\theta_1}(1)}{\pi_{\theta_1}(3)}(r(1)-r(3))\|\phi_1-\phi_3\|_2^2\Biggr] \tag{since $\frac{\pi_{\theta_t}(1)}{\pi_{\theta_t}(3)} \leq \frac{\pi_{\theta_1}(1)}{\pi_{\theta_1}(3)}$} 
\end{align*}
Then, we can choose $\pi_{\theta_1}(1) / \pi_{\theta_1}(3)$ to be small enough such that
\begin{align*}
\frac{\pi_{\theta_1}(1)}{\pi_{\theta_1}(3)} \le \frac{-\langle \phi_2-\phi_3,\phi_1-\phi_3\rangle(r(2)-r(3))}{2\langle \phi_1-\phi_2,\phi_1-\phi_3\rangle(r(1)-r(2))}.
\end{align*}
Therefore, we have that
\begin{align*}
    \MoveEqLeft \sum_{i=1}^{3}\langle \phi_i,\phi_1-\phi_3\rangle\pi_{\theta_t}(i)\bigl(r(i)-\langle\pi_{\theta_t},r\rangle\bigr) \\
    &\leq \pi_{\theta_t}(3)^2\left[\frac{1}{2}\frac{\pi_{\theta_t}(2)}{\pi_{\theta_t}(3)}\underbrace{\langle\phi_2-\phi_3,\phi_1-\phi_3\rangle}_{<0} (r(2)-r(3))+\frac{\pi_{\theta_t}(1)}{\pi_{\theta_t}(3)}(r(1)-r(3))\|\phi_1-\phi_3\|_2^2 \right] .
\end{align*}
We can then choose $\pi_{\theta_1}(2) / \pi_{\theta_1}(3)$ to be large enough such that
\begin{align*}
    \frac{\pi_{\theta_t}(2)}{\pi_{\theta_t}(3)}>\frac{2\frac{\pi_{\theta_t}(1)}{\pi_{\theta_t}(3)}(r(1)-r(3))\|\phi_1-\phi_3\|_2^2}{-\langle \phi_2-\phi_3,\phi_1-\phi_3\rangle(r(2)-r(3))}.
\end{align*}
Hence, we have that
\begin{align*}
    \sum_{i=1}^{3}\langle \phi_i,\phi_1-\phi_3\rangle\pi_{\theta_t}(i)\bigl(r(i)-\langle\pi_{\theta_t},r\rangle\bigr) <0 .
\end{align*}
This further implies that
\begin{align*}
    \frac{\pi_{\theta_{t+1}}(1)}{\pi_{\theta_{t+1}}(3)}=\frac{\pi_{\theta_t}(1)}{\pi_{\theta_t}(3)}\exp\left(\eta\,\sum_{i=1}^{3}\langle \phi_i,\phi_1-\phi_3\rangle\pi_{\theta_t}(i)\bigl(r(i)-\langle\pi_{\theta_t},r\rangle\bigr)\right)<\frac{\pi_{\theta_t}(1)}{\pi_{\theta_t}(3)}\le\frac{\pi_{\theta_1}(1)}{\pi_{\theta_1}(3)} .
\end{align*}

\paragraph{Part II:} Next, we consider the following ratio:

\begin{align*}
\frac{\pi_{\theta_{t+1}}(2)}{\pi_{\theta_{t+1}}(3)} &= \exp\left([\phi\,\theta_{t+1}](2)-[\phi\,\theta_{t+1}](3)\right).
\end{align*}
Using the update of~\cref{alg:det_spg},
\begin{align*}
[\phi\,\theta_{t+1}](2)-[\phi\,\theta_{t+1}](3) &= [\phi\,\theta_t](2)-[\phi\,\theta_t](3)+\eta\,\sum_{i=1}^{3}\langle\phi_i,\phi_2-\phi_3\rangle\pi_{\theta_t}(i)\bigl(r(i)-\langle\pi_{\theta_t},r\rangle\bigr).
\end{align*}
We then analyze the term $\sum_{i=1}^{3}\langle\phi_i,\phi_2-\phi_3\rangle\pi_{\theta_t}(i)\bigl(r(i)-\langle\pi_{\theta_t},r\rangle\bigr)$.
\begin{align*}
\MoveEqLeft \sum_{i=1}^{3}\langle\phi_i,\phi_2-\phi_3\rangle\pi_{\theta_t}(i)\bigl(r(i)-\langle\pi_{\theta_t},r\rangle\bigr) \\
&= \sum_{i=1}^{3}\langle\phi_i-\phi_3,\phi_2-\phi_3\rangle\pi_{\theta_t}(i)\bigl(r(i)-\langle\pi_{\theta_t},r\rangle\bigr) \\
\tag{since $\sum_{i=1}^{3}\pi_{\theta_t}(i)\,(r(i)-\langle\pi_{\theta_t},r\rangle)=0$} \\
&= \langle\phi_1-\phi_3,\phi_2-\phi_3\rangle\pi_{\theta_t}(1)\bigl(r(1)-\langle\pi_{\theta_t},r\rangle\bigr)+\|\phi_2-\phi_3\|_2^2\pi_{\theta_t}(2)\bigl(r(2)-\langle\pi_{\theta_t},r\rangle\bigr) \\
&= \langle\phi_1-\phi_3,\phi_2-\phi_3\rangle\pi_{\theta_t}(1)\Bigl[\pi_{\theta_t}(2)\,(r(1)-r(2))+\pi_{\theta_t}(3)\,(r(1)-r(3))\Bigr] \\
&\quad+\|\phi_2-\phi_3\|_2^2\pi_{\theta_t}(2)\Bigl[-\pi_{\theta_t}(1)\,(r(1)-r(2))+\pi_{\theta_t}(3)\,(r(2)-r(3))\Bigr] \\
&= \pi_{\theta_t}(1)\,\pi_{\theta_t}(2)\,(r(1)-r(2))\Bigl[\langle\phi_1-\phi_3,\phi_2-\phi_3\rangle-\|\phi_2-\phi_3\|_2^2\Bigr] \\
&\quad+\pi_{\theta_t}(1)\,\pi_{\theta_t}(3)\langle\phi_1-\phi_3,\phi_2-\phi_3\rangle(r(1)-r(3))+\pi_{\theta_t}(2)\,\pi_{\theta_t}(3)\|\phi_2-\phi_3\|_2^2(r(2)-r(3)) \\
&= \pi_{\theta_t}(1)\,\pi_{\theta_t}(2)\langle\phi_1-\phi_2,\phi_2-\phi_3\rangle(r(1)-r(2)) \\
&\quad+\pi_{\theta_t}(1)\,\pi_{\theta_t}(3)\langle\phi_1-\phi_3,\phi_2-\phi_3\rangle(r(1)-r(3))+\pi_{\theta_t}(2)\,\pi_{\theta_t}(3)\|\phi_2-\phi_3\|_2^2(r(2)-r(3)) \\
&= \pi_{\theta_t}(3)^2\Biggl[\frac{\pi_{\theta_t}(2)}{\pi_{\theta_t}(3)}\left(\|\phi_2-\phi_3\|_2^2(r(2)-r(3))+\underbrace{\langle\phi_1-\phi_2,\phi_2-\phi_3\rangle}_{<0}(r(1)-r(2))\frac{\pi_{\theta_t}(1)}{\pi_{\theta_t}(3)}\right) \\
&\qquad\qquad\qquad+\frac{\pi_{\theta_t}(1)}{\pi_{\theta_t}(3)}\underbrace{\langle\phi_1-\phi_3,\phi_2-\phi_3\rangle}_{<0}(r(1)-r(3))\Biggr].
\end{align*}
Therefore, we can choose $\pi_{\theta_1}(1) / \pi_{\theta_1}(3)$ to be small enough such that
\begin{align*}
    \frac{\pi_{\theta_1}(1)}{\pi_{\theta_1}(3)} \le \frac{\|\phi_2-\phi_3\|_2^2(r(2)-r(3))}{-2\langle\phi_1-\phi_2,\phi_2-\phi_3\rangle(r(1)-r(2))}.
\end{align*}
Therefore, we have that
\begin{align*}
    \MoveEqLeft \sum_{i=1}^{3}\langle\phi_i,\phi_2-\phi_3\rangle\pi_{\theta_t}(i)\bigl(r(i)-\langle\pi_{\theta_t},r\rangle\bigr) \\
    &\ge \pi_{\theta_t}(3)^2\left[\frac{1}{2}\frac{\pi_{\theta_t}(2)}{\pi_{\theta_t}(3)}\|\phi_2-\phi_3\|_2^2(r(2)-r(3))+\frac{\pi_{\theta_t}(1)}{\pi_{\theta_t}(3)}\underbrace{\langle\phi_1-\phi_3,\phi_2-\phi_3\rangle}_{<0}(r(1)-r(3))\right] \\
    &\ge \pi_{\theta_t}(3)^2\left[\frac{1}{2}\frac{\pi_{\theta_1}(2)}{\pi_{\theta_1}(3)}\|\phi_2-\phi_3\|_2^2(r(2)-r(3))+\frac{\pi_{\theta_1}(1)}{\pi_{\theta_1}(3)}\langle\phi_1-\phi_3,\phi_2-\phi_3\rangle(r(1)-r(3))\right].
\end{align*}
We can then choose $\pi_{\theta_1}(2) / \pi_{\theta_1}(3)$ to be large enough such that
\begin{align*}
    \frac{\pi_{\theta_1}(2)}{\pi_{\theta_1}(3)} &> \frac{2\frac{\pi_{\theta_1}(1)}{\pi_{\theta_1}(3)}\bigl(-\langle\phi_1-\phi_3,\phi_2-\phi_3\rangle\bigr)(r(1)-r(3))}{\|\phi_2-\phi_3\|_2^2(r(2)-r(3))}.
\end{align*}
Hence, we have that
\begin{align*}
\sum_{i=1}^{3}\langle\phi_i,\phi_2-\phi_3\rangle\pi_{\theta_t}(i)\bigl(r(i)-\langle\pi_{\theta_t},r\rangle\bigr) &>0.
\end{align*}
This further implies that
\begin{align*}
\frac{\pi_{\theta_{t+1}}(2)}{\pi_{\theta_{t+1}}(3)} &= \frac{\pi_{\theta_t}(2)}{\pi_{\theta_t}(3)}\exp\left(\eta\,\sum_{i=1}^{3}\langle\phi_i,\phi_2-\phi_3\rangle\pi_{\theta_t}(i)\bigl(r(i)-\langle\pi_{\theta_t},r\rangle\bigr)\right)>\frac{\pi_{\theta_t}(2)}{\pi_{\theta_t}(3)}\ge\frac{\pi_{\theta_1}(2)}{\pi_{\theta_1}(3)}.
\end{align*}

Combining Part I and II completes the induction. Consequently,~\cref{eq:ratio_monotonicity} holds for every iteration $t \geq 1$.
In particular, $\pi_{\theta_t}(1)/\pi_{\theta_t}(3)$ remains bounded for all $t$. If $\pi_{\theta_t}(1)\to 1$, then $\pi_{\theta_t}(3)\to 0$. This would imply $\pi_{\theta_t}(1)/\pi_{\theta_t}(3)\to\infty$, leading to a contradiction. Therefore, we finally have that $\lim_{t\to\infty}\pi_{\theta_t}(1)\neq 1$, and thus~\cref{alg:det_spg} with any constant learning rate $\eta > 0$ fails to converge to the optimal policy.
\end{proof}

We can then instantiate \cref{prop:necessity_for_three_armed_feature_conditions} to a concrete example which is only slightly different from~\cref{eg:sufficient_feature_conditions}.

\begin{example}
\label{eg:necesity_feature_conditions}
Let $K = 3$, $d = 2$, $\phi^\top = \begin{bmatrix}\,0 & 0.6 & 1\,\vspace{0.3ex} \\\,-1 & 0.6 & 0\,\end{bmatrix}$, and $r = (1, 0.5, 0)^\top$. Since $r^{\prime} = \phi\,w = (1, -1.8, -2)^\top$ with $w = (-2, -1)^\top$,~\cref{assumption:reward_ordering_preservation} is satisfied.~\cref{assumption:feature_conditions_for_three_armed_linear_bandits} is not satisfied because $\dpd{\phi_2 - \phi_3, \phi_1 - \phi_3} = - 0.2 < 0$.
\end{example}

In~\cref{eg:necesity_feature_conditions}, we can choose a sufficiently large $C>0$ and set $\theta_1=C(\phi_3-\phi_1)$. By~\cref{prop:necessity_for_three_armed_feature_conditions}, such an initialization satisfies the conditions required for non-convergence, thereby demonstrating that Softmax PG must fail in this specific case (as illustrated in~\cref{fig:feature_conditions} (right)). On the other hand, we can construct another example to show that \cref{assumption:reward_ordering_preservation} is still required, even if \cref{assumption:feature_conditions_for_three_armed_linear_bandits} is satisfied, thus reinforcing that each of these assumptions is independently necessary.

\begin{proposition}
\label{prop:necessity_for_reward_ordering}
Let $K = 3$, $d = 2$, $\phi^\top = \begin{bmatrix}\,3 & 5 & 1\,\vspace{0.3ex}\\\,4 & 6 & 2\,\end{bmatrix}$, and $r = \left(3, 2, 1\right)^\top $. In this case,~\cref{assumption:no_identical_arms,assumption:feature_conditions_for_three_armed_linear_bandits} are satisfied, but~\cref{assumption:reward_ordering_preservation} is not, and the features do not allow any sequence of parameters to induce policies converging to the optimal policy. Therefore,~\cref{alg:det_spg} does not achieve global convergence for any initialization.
\end{proposition}

\begin{proof}
We first verify that the features satisfy~\cref{assumption:feature_conditions_for_three_armed_linear_bandits} as $\langle \phi_2 - \phi_3, \phi_1 - \phi_3 \rangle = 16 > 0$.

Next, we evaluate~\cref{assumption:reward_ordering_preservation}. Suppose a weight vector $w = (w(1), w(2))^{\top}$ exists, and its projection $r^{\prime} = \phi\,w$ preserves the relative reward ordering. To maintain this order, this first requires $r^{\prime}(1) > r^{\prime}(2)$. This inequality yields that $\langle \phi_1, w \rangle > \langle \phi_2, w \rangle$. Substituting the features yields $3w(1) + 4w(2) > 5w(1) + 6w(2)$. This simplifies to $w(1) + w(2) < 0$. On the other hand, the reward order preservation also requires $r^{\prime}(1) > r^{\prime}(3)$. This inequality also yields that $\langle \phi_1, w \rangle > \langle \phi_3, w \rangle$. Substituting the features yields $3w(1) + 4w(2) > w(1) + 2w(2)$. This simplifies to $w(1) + w(2) > 0$. This results in a contradiction. Thus, no such vector $w$ can exist to preserve the reward order, and the features do not satisfy~\cref{assumption:reward_ordering_preservation}.

Furthermore, we show that~\cref{alg:det_spg} can never converge to the optimal policy. Suppose, for contradiction, that there exists a sequence of parameters $\{\theta_n\}_{n\geq1}$ such that $\lim_{n\to\infty}\pi_{\theta_n}(1)=1$. By the softmax parametrization, this requires $\lim_{n\to\infty}\dpd{\phi_1-\phi_2,\theta_n}=\infty$ and $\lim_{n\to\infty}\dpd{\phi_1-\phi_3,\theta_n}=\infty$. Substituting the features gives $\dpd{\phi_1-\phi_2,\theta_n}=-2(\theta_n(1)+\theta_n(2))$ and $\dpd{\phi_1-\phi_3,\theta_n}=2(\theta_n(1)+\theta_n(2))$. Hence, the first limit requires $\lim_{n\to\infty}(\theta_n(1)+\theta_n(2))=-\infty$, whereas the second requires $\lim_{n\to\infty}(\theta_n(1)+\theta_n(2))=\infty$, which is impossible. Therefore, no sequence of parameters can induce policies converging to the optimal policy, and~\cref{alg:det_spg} cannot achieve global convergence from any initialization.
\end{proof}

\subsection{Guarantee of Global Convergence for \texorpdfstring{$K \geq 3$}{K ≥ 3}}
\label{appendix:elb_general}

\deterministiclinearbandits*

\begin{proof}
Under~\cref{assumption:reward_ordering_preservation}, according to~\cref{lemma:monotonicity_and_onehotpolicy}, we know that for all finite $t \geq 1$, \begin{align}
\dpd{\pi_{\thtt}, r} > \dpd{\pit, r}, \label{eq:monotonicity_det_general}
\end{align}
and $\lim_{t \to \infty} \pit(a) = 1$ for some action $a \in [K]$.
For any bounded initialization $\theta_1$, we have $\dpd{\pi_{\theta_1}, r} > r(K)$. The above two inequalities imply that $\lim_{t \to \infty} \pit(K) \neq 1$.

Next, we show that $\lim_{t \to \infty} \dpd{\pit, r} \neq r(a)$ for any $a \in \{2, 3, \dots, K-1\}$. We will prove this by contradiction. For this, in the subsequent proof, we assume that $\lim_{t \to \infty} \dpd{\pit, r} = r(a)$ for some $a \in \{2, 3, \dots, K-1\}$.  Therefore, there exists a large enough finite $\tau$ such that for all finite $t \geq \tau$, $r(a) > \dpd{\pi_{\theta_t}, r} > r(a + 1)$.

We will first prove that $\lim_{t \to \infty} \frac{\pit(1)}{\pit(k)} = \infty$ for all $k \in [a+1, K]$. Considering a fixed action $k \in [a+1, K]$, we have that for all finite $t \geq \tau$,
\begin{align*}
   \frac{\pitt(1)}{\pitt(k)}  &= \exp\parens*{[\phi\,\thtt](1) - [\phi\,\thtt](k)}  \\
   &= \exp\parens*{[\phi\,\tht](1) - [\phi\,\tht](k) + \eta\,\parens*{\sum_{i=1}^K \dpd{\phi_i, \phi_1 - \phi_k}\,\pit(i)\,(r(i) - \dpd{\pit, r})}}  \tag{by the update in~\cref{alg:det_spg}} \\
   &= \frac{\pit(1)}{\pit(k)}\,\exp\Biggl\{\eta\,\underbrace{\parens*{\sum_{i=1}^K \dpd{\phi_i, \phi_1 - \phi_k}\,\pit(i)\,(r(i) - \dpd{\pit, r})}}_{\eqcolon P_t}\Biggr\}, \numberthis \label{eq:pi_1_pi_k+1_ratio}
\end{align*}
and the sign of $P_t$ will dictate whether $\frac{\pit(1)}{\pit(k)}$ will increase or decrease.

Next, to examine the sign of $P_t$, we have that for all finite $t \geq \tau$,
\begin{align*}
    P_t &= \sum_{i=1}^{K}{ \dpd{\phi_i, \phi_1 - \phi_k}\,\pi_{\theta_t}(i)\,(r(i) - \dpd{\pi_{\theta_t}, r}) }  \\
    &= \sum_{\substack{i=1 \\ i \neq a}}^{K}{ \dpd{\phi_i - \phi_a, \phi_1 - \phi_k}\,\pi_{\theta_t}(i)\,(r(i) - \dpd{\pi_{\theta_t}, r}) } \tag{since $\sum_{i=1}^K \dpd{\phi_a, \phi_1 - \phi_k}\,\pit(i)\,(r(i) - \dpd{\pit, r}) = 0$}  \\
    &= \sum_{i=1}^{a-1} { \underbrace{\dpd{\phi_i - \phi_a, \phi_1 - \phi_k}}_{\substack{\geq 0 \text{ due to~\cref{assumption:general_feature_conditions} } \\ (\text{since } i < a \text{ and } k \ge a+1 > a)}}}\,\pi_{\theta_t}(i)\,\underbrace{(r(i) - \dpd{\pi_{\theta_t}, r})}_{0 \text{ (since } i < a)} \\
    & \quad + \sum_{i=a+1}^{K} { \underbrace{\dpd{\phi_a - \phi_i, \phi_1 - \phi_k}}_{\substack{\geq 0 \text{ due to~\cref{assumption:general_feature_conditions} } \\ (\text{since } i > a \text{ and } k \ge a+1 > a)}}}\,\pi_{\theta_t}(i)\,\underbrace{(\dpd{\pi_{\theta_t}, r} - r(i))}_{>0 \text{ (since } i > a)}  \\
    &> \underbrace{\dpd{\phi_1 - \phi_a, \phi_1 - \phi_k} }_{> 0 \text{ due to~\cref{assumption:general_feature_conditions}}}\,\,\pi_{\theta_t}(1)\,(r(1) - r(a)) \eqcolon C\,\pi_{\theta_t}(1) > 0 , \numberthis \label{eq:lower_bound_P_t} \\ \tag{since $r(a) > \dpd{\pi_{\theta_t},r}$ and $\dpd{\pi_{\theta_t}, r} \geq \dpd{\pi_{\theta_\tau}, r}$ for all finite $t \geq \tau$}
\end{align*}
where $C \coloneq \dpd{\phi_1 - \phi_a, \phi_1 - \phi_k}\,(r(1) - r(a)) > 0$.

In particular, we have that $P_t>0$, and therefore~~\cref{eq:pi_1_pi_k+1_ratio} implies that $\pi_{\theta_t}(1)/\pi_{\theta_t}(k)$ is strictly increasing for $t\geq \tau$. Hence, for all finite $t\geq \tau$, $\frac{\pi_{\theta_t}(k)}{\pi_{\theta_t}(1)} \leq \frac{\pi_{\theta_\tau}(k)}{\pi_{\theta_\tau}(1)}$. Define that $D \coloneq \sum_{k=a+1}^{K}\frac{\pi_{\theta_\tau}(k)}{\pi_{\theta_\tau}(1)}<\infty$. Then, for every finite $t\geq \tau$,
\begin{align}
    \sum_{k=a+1}^{K}\pi_{\theta_t}(k) \leq D\,\pi_{\theta_t}(1) . \label{eq:sum_pi_k_leq_D_pi_1}
\end{align}
We next show that $\sum_{t=1}^\infty \pit(1) = \infty$. Since $\langle \pi_{\theta_t},r\rangle<r(a)$ for all $t\geq \tau$, we have that
\begin{align*}
0 &<r(a)-\langle \pi_{\theta_t},r\rangle = -\sum_{i=1}^{a-1}\pi_{\theta_t}(i)\bigl(r(i)-r(a)\bigr) +\sum_{i=a+1}^{K}\pi_{\theta_t}(i)\bigl(r(a)-r(i)\bigr).
\end{align*}
Thus, we have that
\begin{align*}
    \sum_{i=1}^{a-1}\pi_{\theta_t}(i)\bigl(r(i)-r(a)\bigr) < \sum_{i=a+1}^{K}\pi_{\theta_t}(i)\bigl(r(a)-r(i)\bigr) .
\end{align*}
Using the reward ordering and~\cref{eq:sum_pi_k_leq_D_pi_1}, we obtain that
\begin{align*}
\bigl(r(a-1)-r(a)\bigr)\sum_{i=1}^{a-1}\pi_{\theta_t}(i) &< \bigl(r(a)-r(K)\bigr)\sum_{i=a+1}^{K}\pi_{\theta_t}(i) \\
& \leq \bigl(r(a)-r(K)\bigr)\,D\,\pi_{\theta_t}(1) .
\end{align*}
Therefore, we have that
\begin{align*}
1-\pi_{\theta_t}(a) &= \sum_{i=1}^{a-1}\pi_{\theta_t}(i) +
\sum_{i=a+1}^{K}\pi_{\theta_t}(i) \leq \underbrace{D\left(1+\frac{r(a)-r(K)}{r(a-1)-r(a)}\right)}_{\eqcolon D^\prime}\pi_{\theta_t}(1).
\end{align*}
Note that~\cref{lemma:multi_arm_sum_to_infty} shows that for any $a \in [K]$, $\sum_{s=1}^{\infty} (1 - \pi_{\theta_s} (a)) = \infty$. Using the above inequality, we have that $\sum_{t=1}^\infty \pit(1) \geq (1/D^\prime) \sum_{s=1}^{\infty} (1 - \pi_{\theta_s} (a)) = \infty$. By recursing~\cref{eq:pi_1_pi_k+1_ratio}, we get that, for all finite $t \geq \tau$, 
\begin{align*}
\frac{\pi_{\theta_{t}}(1)}{\pi_{\theta_{t}}(k)} = \frac{\pi_{\theta_\tau}(1)}{\pi_{\theta_\tau}(k)}\,\exp (\eta\,\sum_{s=\tau}^{t -1} P_s)  > \frac{\pi_{\theta_\tau}(1)}{\pi_{\theta_\tau}(k)}\,\exp\left(\eta\,C\,\sum_{s=\tau}^{t -1} \pi_{\theta_s}(1) \right) . \tag{by~\cref{eq:lower_bound_P_t}} 
\end{align*}
Combining the above equations, we conclude that $\lim_{t \to \infty} \frac{\pit(1)}{\pit(k)} = \infty$ and hence $\lim_{t \to \infty} \frac{\pit(k)}{\pit(1)} = 0$ for all $k \in [a+1, K]$. As a result, there exists a large enough $\tau^\prime \geq \tau$ such that
\begin{align*}
\MoveEqLeft
r(a) - \dpd{\pi_{\theta_{\tau^\prime}}, r}  \\
=\,& \sum_{i=1}^{K} \pi_{\theta_{\tau^\prime}}(i)\,(r(a) - r(i)) =  \sum_{i=1}^{a-1} \pi_{\theta_{\tau^\prime}}(i)\,\underbrace{(r(a) - r(i))}_{<0} + \sum_{i=a+1}^{K} \pi_{\theta_{\tau^\prime}}(i)\,\underbrace{(r(a) - r(i))}_{>0}  \\
<\,& \pi_{\theta_{\tau^\prime}}(1)\,(r(a) - r(1)) + \sum_{i=a+1}^{K} \pi_{\theta_{\tau^\prime}}(i)\,(r(a) - r(i))  \\
=\,& \pi_{\theta_{\tau^\prime}}(1)\,(r(1) - r(a)) \Biggl[ \sum_{i=a+1}^{K} \underbrace{\frac{\pi_{\theta_{\tau^\prime}}(i)}{\pi_{\theta_{\tau^\prime}}(1)} }_{\to 0}\,\underbrace{\frac{r(a) - r(i)}{r(1) - r(a)}}_{>0} - 1 \Biggr]\\
<\,& 0.  \tag{$\tau^\prime$ is large enough} 
\end{align*}

Therefore, we know that $\dpd{\pi_{\theta_{\tau^\prime}}, r} > r(a)$. Combined with~\cref{eq:monotonicity_det_general}, we know that for all $t \geq \tau^\prime$, $\dpd{\pi_{\theta_{t}}, r} > r(a)$. This contradicts the assumption that $\lim_{t \to \infty}  \dpd{\pit, r} = r(a)$. This implies that $\lim_{t \to \infty}  \dpd{\pit, r} \neq r(a)$ for all $a \in \{2, 3, \cdots, K\}$, and hence the only possible scenario left is $\lim_{t \to \infty} \dpd{\pit, r} = r(1)$, which completes the proof. 

\end{proof}

\subsection{Additional Properties}
\label{appendix:deterministic_linear_bandit_additional_properties}
\monotonicityandonehotpolicy*
\begin{proof}
According to~\cref{lemma:smoothness_expected_reward_log_linear_policy}, we have that for all $t \ge 1$,
\begin{align*}
    \bigg| \dpd{\pi_{\theta_{t+1}} - \pi_{\theta_t}, r} - \Big\langle \frac{d\,\dpd{\pit, r}}{d \theta_t}, \theta_{t+1} - \theta_t \Big\rangle \bigg| \le \frac{9}{4}\,R_{\max}\,\lambda_{\max}(\phi\,\phi^\top)\,\| \theta_{t+1} - \theta_t \|_2^2 .
\end{align*}
This implies that
\begin{align*}
\dpd{\pi_{\theta_{t+1}}, r} - \dpd{\pi_{\theta_t}, r} &\ge \Big\langle \frac{d\,\dpd{\pit, r}}{d \theta_t}, \theta_{t+1} - \theta_t \Big\rangle -  \frac{9}{4}\,R_{\max}\,\lambda_{\max}(\phi\,\phi^\top)\,\| \theta_{t+1} - \theta_t \|_2^2 \\
&= \Big( \eta  - \eta^2\,\frac{9}{4}\,R_{\max}\,\lambda_{\max}(\phi\,\phi^\top) \Big)\,\bigg\| \frac{d\,\dpd{\pit, r}}{d \theta_t} \bigg\|_2^2 \tag{by the update in~\cref{alg:det_spg}}.
\end{align*}
We consider a constant learning rate in the following range:
\begin{align*}
    0< \eta < \frac{4}{ 9\,R_{\max}\,\lambda_{\max}(\phi\,\phi^\top)}.
\end{align*}
Then, we have that
\begin{align*}
    \dpd{\pitt, r} - \dpd{\pit, r} \ge \eta\,\Big( 1  - \eta\,\frac{9\,R_{\max}\,\lambda_{\max}(\phi\,\phi^\top)}{4}  \Big)\,\bigg\| \frac{d\,\dpd{\pit, r}}{d \theta_t} \bigg\|_2^2 \ge 0.
\end{align*}
Note that $\dpd{\pit, r} \le r(a^\star) < \infty$. Since $\dpd{\pit, r}$ is non-decreasing and bounded above by $r(a^\star)$, it converges to a finite limit (i.e., $\lim_{t \to \infty} \dpd{\pit, r} \leq r(a^\star)$). Using the above inequality, we know that
\begin{align*}
    \lim_{t \to \infty}{ \bigg\| \frac{d\,\dpd{\pit, r}}{d \theta_t} \bigg\|_2^2 } = 0. \numberthis \label{eq:gradient_equals_zero_at_infinity}
\end{align*}
Next, we prove that there is no stationary point in the finite region by contradiction. Suppose there exists $\theta^\prime \in \sR^d$ ($\| \theta^\prime \|_2 < \infty$) such that
\begin{align*}
    \frac{d\,\dpd{\pi_{\theta^\prime}, r} }{d \theta^\prime} = \phi^\top \left( \mathrm{diag}(\pi_{\theta^\prime}) - \pi_{\theta^\prime}\,\pi_{\theta^\prime}^\top \right)\,r = \bm{0}.
\end{align*}
Suppose $r^\prime \coloneqq \phi\,w$. Taking the inner product with $w$ on both sides of the above equation,
\begin{align*}
    w^\top \phi^\top \left( \mathrm{diag}(\pi_{\theta^\prime}) - \pi_{\theta^\prime}\,\pi_{\theta^\prime}^\top \right)\,r = {r^\prime}^\top \left( \mathrm{diag}(\pi_{\theta^\prime}) - \pi_{\theta^\prime}\,\pi_{\theta^\prime}^\top \right)\,r =  w^\top \bm{0} = 0.
    \numberthis \label{eq:r_prime_H_r_=_0}
\end{align*}
Since $\| \theta^\prime \|_2 < \infty$ and $\phi$ is bounded ($\max_{i \in [K],\,j \in [d]}{ |\phi_{i,j}|} \le C$ for some $C < \infty$), we have that
\begin{align*}
    \forall i \in [K],\,\pi_{\theta^{\prime}}(i) = \frac{ \exp\parens*{[\phi\,\theta^\prime](i)} }{ \sum_{j \in [K]}{ \exp\parens*{[\phi\,\theta^\prime](j)} } } > 0 .
\end{align*}
Next, according to~\cref{lemma:alternative_expression_covariance}, we have that
\begin{align*}
    {r^\prime}^\top \left( \text{diag}(\pi_{\theta^\prime}) - \pi_{\theta^\prime}\,\pi_{\theta^\prime}^\top \right)\,r =  \sum_{i=1}^{K-1}{ \pi_{\theta^\prime}(i)\,\sum_{j = i+1}^{K}{\pi_{\theta^\prime}(j)\,\left( r^\prime(i) - r^\prime(j) \right)\,\left( r(i) - r(j) \right) }  }. \numberthis \label{eq:r_prime_H_r}
\end{align*}

Consider a non-trivial reward vector, i.e., $r \neq c\,\bm{1}$ for any $c \in \sR$. Under~\cref{assumption:reward_ordering_preservation}, there exists $r^\prime \in \sR^K$ that preserves the order of $r \in \sR^K$, i.e., for all $i, j \in [K]$, $r(i) > r(j)$ if and only if $r^\prime(i) > r^\prime(j)$. This implies that for all $i, j \in [K]$, $\left( r^\prime(i) - r^\prime(j) \right)\,\left( r(i) - r(j) \right) \geq 0$. On the other hand, since $r \neq c\,\bm{1}$, there exists at least one pair of $i \neq j$, such that, $ \left( r^\prime(i) - r^\prime(j) \right)\,\left( r(i) - r(j) \right) > 0$ . Therefore, we can conclude that
\begin{align*}
    {r^\prime}^\top \left( \text{diag}(\pi_{\theta^\prime}) - \pi_{\theta^\prime}\,\pi_{\theta^\prime}^\top \right)\,r > 0 .
\end{align*}
This leads to a contradiction with \cref{eq:r_prime_H_r_=_0}. Therefore, for any $\theta^\prime \in \sR^d$ in the finite region ($\| \theta^\prime \|_2 < \infty$), $\theta^\prime$ is not a stationary point.

Next, we show that $\lim_{t\to\infty}\|\theta_t\|_2=\infty$ also by contradiction. Suppose that $\lim_{t\to\infty}\|\theta_t\|_2\neq\infty$. Then, there exist a constant $C<\infty$ and an infinite subsequence $\{t_n\}_{n\geq1}$ such that $\theta_{t_n}\in S_C\coloneq\{\theta\in\sR^d:\|\theta\|_2\leq C\}$ for all $n\geq1$. From the above arguments, for all $\theta\in S_C$, $\big\|\frac{d\,\dpd{\pi_\theta,r}}{d\theta}\big\|_2>0$. Since $\big\|\frac{d\,\dpd{\pi_\theta,r}}{d\theta}\big\|_2$ is continuous in $\theta$ and $S_C$ is compact, there exists $\varepsilon>0$ such that
\begin{align*}
\inf_{\theta\in S_C}\bigg\|\frac{d\,\dpd{\pi_\theta,r}}{d\theta}\bigg\|_2\geq\varepsilon>0.
\end{align*}
Therefore, $\big\|\frac{d\,\dpd{\pi_{\theta_{t_n}},r}}{d\theta_{t_n}}\big\|_2\geq\varepsilon$ for all $n\geq1$, contradicting~\cref{eq:gradient_equals_zero_at_infinity}. Hence, $\lim_{t\to\infty}\|\theta_t\|_2=\infty$.

Finally, we show that $\lim_{t\to\infty}\pi_{\theta_t}(a)=1$ for some action $a\in[K]$. Recall that $r'=\phi\,w$ preserves the ordering of $r$ and $\|w\|_2=1$. Under~\cref{assumption:no_identical_arms,assumption:reward_ordering_preservation}, we can define that
\begin{align*}
\delta\coloneq\min_{1\leq i<j\leq K}(r'(i)-r'(j))(r(i)-r(j))>0.
\end{align*}
Using~\cref{lemma:alternative_expression_covariance}, we have that
\begin{align*}
\bigg\|\frac{d\,\dpd{\pit,r}}{d\theta_t}\bigg\|_2&\geq\dpd{w,\frac{d\,\dpd{\pit,r}}{d\theta_t}}=\sum_{i=1}^{K-1}\sum_{j=i+1}^{K}\pit(i)\,\pit(j)(r'(i)-r'(j))(r(i)-r(j))\\
&\geq\delta\sum_{i=1}^{K-1}\sum_{j=i+1}^{K}\pit(i)\,\pit(j)=\frac{\delta}{2}\bigg(1-\sum_{i=1}^{K}\pit(i)^2\bigg).
\end{align*}
Combining this with~\cref{eq:gradient_equals_zero_at_infinity} implies that $\lim_{t\to\infty}\bigg(1-\sum_{i=1}^{K}\pit(i)^2\bigg)=0$. Let $M_t\coloneq\max_{i\in[K]}\pit(i)$. Since $\sum_{i=1}^{K}\pit(i)^2\leq M_t$, it follows that $\lim_{t\to\infty}M_t=1$. For each $t$, choose $a_t\in\argmax_{i\in[K]}\pit(i)$. Then, $\left|\dpd{\pit,r}-r(a_t)\right|\leq2R_{\max}(1-M_t)$, and hence $\lim_{t\to\infty}|\dpd{\pit,r}-r(a_t)|=0$. Since $\dpd{\pit,r}$ converges and $\{r(i)\}_{i\in[K]}$ is a finite set of distinct values, $a_t$ must be constant for all sufficiently large $t$. Denoting this action by $a$, we conclude that $\lim_{t\to\infty}\pit(a)=1$.
\end{proof}

\begin{lemma}
\label{lemma:alternative_expression_covariance}
Given any vectors $x \in \sR^K$, $y \in \sR^K$, we have that for all policy $\pi \in \Delta(K)$,
\begin{align*}
    \dpd{x, \left( \mathrm{diag}(\pi) - \pi \pi^\top \right) y} = \sum_{i=1}^{K-1}{ \pi(i)\,\sum_{j = i+1}^{K}{\pi(j)\,\left( x(i) - x(j) \right)\,\left( y(i) - y(j) \right) } }.
\end{align*}
\end{lemma}
\begin{proof}
\begin{align*}
\MoveEqLeft
    \dpd{x, \left( \mathrm{diag}(\pi) - \pi \pi^\top \right) y} = \sum_{i=1}^{K}{  \pi(i)\,x(i)\,y(i)  } - \sum_{i=1}^{K}{ \pi(i)\,y(i) }\,\sum_{j=1}^{K}{ \pi(j)\,x(j) } \\
    &= \sum_{i=1}^{K}{  \pi(i)\,x(i)\,y(i)  } - \sum_{i=1}^{K}{ \pi(i)^2\,x(i)\,y(i)  } - \sum_{i=1}^{K}{ \pi(i)\,y(i) }\,\sum_{j \neq i}{ \pi(j)\,x(j) } \\
    &= \sum_{i=1}^{K}{  \pi(i)\,x(i)\,y(i)\,\left( 1 - \pi(i) \right) } - \sum_{i=1}^{K}{ \pi(i)\,y(i) }\,\sum_{j \neq i}{ \pi(j)\,x(j) } \\
    &= \sum_{i=1}^{K}{  \pi(i)\,x(i)\,y(i)  }\,\sum_{j \neq i}{ \pi(j) } - \sum_{i=1}^{K}{ \pi(i)\,y(i) }\,\sum_{j \neq i}{ \pi(j)\,x(j) } \\
    &= \sum_{i=1}^{K-1}{ \pi(i)\,\sum_{j = i+1}^{K}{\pi(j)\,\left( x(i)\,y(i) + x(j)\,y(j) \right) }  } - \sum_{i=1}^{K-1}{ \pi(i)\,\sum_{j = i+1}^{K}{\pi(j)\,\left( x(j)\,y(i) + x(i)\,y(j) \right)}  } \\
    &= \sum_{i=1}^{K-1}{ \pi(i)\,\sum_{j = i+1}^{K}{\pi(j)\,\left( x(i) - x(j) \right)\,\left( y(i) - y(j) \right) }  }.
\end{align*}
\end{proof}

\begin{lemma}
\label{lemma:multi_arm_sum_to_infty} 
Given a reward vector $r \in \mathbb{R}^{K}$ and a feature matrix $\phi \in \mathbb{R}^{K \times d}$ such that $d \leq K$ and \cref{assumption:no_identical_arms,assumption:reward_ordering_preservation} are satisfied, \cref{alg:det_spg} guarantees that for all $a \in [K]$
\begin{align*}
    \sum_{t=1}^\infty (1 - \pit(a)) = \infty .
\end{align*}
\end{lemma}

\begin{proof}
We prove this by contradiction. Under~\cref{assumption:reward_ordering_preservation}, according to~\cref{lemma:monotonicity_and_onehotpolicy}, we have that $\lim_{t \to \infty} \pit(a) = 1$ for some action $a \in [K]$. For a fixed $a \in[K]$, suppose $\sum_{t\ge 1}{ ( 1 - \pi_{\theta_t}(a) ) } < \infty$. Then, for all $a^\prime \in [K]$, we have that
\begin{align*}
\MoveEqLeft
    \abs{[\phi\,\theta_{t+1}](a^\prime) - [\phi\,\theta_t](a^\prime)}  \\
    &= \eta\,\abs{ \sum_{i=1}^{K}{ \dpd{\phi_{a^\prime}, \phi_i}\,\pi_{\theta_t}(i)\,(r(i) - \dpd{\pit, r}) } }  \tag{by the update in~\cref{alg:det_spg}} \\
    &\leq C\,\sum_{i=1}^{K}{ \pi_{\theta_t}(i)\,\bigg| (r(i) - \dpd{\pit, r}) } \bigg| \tag{setting $C \coloneq \eta\,\max_{i \in [K]} \abs{\dpd{\phi_{a^\prime}, \phi_i}} \geq 0$ and using triangle inequality}  \\
    &\leq C \left[  \sum_{\substack{i=1 \\ i \neq a}}^{K}{ \pi_{\theta_t}(i)\,\bigg| \underbrace{(r(i) - \dpd{\pit, r})}_{\leq r(1) - r(K)} } \bigg| + \underbrace{\pi_{\theta_t}(a)}_{\leq 1}\,\big| (r(a) - \dpd{\pit, r} \big| \right] \\
    &\le C\,\Big( (r(1) - r(K))\,\sum_{\substack{i = 1 \\ i \neq a}}^{K} \pi_{\theta_t}(i) + \big| r(a) -  \dpd{\pit, r} \big| \Big)  \\
    &= C\,\Big( (r(1) - r(K))\,(1 - \pit(a))  + \Big| \sum_{\substack{i=1 \\ i \neq a}}^K \pit(i)\,(r(a) - r(i)) \Big| \Big)  \\
    &\le C\,\Big( (r(1) - r(K))\,(1 - \pit(a) + \sum_{\substack{i=1 \\ i \neq a}}^K \pit(i) \big| \underbrace{(r(a) - r(i)}_{\leq r(1) - r(K)}) \big| \Big) \tag{by the triangle inequality}  \\
    &\le 2\,C\,(r(1) - r(K))\,\big( 1 - \pi_{\theta_t}(a) \big) .
\end{align*}
This implies that for all $t > 1$,
\begin{align*}
    | [\phi\,\theta_{t}](a^\prime) - [\phi\,\theta_{1}](a^\prime) | \le 2\,C\,(r(1) - r(K))\,\sum_{s=1}^{t-1}{ \big( 1 - \pi_{\theta_s}(a) \big) }. 
\end{align*}
Therefore, if $\sum_{t\ge 1}{ ( 1 - \pi_{\theta_t}(a) ) } < \infty$, then we have that
\begin{align*}
    \sup_{t \ge 1}{ | [\phi\,\theta_t](a^\prime) | } &\le \sup_{t \ge 1}{ | [\phi\,\theta_{t}](a^\prime) - [\phi\,\theta_1](a^\prime) |} + |  [\phi\,\theta_1](a^\prime) | < \infty .
\end{align*}
Therefore, there exists $\epsilon > 0$, such that for all $a \in [K]$,
\begin{align*}
    \inf_{t \ge 1}{ \pi_{\theta_t}(a) } = \inf_{t \ge 1}{ \frac{ \exp\parens*{ [\phi\,\theta_t ](a)} }{ \sum_{a^\prime \in [K]}{ \exp\parens*{ [\phi\,\theta_t](a^\prime) }} } } \ge \epsilon > 0 .
\end{align*}
This implies that the algorithm does not converge to a one-hot policy, which leads to a contradiction. Hence, $\sum_{t=1}^\infty (1 - \pit(a)) = \infty$ for all $a \in [K]$.
\end{proof}

\begin{lemma}[Smoothness]
\label{lemma:smoothness_expected_reward_log_linear_policy}
Given any reward vector $r \in \sR^K$ and feature matrix $\phi \in \sR^{K \times d}$. The expected reward function $\theta \mapsto \dpd{\pitheta, r}$ with $\pi_{\theta} = \softmax{( \phi\,\theta ) }$ is $L$-smooth where 
\begin{align*}
    L = \frac{9\,R_{\max}\,\lambda_{\max}(\phi\,\phi^\top)}{2} . \numberthis \label{eq:smoothness_expected_reward_log_linear_policy_results_1_appendix}
\end{align*}
\end{lemma}
\begin{proof}
Let  $S \coloneqq S(\phi, r, \theta)\in \R^{d \times d}$ be 
the second-order derivative of the value map $\theta \mapsto \dpd{\pi_\theta, r}$.
By Taylor's theorem, it suffices to show that the spectral radius of $S$ (regardless of $\theta$) is bounded by $L$. Now, by its definition, we have that
\begin{align*}
    S &= \frac{d }{d \theta } \left\{ \frac{d \dpd{\pi_\theta, r}}{d \theta} \right\} = \frac{d }{d \theta } \left\{ \phi^\top\,( \mathrm{diag}(\pi_\theta) - \pi_\theta \pi_\theta^\top)\,r \right\}.
\end{align*}
Continue our calculation with a pair of fixed $i, j \in [d]$. Then, we have that 
\begin{align*}
    S_{i, j} &= \frac{d \big\{ \sum_{a = 1}^{K} \phi_{a, i}\,\pi_\theta(a)\,( r(a) - \pi_\theta^\top r ) \big\} }{d \theta(j)} \\
    &= \sum_{a = 1}^{K} \phi_{a, i}\,\frac{ d \pi_\theta(a) }{d \theta(j)}\,\left( r(a) - \pi_\theta^\top r \right) -  \sum_{a = 1}^{K} \phi_{a, i}\,\pi_\theta(a)\,\sum_{a^\prime = 1}^{K} \frac{ d \pi_\theta(a^\prime) }{d \theta(j)}\,r(a^\prime). \numberthis \label{eq:smoothness_expected_reward_log_linear_policy_intermediate_2} 
\end{align*}
For all $a \in [K]$ and $j \in [d]$, we have that
\begin{align*}
\MoveEqLeft
    \frac{ d \pi_\theta(a) }{d \theta(j)} = \frac{d}{d \theta(j)} \bigg\{ \frac{\exp\{ [\phi\,\theta](a) \} }{ \sum_{a^\prime \in [K]}{\exp\{ [\phi\,\theta](a^\prime) \} } } \bigg\} \\
    &= \frac{ \frac{d\,\exp\{ [\phi\,\theta](a) \}}{d \theta(j)}\,\sum_{a^\prime \in [K]}{\exp\{ [\phi\,\theta](a^\prime) \}} - \exp\{ [\phi\,\theta](a) \}\,\frac{d\,\sum_{a^\prime \in [K]}{\exp\{ [\phi\,\theta](a^\prime) \}} }{d \theta(j)} }{ \big( \sum_{a^\prime \in [K]}{\exp\{ [\phi\,\theta](a^\prime) \} } \big)^2 } \\
    &= \frac{ \exp\{ [\phi\,\theta](a) \}\,\phi_{a,j}\,\sum_{a^\prime \in [K]}{\exp\{ [\phi\,\theta](a^\prime) \}} - \exp\{ [\phi\,\theta](a) \}\,\sum_{a^\prime \in [K]}{\exp\{ [\phi\,\theta](a^\prime) \}}\,\phi_{a^\prime,j} }{ \big( \sum_{a^\prime \in [K]}{\exp\{ [\phi\,\theta](a^\prime) \} } \big)^2 } \\
    &=  \frac{ \exp\{ [\phi\,\theta](a) \}\,\phi_{a,j}  - \exp\{ [\phi\,\theta](a) \}\,\sum_{a^\prime \in [K]}{\pi_\theta(a^\prime)\,\phi_{a^\prime,j} } }{ \sum_{a^\prime \in [K]}{\exp\{ [\phi\,\theta](a^\prime) \} } } \\
    &= \pi_{\theta}(a)\,\Big( \phi_{a,j} - \sum_{a^\prime \in [K]}{ \pi_\theta(a^\prime)\,\phi_{a^\prime,j} } \Big). \numberthis \label{eq:smoothness_expected_reward_log_linear_policy_intermediate_3} 
\end{align*}
Combining \cref{eq:smoothness_expected_reward_log_linear_policy_intermediate_2,eq:smoothness_expected_reward_log_linear_policy_intermediate_3}, we have that
\begin{align*}
    S_{i, j} &=  \sum_{a = 1}^{K} \phi_{a, i}\,\pi_\theta(a)\,( r(a) - \pi_\theta^\top r )\,\phi_{a, j} - \sum_{a = 1}^{K} \phi_{a, i}\,\pi_\theta(a)\,( r(a) - \pi_\theta^\top r )\,\sum_{a^\prime = 1}^{K}{ \pi_\theta(a^\prime)\,\phi_{a^\prime,j} } \\
    &\qquad -  \sum_{a = 1}^{K} \phi_{a, i}\,\pi_\theta(a)\,\sum_{a^\prime = 1}^{K} \pi_\theta(a^\prime)\,\Big( \phi_{a^\prime,j} - \sum_{a^{\prime\prime} = 1}^{K}{ \pi_\theta(a^{\prime\prime})\,\phi_{a^{\prime\prime},j} } \Big)\,r(a^\prime).
\end{align*}
To show the bound on 
the spectral radius of $S$, pick $y \in \sR^d$. Then, we have that
\begin{align*}
\MoveEqLeft
    \left| y^\top\,S\,y \right| = \bigg| \sum\limits_{i=1}^{d}{ \sum\limits_{j=1}^{d}{ S_{i,j}\,y(i)\,y(j)} } \bigg| \\
    &= \bigg| 
    \sum_{i=1}^{d} \sum_{j=1}^{d} \sum_{a = 1}^{K} y(i)\,\phi_{a, i}\,\pi_\theta(a)\,( r(a) - \pi_\theta^\top r )\,\phi_{a, j}\,y(j) \\
    &\qquad - \sum_{i=1}^{d} \sum_{j=1}^{d} \sum_{a = 1}^{K} y(i)\,\phi_{a, i}\,\pi_\theta(a)\,( r(a) - \pi_\theta^\top r )\,\sum_{a^\prime = 1}^{K}{ \pi_\theta(a^\prime)\,\phi_{a^\prime,j}\,y(j) } \\
    &\qquad - \sum_{i=1}^{d} \sum_{j=1}^{d} \sum_{a = 1}^{K} y(i)\,\phi_{a, i}\,\pi_\theta(a)\,\sum_{a^\prime = 1}^{K} \pi_\theta(a^\prime)\,\Big( \phi_{a^\prime,j} - \sum_{a^{\prime\prime} = 1}^{K}{ \pi_\theta(a^{\prime\prime})\,\phi_{a^{\prime\prime},j} } \Big)\,r(a^\prime)\,y(j) \bigg| \\
    &= \bigg| \sum_{a = 1}^{K} [ \phi\,y ](a)\,\pi_\theta(a)\,( r(a) - \pi_\theta^\top r )\,[ \phi\,y ](a) \\
    &\qquad - \sum_{a = 1}^{K} [ \phi\,y ](a)\,\pi_\theta(a)\,( r(a) - \pi_\theta^\top r )\,\sum_{a^\prime = 1}^{K} \pi_\theta(a^\prime)\,[ \phi\,y ](a^\prime) \\
    &\qquad - \sum_{a = 1}^{K} [ \phi\,y ](a)\,\pi_\theta(a)\,\sum_{a^\prime = 1}^{K} \pi_\theta(a^\prime)\,r(a^\prime)\,\Big( [ \phi\,y ](a^\prime) - \sum_{a^{\prime\prime} = 1}^{K} \pi_\theta(a^{\prime\prime})\,[ \phi\,y ](a^{\prime\prime}) \Big) \bigg|.
\end{align*} 
By defining $H(\pi_\theta)$ as $H(\pi_\theta) \coloneqq \mathrm{diag}(\pi_\theta) - \pi_\theta \pi_\theta^\top \in \sR^{K \times K}$, we have that
\begin{align*}
    \left| y^\top\,S\,y \right| &= \bigg| \big( H(\pi_\theta)\,r \big)^\top \left( \phi\,y \odot \phi\,y \right) - \big( H(\pi_\theta)\,r \big)^\top \big( \phi\,y \big)\,\big( \pi_\theta^\top \phi\,y \big) - \big( \pi_\theta^\top \phi\,y \big)\,\big( H(\pi_\theta) \phi\,y \big)^\top r \bigg| \\
    &= \bigg| \big( H(\pi_\theta)\,r \big)^\top \left( \phi\,y \odot \phi\,y \right) - 2\,\big( H(\pi_\theta)\,r \big)^\top \big( \phi\,y \big)\,\big( \pi_\theta^\top \phi\,y \big)  \bigg|,
\end{align*}
where $\odot$ is Hadamard (component-wise) product. Using the triangle inequality and H{\" o}lder's inequality, we have that
\begin{align*}
\MoveEqLeft
    \left| y^\top\,S\,y \right| \le \Big| \big( H(\pi_\theta)\,r \big)^\top \left( \phi\,y \odot \phi\,y \right) \Big| + 2\,\Big| \big( H(\pi_\theta)\,r \big)^\top \big( \phi\,y \big) \Big|\,\big| \pi_\theta^\top \phi\,y  \big| \\
    &\le \left\| H(\pi_\theta)\,r \right\|_\infty\,\left\| \phi\,y \odot \phi\,y \right\|_1 + 2\,\left\| H(\pi_\theta)\,r \right\|_1\,\left\| \phi\,y \right\|_\infty\,\left\| \pi_\theta \right\|_1\,\left\| \phi\,y \right\|_\infty \tag{by Cauchy--Schwarz inequality} \\
    &= \left\| H(\pi_\theta)\,r \right\|_\infty\,\left\| \phi\,y \right\|_2^2 + 2\,\left\| H(\pi_\theta)\,r \right\|_1\,\left\| \phi\,y \right\|_\infty^2 \tag{since $\| \phi\,y \odot \phi\,y \|_1 = \| \phi\,y \|_2^2,\,\| \pi_\theta \|_1 = 1 $ } \\
    &\le \left\| H(\pi_\theta)\,r \right\|_\infty\,\left\| \phi\,y \right\|_2^2 + 2\,\left\| H(\pi_\theta)\,r \right\|_1\,\left\| \phi\,y \right\|_2^2. \tag{since $\| \phi\,y \|_\infty \le \| \phi\,y \|_2$ }
\end{align*}
For $a \in [K]$, denote by $H_{a,:}(\pi_\theta)$ the $a$-th row of $H(\pi_\theta)$ as a row vector. Then, we get that
\begin{align*}
    \left\| H_{a,:}(\pi_\theta) \right\|_1 &= \pi_\theta(a) - \pi_\theta(a)^2 + \pi_\theta(a)\,\sum_{a^\prime \neq a}{ \pi_\theta(a^\prime) } \\
    &= \pi_\theta(a) - \pi_\theta(a)^2 +  \pi_\theta(a)\,( 1 - \pi_\theta(a) ) \\
    &= 2\,\pi_\theta(a)\,( 1 - \pi_\theta(a) ) \\
    &\le \frac{1}{2}. \tag{since $x\,(1 - x ) \le 1/4 \text{ for all } x \in [0, 1]$}
\end{align*}
On the other hand, we have that
\begin{align*}
    \left\| H(\pi_\theta)\,r \right\|_1 &= \sum_{a \in [K]}{ \pi_\theta(a)\,\left| r(a) - \pi_\theta^\top r \right| } \\
    &\le \max_{a \in [K]}{ \left| r(a) - \pi_\theta^\top r \right| } \\
    &\le 2\,R_{\max}. \tag{since $\norm{r}_{\infty} \leq R_{\max}$} 
\end{align*}
Therefore, we have that
\begin{align*}
\MoveEqLeft
    \left| y^\top\,S(\phi, r, \theta)\,y \right| \le \left\| H(\pi_\theta)\,r \right\|_\infty\,\left\| \phi\,y \right\|_2^2 + 2\,\left\| H(\pi_\theta)\,r \right\|_1\,\left\| \phi\,y \right\|_2^2 \\
    &= \max_{a \in [K]} \left| \left( H_{a,:}(\pi_\theta) \right)^\top r \right|\,\left\| \phi\,y \right\|_2^2 + 2\,\left\| H(\pi_\theta)\,r \right\|_1\,\left\| \phi\,y \right\|_2^2 \\
    &\le \max_{a \in [K]} \left\| H_{a,:}(\pi_\theta) \right\|_1\,R_{\max}\,\left\| \phi\,y \right\|_2^2 + 4\,R_{\max}\,\left\| \phi\,y \right\|_2^2 \\
    &\le \Big( \frac{1}{2} + 4 \Big)\,R_{\max}\,\left\| \phi\,y \right\|_2^2 \\
    &\le \frac{9}{2}\,R_{\max}\,\left\| \phi \right\|_\text{op}^2\,\left\| y \right\|_2^2 \\
    &= \frac{9}{2}\,R_{\max}\,\lambda_{\max}(\phi\,\phi^\top)\,\left\| y \right\|_2^2 .
\end{align*}
$\left\| \phi \right\|_\text{op}$ is the operator norm of $\phi \in \sR^{K \times d}$ defined as $\left\| \phi \right\|_\text{op} = \sup\big\{ \| \phi\,v \|_2: \| v \|_2 \le 1,\,v \in \sR^d \big\}$, which is equivalent to the square root of the largest eigenvalue of $\phi^\top\,\phi$.
According to Taylor's theorem, for all $ \theta,\,\theta^\prime \in \sR^d$, there exists $\theta_{\zeta} \coloneqq \zeta\,\theta + \left( 1 - \zeta \right)\,\theta^{\prime}$ with $\zeta \in [0, 1]$, such that
\begin{align*}
    \left| \dpd{\pi_{\theta^\prime} - \pi_\theta, r} - \Big\langle \frac{d \dpd{\pi_\theta, r}}{d \theta}, \theta^\prime - \theta \Big\rangle \right| &= \frac{1}{2}\,\left| \left( \theta^\prime - \theta \right)^\top S(\phi, r,\theta_\zeta) \left( \theta^\prime - \theta \right) \right| \\
    &\le \frac{9}{4}\,R_{\max}\,\lambda_{\max}(\phi\,\phi^\top)\,\| \theta^\prime - \theta \|_2^2.
\end{align*}
\end{proof}

\section{Proofs of \texorpdfstring{\cref{sec:stochastic_linear_bandits}}{}}

\subsection{Guarantee of Global Convergence} \label{subsec:stochastic_linear_bandits}

\subsubsection{Proof of \texorpdfstring{~\cref{theorem:stochastic_linear_bandits}}{}}
\stochasticlinearbandits*

\begin{proof}
According to~\cref{lemma:monotonicity_for_stochastic_bandits}, for all $t\geq1$, we have that $\E_t[\dpd{\pi_{\theta_{t+1}},r}]\geq\dpd{\pit,r}$. Hence, $\{\dpd{\pit,r}\}_{t\geq1}$ is a bounded submartingale and therefore converges almost surely. Moreover,~\cref{lemma:monotonicity_for_stochastic_bandits} guarantees that there exists a possibly random action $a\in[K]$ such that $\lim_{t\to\infty}\pit(a)=1$ almost surely. Consequently, $\lim_{t\to\infty}\dpd{\pit,r}=r(a)$ almost surely.

Formally, define $\gC_k\coloneq\{\lim_{t\to\infty}\pit(k)=1\}$ as the event that the policy converges to action $k\in[K]$. Throughout the proof, we implicitly localize $\gC_k$ over countably many deterministic starting times and finite deterministic bounds whenever necessary. Since the policy converges to a one-hot policy almost surely, $\sP[\bigcup_{k\in[K]}\gC_k]=1$. We will show that $\sP[\gC_k]=0$ for every $k\neq a^\star$, which implies that $\gC_{a^\star}$ occurs almost surely.

We proceed by contradiction. Fix some $k\in[2,K]$ and suppose that the event $\gC_k$ holds. On this event, $\lim_{t\to\infty}\dpd{\pit,r}=r(k)$. Let $\Delta\coloneq\min_{1\leq i\leq K-1}(r(i)-r(i+1))>0$. Then, there exists a finite iteration $\tau_k>1$ such that, for all $t\geq\tau_k$,
\begin{align*}
\forall i<k:\,r(i)-\dpd{\pit,r}>\frac{\Delta}{2};\quad \forall i>k:\,\dpd{\pit,r}-r(i)>\frac{\Delta}{2}. \numberthis\label{eq:diff_r_i_and_pi_r}
\end{align*}

We will then establish the following claims. In Claim I, we first prove that the total probability mass assigned to actions other than $k$ is not summable. We then develop a progress-versus-noise argument that will be used repeatedly. Claim II uses this argument to rule out the boundary case $k=K$. For the remaining case $k\in[2,K-1]$, Claim III identifies an action $a_0>k$ with infinite cumulative probability, Claim IV uses this action to show that the optimal action itself has infinite cumulative probability, and Claim V rules out convergence to any such suboptimal action $k$.
\paragraph{Claim I:} Under the event $\gC_k$, $\sum_{t=1}^{\infty}(1-\pit(k))=\infty$.\\
We prove the claim by contradiction. Suppose that $\sum_{t=1}^{\infty}(1-\pit(k))<\infty$. Since $\pit(i)\leq1-\pit(k)$ for every $i\neq k$, it follows that $\sum_{t=1}^{\infty}\pit(i)<\infty$ for every $i\neq k$. By the Borel--Cantelli lemma (\cref{lemma:extended_borel_cantelli}), every action $i\neq k$ is then sampled only finitely many times almost surely. Since there are finitely many actions, there exists a finite random iteration $\tau_{k,0}$ after which only action $k$ is sampled. Thus, for every $s\geq\tau_{k,0}$, we have $a_s=k$. By the update in~\cref{alg:sto_spg},
\begin{align*}
\theta_{s+1}-\theta_s=\eta\,R_s(k)\Big(\phi_k-\sum_{j\in[K]}\pis(j)\phi_j\Big).
\end{align*}
Therefore, for every action $b\in[K]$,
\begin{align*}
z_{s+1}(b)-z_s(b)=\eta\,R_s(k)\sum_{j\neq k}\pis(j)\dpd{\phi_b,\phi_k-\phi_j}.
\end{align*}
Define $C_1\coloneq\max_{b,j\in[K]}|\dpd{\phi_b,\phi_k-\phi_j}|<\infty$. Since $|R_s(k)|\leq R_{\max}$,
\begin{align*}
|z_{s+1}(b)-z_s(b)|\leq\eta\,R_{\max}\,C_1\,\sum_{j\neq k}\pis(j)=\eta\,R_{\max}\,C_1(1-\pis(k)).
\end{align*}
Hence, $\sum_{s=\tau_{k,0}}^{\infty}|z_{s+1}(b)-z_s(b)|<\infty$ for every $b\in[K]$, so each sequence $\{z_s(b)\}_{s\geq1}$ converges to a finite limit. The softmax of a finite vector assigns strictly positive probability to every action, which contradicts $\lim_{t\to\infty}\pit(k)=1$. Therefore, $\sum_{t=1}^{\infty}(1-\pit(k))=\infty$.
\paragraph{A Progress-versus-Noise Argument.} We compare the probabilities of two actions through their logit difference. In particular, for any action $a\neq a^\star$, the softmax parameterization gives
\begin{align*}
\frac{\pit(a^\star)}{\pit(a)}=\exp\parens*{[\phi\tht](a^\star)-[\phi\tht](a)}=\exp(z_t(a^\star)-z_t(a)).
\end{align*}
Thus, it suffices to show that the corresponding logit difference diverges to infinity. Recall from~\cref{eq:stochastic_logit_difference_decomposition} that, for any finite $\tau_k<t$,
\begin{align*}
z_t(a^\star)-z_t(a)=z_{\tau_k}(a^\star)-z_{\tau_k}(a)+\sum_{s=\tau_k}^{t-1}[P_s(a^\star)-P_s(a)]+\sum_{s=\tau_k+1}^{t}[W_s(a^\star)-W_s(a)].
\end{align*}
Therefore, we control the logit difference by showing that the cumulative progress dominates the cumulative stochastic noise.

We first derive a general expression for the progress between $a^\star$ and any action $a\neq a^\star$. Recall from the decomposition of the stochastic process in~\cref{subsec:decomposition_of_stochastic_process} that $P_s(a)\coloneq\E_s[z_{s+1}(a)]-z_s(a)$. Therefore,
\begin{align*}
P_s(a^\star)-P_s(a)&=\E_s[z_{s+1}(a^\star)]-z_s(a^\star)-\left(\E_s[z_{s+1}(a)]-z_s(a)\right)\\
&=\E_s\left[[\phi\,\theta_{s+1}](a^\star)-[\phi\,\theta_s](a^\star)\right]-\E_s\left[[\phi\,\theta_{s+1}](a)-[\phi\,\theta_s](a)\right]\\
&=\eta\,\dpd{\phi_{a^\star},\E_s\left[\frac{d\dpd{\pis,\hat r_s}}{d\theta_s}\right]}-\eta\,\dpd{\phi_a,\E_s\left[\frac{d\dpd{\pis,\hat r_s}}{d\theta_s}\right]} \tag{by the update in~\cref{alg:sto_spg}}\\
&=\eta\,\dpd{\phi_{a^\star}-\phi_a,\frac{d\dpd{\pis,r}}{d\theta_s}} \tag{by~\cref{lemma:unbiased_gradient}}\\
&=\eta\,\sum_{i\in[K]}\dpd{\phi_i,\phi_{a^\star}-\phi_a}\,\pis(i)\,(r(i)-\dpd{\pis,r})\\
&=\eta\,\sum_{\substack{i\in[K]\\i\neq k}}\dpd{\phi_i-\phi_k,\phi_{a^\star}-\phi_a}\,\pis(i)\,(r(i)-\dpd{\pis,r}). \numberthis\label{eq:progress_a_star_a}
\end{align*}
The last equality follows from $\sum_{i\in[K]}\pis(i)\,(r(i)-\dpd{\pis,r})=0$, which allows us to subtract $\phi_k$ from every feature vector. We will repeatedly use~\cref{eq:progress_a_star_a} below.

We next use~\cref{lemma:progress_vs_noise} to control the cumulative noise relative to the cumulative progress. In particular, for a fixed action $k$ and any $a\neq a^\star$, recall that
\begin{align*}
\mathcal{X}(k,a)\coloneq\left\{i\in[K]\mid\left|\left\langle\phi_i-\phi_k,\phi_{a^\star}-\phi_a\right\rangle\right|>0\right\},\\
\Gamma_s(k,a)\coloneq\sum_{i\in\mathcal{X}(k,a)}\pis(i), \quad S_t(k,a)\coloneq\sum_{s=\tau_k}^{t-1}\Gamma_s(k,a).
\end{align*}
According to~\cref{lemma:progress_vs_noise}, whenever there exist constants $C_2>0$ and $C_3\geq0$ such that, for every $t>\tau_k$,
\begin{align}
\sum_{s=\tau_k}^{t-1}[P_s(a^\star)-P_s(a)]\geq\eta\,C_2\,S_t(k,a)-C_3, \label{eq:progress_condition_for_claims}
\end{align}
there exist constants $C_4,C_5>0$ such that, for any $\delta\in(0,1)$, with probability at least $1-\delta$,
\begin{align}
z_t(a^\star)-z_t(a)&\geq z_{\tau_k}(a^\star)-z_{\tau_k}(a)+\eta\,C_2\,S_t(k,a)-C_3 \notag\\
&\quad-6C_5\sqrt{\left(C_4\,S_t(k,a)+\frac{4}{3}\right)\log\left(\frac{C_4\,S_t(k,a)+1}{\delta}\right)}-2C_5\log\left(\frac{1}{\delta}\right)-\frac{4C_5}{3}\log3. \label{eq:progress_noise_lemma_bound}
\end{align}
We will verify~\cref{eq:progress_condition_for_claims} for each relevant action pair. By~\cref{eq:stochastic_logit_difference_decomposition}, the logit difference is the sum of the cumulative progress and cumulative noise. When $\lim_{t\to\infty}S_t(k,a)=\infty$,~\cref{eq:progress_condition_for_claims} shows that the cumulative progress is lower bounded by a positive term linear in $S_t(k,a)$, while~\cref{eq:progress_noise_lemma_bound} shows that the negative noise contribution is of order $\sqrt{S_t(k,a)\,\log S_t(k,a)}$. Since $\sqrt{S_t(k,a)\,\log S_t(k,a)}=o(S_t(k,a))$, the cumulative progress eventually dominates the cumulative noise, and thus $\lim_{t\to\infty}[z_t(a^\star)-z_t(a)]=\infty$ with probability at least $1-\delta$.

To upgrade this statement to an almost-sure one, for each $\delta\in(0,1)$, let $\mathcal{E}_\delta$ denote the event on which~\cref{eq:progress_noise_lemma_bound} holds. Take a sequence $\{\delta_m\}_{m\geq1}$ satisfying $\lim_{m\to\infty}\delta_m=0$. Since $\sP[\bigcap_{m\geq1}\mathcal{E}_{\delta_m}^c]\leq\inf_{m\geq1}\delta_m=0$, almost surely~\cref{eq:progress_noise_lemma_bound} holds for at least one $\delta_m>0$. Hence, whenever $\lim_{t\to\infty}S_t(k,a)=\infty$ and~\cref{eq:progress_condition_for_claims} holds, it follows almost surely that
\begin{align}
\lim_{t\to\infty}[z_t(a^\star)-z_t(a)]=\infty \implies \lim_{t\to\infty}\frac{\pit(a^\star)}{\pit(a)}=\infty. \label{eq:progress_noise_consequence}
\end{align}
\paragraph{Claim II:} $\sP[\gC_K]=0$.\\
Suppose that the event $\gC_K$ holds. In this case, let $\tau_K>1$ be the finite iteration specified in~\cref{eq:diff_r_i_and_pi_r} for $k=K$. By Claim I, $\sum_{s=\tau_K}^{\infty}\sum_{i=1}^{K-1}\pis(i)=\infty$. Moreover,~\cref{assumption:general_feature_conditions} gives $\dpd{\phi_i-\phi_K,\phi_{a^\star}-\phi_K}>0$ for every $i<K$. Hence, $\mathcal{X}(K,K)=[K-1]$ and $\Gamma_s(K,K)=\sum_{i=1}^{K-1}\pis(i)$. Accordingly, in this claim we have $S_t(K,K)=\sum_{s=\tau_K}^{t-1}\Gamma_s(K,K)$. Define
\begin{align*}
C_6&\coloneq\min_{\substack{a_1,a_2\in[K],\,a\in[K]\setminus\{a^\star\}\\|\langle\phi_{a_1}-\phi_{a_2},\phi_{a^\star}-\phi_a\rangle|>0}}\left|\left\langle\phi_{a_1}-\phi_{a_2},\phi_{a^\star}-\phi_a\right\rangle\right|>0,\\
C_7&\coloneq\frac{C_6\Delta}{2}>0.
\end{align*}
Setting $a=K$ in~\cref{eq:progress_a_star_a} and using~\cref{eq:diff_r_i_and_pi_r},
\begin{align*}
P_s(a^\star)-P_s(K)&=\eta\,\sum_{i=1}^{K-1}\underbrace{\dpd{\phi_i-\phi_K,\phi_{a^\star}-\phi_K}}_{>0}\pis(i)\underbrace{(r(i)-\dpd{\pis,r})}_{>\Delta/2}\\
&\geq\eta\,C_7\,\sum_{i=1}^{K-1}\pis(i)=\eta\,C_7\,\Gamma_s(K,K).
\end{align*}
Summing from $s=\tau_K$ to $t-1$ gives
\begin{align*}
\sum_{s=\tau_K}^{t-1}[P_s(a^\star)-P_s(K)]\geq\eta\,C_7\,S_t(K,K).
\end{align*}
Thus,~\cref{eq:progress_condition_for_claims} holds with $C_2=C_7$ and $C_3=0$. By Claim I, $\lim_{t\to\infty}S_t(K,K)=\infty$. Applying~\cref{eq:progress_noise_consequence} gives $\lim_{t\to\infty}\frac{\pit(a^\star)}{\pit(K)}=\infty$. On the other hand, on the event $\gC_K$, $\lim_{t\to\infty}\pit(K)=1$. Since $a^\star\neq K$, this implies $\lim_{t\to\infty}\pit(a^\star)=0$ and hence $\lim_{t\to\infty}\frac{\pit(a^\star)}{\pit(K)}=0$, which is a contradiction. Therefore, $\sP[\gC_K]=0$.\\

It remains to rule out $\gC_k$ for any $k\in[2,K-1]$. Hence, throughout the remainder of the proof, fix an arbitrary $k\in[2,K-1]$ and work on the event $\gC_k$, with $\tau_k$ denoting the finite iteration specified in~\cref{eq:diff_r_i_and_pi_r}.
\paragraph{Claim III:} Under the event $\gC_k$ for any fixed $k\in[2,K-1]$, there exists an action $a_0>k$ such that $\sum_{s=1}^{\infty}\pis(a_0)=\infty$.\\
We again argue by contradiction. Suppose that $\sum_{s=\tau_k}^{\infty}\pis(a)<\infty$ for every $a>k$. Then, $\sum_{s=\tau_k}^{\infty}\sum_{i=k+1}^{K}\pis(i)<\infty$. By Claim I, this implies $\sum_{s=\tau_k}^{\infty}\sum_{i=1}^{k-1}\pis(i)=\infty$. Setting $a=k$ in~\cref{eq:progress_a_star_a}, we obtain
\begin{align*}
P_s(a^\star)-P_s(k)&=\eta\,\sum_{i=1}^{k-1}\underbrace{\dpd{\phi_i-\phi_k,\phi_{a^\star}-\phi_k}}_{>0}\pis(i)\underbrace{(r(i)-\dpd{\pis,r})}_{>\Delta/2}\\
&\quad+\eta\,\sum_{i=k+1}^{K}\dpd{\phi_i-\phi_k,\phi_{a^\star}-\phi_k}\,\pis(i)\,(r(i)-\dpd{\pis,r}).
\end{align*}
For every $i<k$,~\cref{assumption:general_feature_conditions} gives $\dpd{\phi_i-\phi_k,\phi_{a^\star}-\phi_k}>0$. By the definition of $C_6$ and~\cref{eq:diff_r_i_and_pi_r}, for every $i<k$, $\dpd{\phi_i-\phi_k,\phi_{a^\star}-\phi_k}(r(i)-\dpd{\pis,r})>C_6\Delta/2=C_7$. Define
\begin{align*}
C_8\coloneq\max_{a_1,a_2,a_3\in[K]}\left|\left\langle\phi_{a_1}-\phi_{a_2},\phi_{a^\star}-\phi_{a_3}\right\rangle\right|>0.
\end{align*}
For $i>k$, we do not control the sign of $\dpd{\phi_i-\phi_k,\phi_{a^\star}-\phi_k}$, but $|\dpd{\phi_i-\phi_k,\phi_{a^\star}-\phi_k}|\leq C_8$ and $|r(i)-\dpd{\pis,r}|\leq2R_{\max}$. Therefore,
\begin{align}
P_s(a^\star)-P_s(k)\geq\eta\,C_7\,\sum_{i=1}^{k-1}\pis(i)-2\eta\,R_{\max}\,C_8\,\sum_{i=k+1}^{K}\pis(i). \label{eq:progress_astar_k_first}
\end{align}
Moreover, $\Gamma_s(k,k)\leq\sum_{i=1}^{k-1}\pis(i)+\sum_{i=k+1}^{K}\pis(i)$, so $\sum_{i=1}^{k-1}\pis(i)\geq\Gamma_s(k,k)-\sum_{i=k+1}^{K}\pis(i)$. Substituting this into~\cref{eq:progress_astar_k_first} gives
\begin{align*}
P_s(a^\star)-P_s(k)\geq\eta\,C_7\,\Gamma_s(k,k)-\eta(C_7+2R_{\max}\,C_8)\sum_{i=k+1}^{K}\pis(i).
\end{align*}
Summing from $s=\tau_k$ to $t-1$ gives
\begin{align*}
\sum_{s=\tau_k}^{t-1}[P_s(a^\star)-P_s(k)]\geq\eta\,C_7\,S_t(k,k)-\eta(C_7+2R_{\max}\,C_8)\sum_{s=\tau_k}^{t-1}\sum_{i=k+1}^{K}\pis(i).
\end{align*}
The final cumulative sum, and hence the second term on the right-hand side, is finite by assumption. Therefore, there exists a finite constant $C_9\geq0$ such that
\begin{align*}
\sum_{s=\tau_k}^{t-1}[P_s(a^\star)-P_s(k)]\geq\eta\,C_7\,S_t(k,k)-C_9.
\end{align*}
Thus,~\cref{eq:progress_condition_for_claims} holds with $C_2=C_7$ and $C_3=C_9$. Furthermore,~\cref{assumption:general_feature_conditions} implies that every $i<k$ belongs to $\mathcal{X}(k,k)$. Hence, $S_t(k,k)\geq\sum_{s=\tau_k}^{t-1}\sum_{i=1}^{k-1}\pis(i)$, and therefore $\lim_{t\to\infty}S_t(k,k)=\infty$. Applying~\cref{eq:progress_noise_consequence} yields $\lim_{t\to\infty}\frac{\pit(a^\star)}{\pit(k)}=\infty$, contradicting the event $\gC_k$. Therefore, there exists at least one action $a_0>k$ such that $\sum_{s=\tau_k}^{\infty}\pis(a_0)=\infty$. Since removing finitely many terms does not affect divergence, $\sum_{s=1}^{\infty}\pis(a_0)=\infty$.
\paragraph{Claim IV:} Under the event $\gC_k$ for any fixed $k\in[2,K-1]$, $\sum_{s=1}^{\infty}\pis(a^\star)=\infty$.\\
We compare the optimal action with the action $a_0$ identified in Claim III. More generally, fix any $a>k$. Using~\cref{eq:progress_a_star_a} and splitting the summation according to whether $i<k$ or $i>k$, we obtain
\begin{align*}
P_s(a^\star)-P_s(a)&=\eta\,\sum_{i=1}^{k-1}\underbrace{\dpd{\phi_i-\phi_k,\phi_{a^\star}-\phi_a}}_{\geq0}\pis(i)\underbrace{(r(i)-\dpd{\pis,r})}_{>\Delta/2}\\
&\quad+\eta\,\sum_{i=k+1}^{K}\underbrace{\dpd{\phi_k-\phi_i,\phi_{a^\star}-\phi_a}}_{\geq0}\pis(i)\underbrace{(\dpd{\pis,r}-r(i))}_{>\Delta/2}.
\end{align*}
By~\cref{assumption:general_feature_conditions}, the indicated feature inner products are non-negative. Together with~\cref{eq:diff_r_i_and_pi_r}, every non-zero term is therefore positive. By the definition of $C_6$, for every $i\in\mathcal{X}(k,a)$, $\dpd{\phi_i-\phi_k,\phi_{a^\star}-\phi_a}\geq C_6$ when $i<k$, while $\dpd{\phi_k-\phi_i,\phi_{a^\star}-\phi_a}\geq C_6$ when $i>k$. Hence,
\begin{align*}
P_s(a^\star)-P_s(a)&\geq\eta\frac{C_6\Delta}{2}\left(\sum_{\substack{i=1\\i\in\mathcal{X}(k,a)}}^{k-1}\pis(i)+\sum_{\substack{i=k+1\\i\in\mathcal{X}(k,a)}}^{K}\pis(i)\right)=\eta\,C_7\,\Gamma_s(k,a).
\end{align*}
Summing from $s=\tau_k$ to $t-1$ gives
\begin{align*}
\sum_{s=\tau_k}^{t-1}[P_s(a^\star)-P_s(a)]\geq\eta\,C_7\,S_t(k,a).
\end{align*}
Thus,~\cref{eq:progress_condition_for_claims} holds with $C_2=C_7$ and $C_3=0$.

In particular, setting $a=a_0$,~\cref{assumption:general_feature_conditions} gives $\dpd{\phi_k-\phi_{a_0},\phi_{a^\star}-\phi_{a_0}}>0$, and hence $a_0\in\mathcal{X}(k,a_0)$. Therefore, $\Gamma_s(k,a_0)\geq\pis(a_0)$ and $S_t(k,a_0)\geq\sum_{s=\tau_k}^{t-1}\pis(a_0)$. By Claim III, $\lim_{t\to\infty}S_t(k,a_0)=\infty$. Applying~\cref{eq:progress_noise_consequence}, we obtain $\lim_{t\to\infty}\frac{\pit(a^\star)}{\pit(a_0)}=\infty$. Hence, for all sufficiently large $t$, $\pit(a^\star)\geq\pit(a_0)$. Since $\sum_{s=\tau_k}^{\infty}\pis(a_0)=\infty$, it follows that $\sum_{s=\tau_k}^{\infty}\pis(a^\star)=\infty$, and consequently $\sum_{s=1}^{\infty}\pis(a^\star)=\infty$.
\paragraph{Claim V:} $\sP[\gC_k]=0$ for any fixed $k\in[2,K-1]$.\\
We first show that, for every action $a>k$, $\lim_{t\to\infty}\frac{\pit(a^\star)}{\pit(a)}=\infty$. Fix an action $a > k$. From the derivation in Claim IV, we have that
\begin{align*}
\sum_{s=\tau_k}^{t-1}[P_s(a^\star)-P_s(a)]\geq\eta\,C_7\,S_t(k,a).
\end{align*}
Thus,~\cref{eq:progress_condition_for_claims} holds with $C_2=C_7$ and $C_3=0$. Also,~\cref{assumption:general_feature_conditions} gives $\dpd{\phi_{a^\star}-\phi_k,\phi_{a^\star}-\phi_a}>0$. Hence, $a^\star\in\mathcal{X}(k,a)$ and $\Gamma_s(k,a)\geq\pis(a^\star)$. Therefore, $S_t(k,a)\geq\sum_{s=\tau_k}^{t-1}\pis(a^\star)$. By Claim IV, $\sum_{s=\tau_k}^{\infty}\pis(a^\star)=\infty$, and hence $\lim_{t\to\infty}S_t(k,a)=\infty$. Applying~\cref{eq:progress_noise_consequence} gives that $\lim_{t\to\infty}\frac{\pit(a^\star)}{\pit(a)}=\infty$. Since $a>k$ was arbitrary, we conclude that
\begin{align*}
\forall a>k:\,\lim_{t\to\infty}\frac{\pit(a^\star)}{\pit(a)}=\infty.
\end{align*}
Since there are finitely many actions, this implies that $\lim_{t\to\infty}\sum_{i=k+1}^{K}\frac{\pit(i)}{\pit(a^\star)}=0$. Since $\sum_{i=1}^{k-1}\pit(i)\geq\pit(a^\star)$, it follows that
\begin{align}
\lim_{t\to\infty}\frac{\sum_{i=k+1}^{K}\pit(i)}{\sum_{i=1}^{k-1}\pit(i)}=0. \label{eq:lim_sum_above_and_below_k}
\end{align}
We next show that the optimal action also asymptotically dominates action $k$. Setting $a=k$ in~\cref{eq:progress_a_star_a} and using the same bounds as in~\cref{eq:progress_astar_k_first},
\begin{align*}
P_s(a^\star)-P_s(k)\geq\eta\,C_7\,\sum_{i=1}^{k-1}\pis(i)-2\eta\,R_{\max}\,C_8\,\sum_{i=k+1}^{K}\pis(i).
\end{align*}
By~\cref{eq:lim_sum_above_and_below_k}, there exists a finite iteration $\tau_{k,1}\geq\tau_k$ such that, for every $s\geq\tau_{k,1}$,
\begin{align*}
2R_{\max}\,C_8\,\sum_{i=k+1}^{K}\pis(i)\leq\frac{C_7}{2}\sum_{i=1}^{k-1}\pis(i) \; \text{and} \; \sum_{i=k+1}^{K}\pis(i)\leq\sum_{i=1}^{k-1}\pis(i).
\end{align*}
Hence, for every $s\geq\tau_{k,1}$,
\begin{align*}
P_s(a^\star)-P_s(k)\geq\eta\frac{C_7}{2}\sum_{i=1}^{k-1}\pis(i) \geq\eta\frac{C_7}{4}\Gamma_s(k,k),
\end{align*}
where the last inequality follows from the fact that for every $s\geq\tau_{k,1}$, $\Gamma_s(k,k)\leq\sum_{i=1}^{k-1}\pis(i)+\sum_{i=k+1}^{K}\pis(i)\leq2\sum_{i=1}^{k-1}\pis(i)$. Then, for every $t>\tau_{k,1}$, we can decompose the cumulative progress as follows.
\begin{align*}
\sum_{s=\tau_k}^{t-1}[P_s(a^\star)-P_s(k)]&=\sum_{s=\tau_k}^{\tau_{k,1}-1}[P_s(a^\star)-P_s(k)]+\sum_{s=\tau_{k,1}}^{t-1}[P_s(a^\star)-P_s(k)]\\
&\geq\sum_{s=\tau_k}^{\tau_{k,1}-1}[P_s(a^\star)-P_s(k)]+\eta\frac{C_7}{4}\sum_{s=\tau_{k,1}}^{t-1}\Gamma_s(k,k).
\end{align*}
Since there are only finitely many iterations between $\tau_k$ and $\tau_{k,1}$, there exists a finite constant $C_{10}\geq0$ such that, for every $t>\tau_k$,
\begin{align*}
\sum_{s=\tau_k}^{t-1}[P_s(a^\star)-P_s(k)]\geq\eta\frac{C_7}{4}S_t(k,k)-C_{10}.
\end{align*}
Hence,~\cref{eq:progress_condition_for_claims} holds with $C_2=C_7/4$ and $C_3=C_{10}$. Moreover,~\cref{assumption:general_feature_conditions} gives $a^\star\in\mathcal{X}(k,k)$, so $\Gamma_s(k,k)\geq\pis(a^\star)$. Therefore, $S_t(k,k)\geq\sum_{s=\tau_k}^{t-1}\pis(a^\star)$. By Claim IV, $\lim_{t\to\infty}S_t(k,k)=\infty$. Applying~\cref{eq:progress_noise_consequence} yields that $\lim_{t\to\infty}\frac{\pit(a^\star)}{\pit(k)}=\infty$.

However, on the event $\gC_k$, we have $\lim_{t\to\infty}\pit(k)=1$. Since $k\neq a^\star$ and the policy probabilities sum to one, it follows that $\lim_{t\to\infty}\pit(a^\star)=0$ and hence $\lim_{t\to\infty}\frac{\pit(a^\star)}{\pit(k)}=0$. This contradicts the conclusion above. Therefore, $\sP[\gC_k]=0$ for any fixed $k\in[2,K-1]$.\\
\paragraph{Putting Everything Together.} By Claim II, $\sP[\gC_K]=0$, while Claim V gives $\sP[\gC_k]=0$ for every $k\in[2,K-1]$. Hence, every suboptimal convergence event has probability zero. Finally, since $\sP[\bigcup_{k\in[K]}\gC_k]=1$, it holds that $\sP[\gC_{a^\star}]=1-\sP\left[\bigcup_{k>1}\gC_k\right]\geq1-\sum_{k>1}\sP[\gC_k]=1$. Therefore, $\lim_{t\to\infty}\pit(a^\star)=1$ almost surely.
\end{proof}

\subsubsection{Proofs of Lemmas}

To prove the claims used in the proof of~\cref{theorem:stochastic_linear_bandits}, we first consider the ratio $\pit(a^\star)/\pit(a)$ for a fixed reference action $k\in[K]$ and an action $a\in[K]\setminus\{a^\star\}$. We introduce some useful notations. For every $i\in[K]$, we define that
\begin{align*}
\bar\phi_{k,a,i}\coloneq\left|\left\langle\phi_i-\phi_k,\phi_{a^\star}-\phi_a\right\rangle\right|.
\end{align*}
Using this, we further define that
\begin{align*}
\mathcal{X}(k,a)\coloneq\left\{i\in[K]\mid\bar\phi_{k,a,i}>0\right\}
\end{align*}
as the set of actions with a non-zero feature interaction in the logit difference between $a^\star$ and $a$. We then define the following summation over this set for any iteration $s$:
\begin{align*}
\Gamma_s(k,a)\coloneq\sum_{i\in\mathcal{X}(k,a)}\pis(i).
\end{align*}
For any finite iteration $\tau\geq1$, we further define its cumulative value as
\begin{align*}
S_t(k,a)\coloneq\sum_{s=\tau}^{t-1}\Gamma_s(k,a).
\end{align*}
Then, we proceed to introduce the following helper lemma.

\begin{lemma}\label{lemma:progress_vs_noise}
Suppose Assumption~1 holds. Fix $k\in[K]$ and $a\in[K]\setminus\{a^\star\}$ such that $\mathcal{X}(k,a)\neq\emptyset$. Suppose there exist a finite iteration $\tau\geq1$ and constants $C_1>0$ and $C_2\geq0$ such that for all $t>\tau$,
\begin{align}
\sum_{s=\tau}^{t-1}[P_s(a^\star)-P_s(a)]\geq\eta\,C_1\,S_t(k,a)-C_2. \label{eq:progress_condition}
\end{align}
Then, there exist constants $C_3,C_4>0$ such that for all $t>\tau$ and any $\delta\in(0,1)$, it holds with probability at least $1-\delta$ that
\begin{align*}
z_t(a^\star)-z_t(a)&\geq z_\tau(a^\star)-z_\tau(a)+\eta\,C_1\,S_t(k,a)-C_2 \notag\\
&\quad-6C_4\sqrt{\left(C_3\,S_t(k,a)+\frac{4}{3}\right)\log\left(\frac{C_3\,S_t(k,a)+1}{\delta}\right)}-2C_4\log\left(\frac{1}{\delta}\right)-\frac{4C_4}{3}\log3.
\end{align*}
\end{lemma}

\begin{proof}
Using the decomposition of the stochastic process in~\cref{subsec:decomposition_of_stochastic_process} with $a_1=a^\star$ and $a_2=a$, and recursing~\cref{eq:stochastic_logit_decomposition} from $s=\tau$ to $t-1$, we have that
\begin{align}
z_t(a^\star)-z_t(a)=z_\tau(a^\star)-z_\tau(a)+\underbrace{\sum_{s=\tau}^{t-1}[P_s(a^\star)-P_s(a)]}_{\text{(i)}}+\underbrace{\sum_{s=\tau}^{t-1}[W_{s+1}(a^\star)-W_{s+1}(a)]}_{\text{(ii)}}. \label{eq:progress_noise_decomposition}
\end{align}
We first investigate Term (i), the cumulative progress. By~\cref{eq:progress_condition}, we directly have that
\begin{align}
\sum_{s=\tau}^{t-1}[P_s(a^\star)-P_s(a)]\geq\eta\,C_1\,S_t(k,a)-C_2. \label{eq:cumulative_progress_bound}
\end{align}
Next, we investigate Term (ii), the cumulative noise. We define that $y_{k,a}\coloneq(\phi-\mathbf{1}\phi_k^\top)(\phi_{a^\star}-\phi_a)$. For every $i\in[K]$, it holds that $y_{k,a}(i)=\left\langle\phi_i-\phi_k,\phi_{a^\star}-\phi_a\right\rangle$. Therefore, $y_{k,a}(i)=0$ for every $i\notin\mathcal{X}(k,a)$. Since $\mathcal{X}(k,a)\neq\emptyset$, we have that $\norm{y_{k,a}}_1>0$.
According to~\cref{lemma:stochastic_noise_difference_bound}, for all $s\geq1$, it holds that
\begin{align*}
\E_s[W_{s+1}(a^\star)-W_{s+1}(a)]&=0,\\
|W_{s+1}(a^\star)-W_{s+1}(a)|&\leq4\eta\,R_{\max}\,\norm{y_{k,a}}_1,\\
\Var_s[W_{s+1}(a^\star)-W_{s+1}(a)]&\leq2\eta^2R_{\max}^2\sum_{\substack{i\in[K]\\i\neq k}}y_{k,a}(i)^2\pis(i)\,(1-\pis(i)).
\end{align*}
Thus, $\{W_{s+1}(a^\star)-W_{s+1}(a)\}_{s\geq1}$ is a martingale difference sequence. Since it is bounded, we can normalize it to lie in $[-1/2,1/2]$ and define that $\widetilde W_{s+1}(a^\star,a)\coloneq\frac{W_{s+1}(a^\star)-W_{s+1}(a)}{8\eta\,R_{\max}\,\norm{y_{k,a}}_1}
$. Then, we further define that
\begin{align*}
C_3&\coloneq\frac{\max_{i\in[K]}\bar\phi_{k,a,i}^2}{32\norm{y_{k,a}}_1^2}>0,\\
C_4&\coloneq8\eta\,R_{\max}\,\norm{y_{k,a}}_1>0.
\end{align*}
Then, by~\cref{lemma:stochastic_noise_difference_bound},
\begin{align*}
\Var_s[\widetilde W_{s+1}(a^\star,a)]&\leq\frac{1}{32\norm{y_{k,a}}_1^2}\sum_{\substack{i\in[K]\\i\neq k}}y_{k,a}(i)^2\pis(i)\,(1-\pis(i))\\
&\leq\frac{\max_{i\in[K]}\bar\phi_{k,a,i}^2}{32\norm{y_{k,a}}_1^2}\sum_{i\in\mathcal{X}(k,a)}\pis(i)\\
&=C_3\,\Gamma_s(k,a).
\end{align*}
Using the above conditional-variance bound and applying~\cref{lemma:martingale} with the filtration $\{\mathcal{F}_s\}_{s\geq1}$, for any $\delta\in(0,1)$, it holds with probability at least $1-\delta$ that
\begin{align*}
\sum_{s=\tau}^{t-1}\widetilde W_{s+1}(a^\star,a)\geq-6\sqrt{\left(C_3\,S_t(k,a)+\frac{4}{3}\right)\log\left(\frac{C_3\,S_t(k,a)+1}{\delta}\right)}-2\log\left(\frac{1}{\delta}\right)-\frac{4}{3}\log3.
\end{align*}
Recalling that $C_4=8\eta\,R_{\max}\,\norm{y_{k,a}}_1$, we have that with probability at least $1-\delta$,
\begin{align}
\sum_{s=\tau}^{t-1}[W_{s+1}(a^\star)-W_{s+1}(a)]&\geq-6C_4\sqrt{\left(C_3\,S_t(k,a)+\frac{4}{3}\right)\log\left(\frac{C_3\,S_t(k,a)+1}{\delta}\right)} \notag\\
&\quad-2C_4\log\left(\frac{1}{\delta}\right)-\frac{4C_4}{3}\log3. \label{eq:cumulative_noise_bound}
\end{align}
Combining~\cref{eq:progress_noise_decomposition,eq:cumulative_progress_bound,eq:cumulative_noise_bound}, it holds with probability at least $1-\delta$ that
\begin{align*}
z_t(a^\star)-z_t(a)&\geq z_\tau(a^\star)-z_\tau(a)+\eta\,C_1\,S_t(k,a)-C_2\\
&\quad-6C_4\sqrt{\left(C_3\,S_t(k,a)+\frac{4}{3}\right)\log\left(\frac{C_3\,S_t(k,a)+1}{\delta}\right)}-2C_4\log\left(\frac{1}{\delta}\right)-\frac{4C_4}{3}\log3.
\end{align*}
This completes the proof.
\end{proof}

\subsection{Rate of Convergence}
\label{appendix:softmax_pg_sublinear_rate}

We first introduce the following definitions.
\paragraph{[Smoothness]} A function $f$ is $L$-smooth if for all $\theta$ and $\theta^\prime$
\begin{align*}
    \abs{f(\theta) - f(\theta^\prime) - \dpd{\gradf{\theta^\prime}, \theta - \theta^\prime}} \leq \frac{L}{2}\,\normsq{\theta - \theta^\prime}.
\end{align*}

\paragraph{[Non-Uniform Smoothness]} A function $f$ is $L$-non-uniform smooth if for all $\theta$ and $\theta^\prime$
\begin{align*}
    \abs{f(\theta) - f(\theta^\prime) - \dpd{\gradf{\theta^\prime}, \theta - \theta^\prime}} \leq \frac{L \norm{\gradf{\theta^\prime}}}{2}\,\normsq{\theta - \theta^\prime}.
\end{align*}

\paragraph{[Polyak–\L ojasiewicz Condition]} A function $f$ satisfies the non-uniform Polyak-\L ojasiewicz condition of degree $\xi \in [0, 1]$ if for all $\theta$,
\begin{align*}
     \norm{\gradf{\theta}}\geq C(\theta) \abs{f^\star - f(\theta)}^{1 - \xi},
\end{align*}
where $f^\star \coloneq \sup_{\theta} f(\theta)$ and $C: \theta  \rightarrow \R^+$.\\

Given these definitions, we can proceed to prove~\cref{lemma:monotonicity_for_stochastic_bandits,theorem:softmax_pg_sublinear_rate}.

\monotonicityforstochasticbandits*
\begin{proof}
To start, similar to the proof of~\cref{lemma:monotonicity_and_onehotpolicy}, we first show that there are no stationary points in the finite region. We will prove this by contradiction. Suppose there exists $\theta^\prime \in \sR^d$ ($\| \theta^\prime \|_2 < \infty$), such that,
\begin{align*}
    \frac{d\,\dpd{\pi_{\theta^\prime}, r} }{d \theta^\prime} = \phi^\top \left( \mathrm{diag}(\pi_{\theta^\prime}) - \pi_{\theta^\prime}\,\pi_{\theta^\prime}^\top \right)\,r = \bm{0}.
\end{align*}
Under~\cref{assumption:reward_ordering_preservation}, there exists $w \in \R^d$ such that $r^\prime \coloneqq \phi\,w$ preserves the ordering of $r$.
Taking the inner product with $w$ on both sides of the above equation, we have that
\begin{align*}
    w^\top \phi^\top \left( \mathrm{diag}(\pi_{\theta^\prime}) - \pi_{\theta^\prime}\,\pi_{\theta^\prime}^\top \right)\,r = {r^\prime}^\top \left( \mathrm{diag}(\pi_{\theta^\prime}) - \pi_{\theta^\prime}\,\pi_{\theta^\prime}^\top \right)\,r =  w^\top \bm{0} = 0.
    \numberthis \label{eq:r_prime_H_r_=_0_sto}
\end{align*}
Since $\| \theta^\prime \|_2 < \infty$ and $\phi$ is bounded ($\max_{i \in [K],\,j \in [d]}{ |\phi_{i,j}|} \le C$ for some $C < \infty$), we have that
\begin{align*}
    \forall i \in [K]:\,\pi_{\theta^{\prime}}(i) = \frac{ \exp\parens*{[\phi\,\theta^\prime](i)} }{ \sum_{j \in [K]}{ \exp\parens*{[\phi\,\theta^\prime](j)} } } > 0.
\end{align*}
According to~\cref{lemma:alternative_expression_covariance}, we have that
\begin{align*}
    \MoveEqLeft
    w^\top \phi^\top \left( \mathrm{diag}(\pi_{\theta^\prime}) - \pi_{\theta^\prime}\,\pi_{\theta^\prime}^\top \right)\,r = {r^\prime}^\top \left( \text{diag}(\pi_{\theta^\prime}) - \pi_{\theta^\prime}\,\pi_{\theta^\prime}^\top \right)\,r \\
    &= \sum_{i=1}^{K-1}{ \pi_{\theta^\prime}(i)\,\sum_{j = i+1}^{K}{\pi_{\theta^\prime}(j)\,\left( r^\prime(i) - r^\prime(j) \right)\,\left( r(i) - r(j) \right) }  } \\
    &\geq \Delta\,\Delta^\prime \sum_{i=1}^{K-1}{ \pi_{\theta^\prime}(i)\,\sum_{j = i+1}^{K}{\pi_{\theta^\prime}(j) }  } \\
    &= \frac{\Delta \Delta'}{2}\left[1-\sum_{i=1}^{K}\pi_{\theta'}(i)^2\right] > 0 , \numberthis \label{eq:w_phi_H_r}
\end{align*}
where $\Delta \coloneqq \min_{i \neq j}\abs{r(i) - r(j)} > 0$ and $\Delta' \coloneqq \min_{i \neq j}\abs{r'(i) - r'(j)} > 0$.
This leads to a contradiction with \cref{eq:r_prime_H_r_=_0_sto}. Therefore, any finite $\theta$ ($\| \theta \|_2 < \infty$) is not a stationary point.

Next, we can use this property along with other properties of stochastic estimates to prove this lemma. For simplicity, we define the following notations:
\begin{align*}
    f(\theta) &\coloneq \dpd{\pitheta, r} , \\
    \gradf{\theta} &\coloneq \frac{d\,\dpd{\pitheta, r}}{d \theta}  = \phi^\top\,(\mathrm{diag}(\pitheta) - \pitheta\,\pitheta^\top)\,r , \\
    \hgrad{\theta} &\coloneq \frac{d\,\dpd{\pitheta, \hat{r}}}{d \theta} = \phi^\top\,(\mathrm{diag}(\pitheta) - \pitheta\,\pitheta^\top)\,\hat{r} .
\end{align*}
We also define $\piz \coloneq \text{softmax}(z)$ for all $z \in \{\phi\,\theta \mid \theta \in \R^d\}$, implying that $\piz = \pitheta$. Additionally, we further define that $J(z) \coloneq \dpd{\piz, r}$, which further implies that
\begin{align*}
    \nabla J(z) = \frac{d\,\dpd{\piz, r}}{d z} = (\mathrm{diag}(\piz) - \piz\,\piz^\top)\,r .
\end{align*}
Following the derivation for~\cref{eq:w_phi_H_r}, we have that
\begin{align}
    w^\top \gradf{\theta} &\geq \frac{\Delta\,\Delta^\prime}{2} \left[ 1 - \sum_{i=1}^K \left( \pi_{\theta}(i) \right)^2 \right] . \label{eq:w_top_nablaf}
\end{align}
Using Cauchy--Schwarz's inequality, we have $w^\top \gradf{\theta} \leq \norm{w}_2\,\norm{\gradf{\theta}}_2$. This further implies that
\begin{align*}
    \norm{\gradf{\theta}}_2 \geq \frac{\Delta\,\Delta^\prime}{2\,\norm{w}_2} \left[ 1 - \sum_{i=1}^K \left( \pi_{\theta}(i) \right)^2 \right] .
\end{align*}
Next, we will relate $\norm{\gradf{\theta}}_2$ to $\norm{\nabla J(z)}_2$. For any $i \in [K]$, we have that
\begin{align*}
    [\nabla J(z)](i) &= \pitheta(i) \left[ r(i) - \dpd{\pi_\theta, r} \right] \\
    &= \pitheta(i) \sum_{\substack{j=1 \\ j \neq i}}^K \pitheta(j) \left[ r(i) - r(j) \right] \\
    &\leq 2\,R_{\max}\,\pitheta(i)\,\sum_{\substack{j=1 \\ j \neq i}}^K \pitheta(j) \tag{since $\supnorm{r} \leq R_{\max}$}.
\end{align*}
This implies that
\begin{align*}
    \norm{\nabla J(z)}_2 &\leq \norm{\nabla J(z)}_1  \tag{since $\| \cdot \|_2 \leq \| \cdot \|_1$}\\
    &\leq 2\,R_{\max}\,\sum_{i=1}^K \pitheta(i) \sum_{\substack{j=1 \\ j \neq i}}^K \pitheta(j) \\
    &= 2\,R_{\max}\,\sum_{i=1}^K \pitheta(i) [1 - \pitheta(i)] \\
    &= 2\,R_{\max} \left[ 1 - \sum_{i=1}^K \left( \pi_{\theta}(i) \right)^2 \right].
\end{align*}
Combining the above inequality with~\cref{eq:w_top_nablaf} yields that
\begin{align*}
    \norm{\gradf{\theta}}_2 \geq \frac{\Delta\,\Delta'}{4\,\norm{w}_2\,R_{\max}}\,\norm{\nabla J(z)}_2 \implies \norm{\gradf{\theta}}_2^2 \geq \underbrace{\frac{( \Delta\,\Delta^\prime )^2}{16\,\norm{w}_2^2\,R_{\max}^2}}_{\eqcolon \gamma} \norm{\nabla J(z)}_2^2 .
\end{align*}
Thus, $\|\nabla J(z)\|_2\leq\gamma^{-1/2}\|\nabla f(\theta)\|_2$. Using~\cref{lemma:lpg_ns}, we further have that
\begin{align*}
\|\nabla^2 f(\theta)\|_2 \leq 3\lambda\|\nabla J(z)\|_2 \leq \frac{3\,\lambda_{\max}(\phi\,\phi^\top)}{\sqrt{\gamma}}\|\nabla f(\theta)\|_2 .
\end{align*}
According to~\cref{lemma:lpg_ns}, $f$ is $L_1$-non-uniform smooth, and by~\cref{lemma:linear_bandit_sg_bounded}, the stochastic gradients are bounded by $B > 0$ where $L_1 \coloneq 3\,\lambda_{\max}(\phi\,\phi^\top)\,\gamma^{-1/2}$ and $B \coloneq \sqrt{2\,\lambda_{\max}(\phi\,\phi^\top)}\,R_{\max}$. Using~\cref{alg:sto_spg} with 
\begin{align*}
    0<\eta\leq\frac{1}{2L_1B}=\frac{\sqrt{\gamma}}{6\sqrt{2} \left[\lambda_{\max}(\phi\,\phi^\top) \right]^{3/2}R_{\max}} ,
\end{align*}
then $1-L_1B\eta\geq 1/2$, and~\cref{lemma:bound_gtheta_zeta} implies that
\begin{align*}
\abs{f(\thtt) - f(\tht) - \dpd{\nabla f(\tht), \thtt - \tht}} &\leq \frac{1}{2}\,\frac{L_1\,\norm{\nabla f(\theta_t)}_2}{1 - L_1\,B\,\eta}\,\normsq{\thtt - \tht} \le L_1\,\norm{\nabla f(\theta_t)}_2\,\normsq{\thtt - \tht} .
\end{align*}
This further implies that
\begin{align*}
    f(\theta_{t+1})-f(\theta_t)-\langle\nabla f(\theta_t),\theta_{t+1}-\theta_t\rangle \geq -L_1\norm{\nabla f(\theta_t)}_2\|\theta_{t+1}-\theta_t\|_2^2 .
\end{align*}
By the update in~\cref{alg:sto_spg}, $\thetatt = \thetat + \eta\,\hgrad{\thetat}$, we have that
\begin{align*}
    f(\theta_{t+1})-f(\theta_t)-\eta\langle\nabla f(\theta_t),\widetilde{\nabla}f(\theta_t)\rangle \geq -\eta^2 L_1 \norm{\nabla f(\theta_t)}_2 \normsq{\widetilde{\nabla}f(\theta_t)} .
\end{align*}
Rearranging the inequality yields that
\begin{align*}
    f(\theta_{t+1}) \geq f(\theta_t) + \eta\,\dpd{\nabla f(\theta_t),\widetilde{\nabla}f(\theta_t)} - \eta^2 L_1 \|\nabla f(\theta_t)\|_2 \|\widetilde{\nabla}f(\theta_t)\|_2^2.
\end{align*}
Taking expectation with respect to the randomness in iteration $t$ on both sides,
\begin{align*}
    \E_t[f(\thtt)] \geq \E_t[f(\tht)] + \eta\,\dpd{\gradf{\tht}, \E_t[\hgrad{\tht}]} - \eta^2\,L_1\,\norm{\nabla f(\theta_t)}_2\,\E_t\left[\norm{\hgrad{\tht}}\right] .
\end{align*}
According to~\cref{lemma:unbiased_gradient}, we have that
\begin{align}
    \E_t[f(\thtt)] \geq \E_t[f(\tht)] + \eta\,\normsq{\gradf{\tht}} - \eta^2\,L_1\,\norm{\nabla f(\theta_t)}_2\,\E_t\left[\normsq{\hgrad{\tht}}\right] .
\end{align}
To simplify the term $\E_t\left[\normsq{\hgrad{\tht}}\right]$ in the above inequality, according to~\cref{lemma:sgc}, the stochastic gradients satisfy the strong growth condition that
\begin{align*}
    \E_t\left[\normsq{\hgrad{\tht}}\right] \leq \lambda_{\max}(\phi\,\phi^\top) \underbrace{\frac{8\,R_{\max}^3\,K^{3/2}}{\Delta^2}}_{\eqcolon \rho}\,\norm{\nabla J(\z_t)} .
\end{align*}
Combining with the fact that $\|\nabla f(\theta)\|_2\leq\sqrt{\lambda_{\max}(\phi\phi^\top)}\|\nabla J(z)\|_2$, this further implies that
\begin{align*}
\mathbb{E}_t[J(z_{t+1})]&\geq\mathbb{E}_t[J(z_t)]+\eta\,\gamma\|\nabla J(z_t)\|_2^2-\eta^2L_1\rho\left[\lambda_{\max}(\phi\phi^\top)\right]^{3/2}\|\nabla J(z_t)\|_2^2\\
&= \E_t[J(z_t)] + \parens*{\eta\,\gamma - \frac{3\,\eta^2\,[\lambda_{\max}(\phi\,\phi^\top)]^{5/2}\,\rho}{\sqrt{\gamma}} }\,\normsq{\nabla J(z_t)} .
\end{align*}
Consequently, if in addition, $0<\eta\leq\frac{\gamma^{3/2}}{6\rho[\lambda_{\max}(\phi\,\phi^\top)]^{5/2}}$, then
\begin{align}
    \mathbb{E}_t[J(z_{t+1})]-J(z_t)\geq\frac{\eta\,\gamma}{2}\|\nabla J(z_t)\|_2^2\geq 0 . \numberthis \label{eq:stochastic_monotonicity}
\end{align}
Then, we examine the following sequence $\{Y_t\}_{t\geq1}$ where $Y_t = r(a^\star) - \dpd{\pit, r} \in [0, 2 R_{\max}]$. Following the proof of~\citet[Corollary 4.7]{mei2023stochastic}, since $Y_t$ is $\gF_t$-measurable, $\theta_t$ and $z_t$ are deterministic functions of $a_1, R_1(a_1)\,\dots, a_{t-1}, R_{t-1}(a_{t-1})$. By~\cref{eq:stochastic_monotonicity}, for all $t \geq 1$, $\langle\pi_{\theta_t},r\rangle-\mathbb{E}_t[\langle\pi_{\theta_{t+1}},r\rangle]\leq 0$, which indicates that $\E[Y_{t+1} \mid \gF_t] \leq Y_t$ is a nonnegative supermartingale. Hence, the conditions of Doob’s super-martingale theorem (\cref{theorem:doobs}) are satisfied and the sequence $\{\dpd{\pit, r}\}_{t \geq 1}$ converges to some finite limit $C \in [-R_{\max}, R_{\max}]$ almost surely.
Since $\dpd{\pit, r} \in [-R_{\max}, R_{\max}]$ and $Z_t \coloneq \dpd{\pit, r}$ for $t \geq 1$ satisfies the conditions of~\citet[Corollary 3]{mei2022role}, we have that almost surely $\lim_{t\to\infty}\left(\mathbb{E}_t[\langle\pi_{\theta_{t+1}},r\rangle]-\langle\pi_{\theta_{t+1}},r\rangle\right)=0$. Moreover, since $\langle\pi_{\theta_t},r\rangle$ converges almost surely, $\lim_{t\to\infty}\left(\langle\pi_{\theta_{t+1}},r\rangle-\langle\pi_{\theta_t},r\rangle\right)=0$, and thus $\lim_{t\to\infty}\left\|\frac{d\langle\pi_{\theta_t},r\rangle}{d(\phi\,\theta_t)}\right\|_2^2=0$. Since $\left\|\frac{d\langle\pi_{\theta_t},r\rangle}{d\theta_t}\right\|_2\leq\sqrt{\lambda_{\max}(\phi\phi^\top)}\left\|\frac{d\langle\pi_{\theta_t},r\rangle}{d(\phi\,\theta_t)}\right\|_2$, we further have that
\begin{align*}
    \lim_{t\to\infty}\left\|\frac{d\langle\pi_{\theta_t},r\rangle}{d\theta_t}\right\|_2^2=0 . \numberthis \label{eq:gradient_equals_zero_at_infinity_sto}
\end{align*}
Finally, by~\cref{eq:w_top_nablaf,eq:gradient_equals_zero_at_infinity_sto}, we have $1-\sum_{i=1}^{K}\pi_{\theta_t}(i)^2\to 0$. Hence, $\max_{i\in[K]}\pi_{\theta_t}(i)\to 1$. Let $a_t\in\arg\max_{i\in[K]}\pi_{\theta_t}(i)$. Thus, $\pi_{\theta_t}(a_t)\to 1$, and consequently $|\langle\pi_{\theta_t},r\rangle-r(a_t)|\leq 2R_{\max}(1-\pi_{\theta_t}(a_t))\to 0$. Since $\langle\pi_{\theta_t},r\rangle$ converges almost surely and the rewards are distinct, $a_t$ must eventually be equal to some fixed action $a\in[K]$. Therefore, $\lim_{t\to\infty}\pi_{\theta_t}(a)=1$, i.e., $\pi_{\theta_t}$ approaches a one-hot policy.
\end{proof}

\softmaxpgsublinearrate*
\begin{proof}
Under~\cref{assumption:no_identical_arms,assumption:reward_ordering_preservation}, according to~\cref{lemma:monotonicity_for_stochastic_bandits}, for all $t \geq 1$, we have that
\begin{align*}
    \mathbb{E}_t[\langle\pi_{\theta_{t+1}},r\rangle]-\langle\pi_{\theta_t},r\rangle\geq\frac{\eta\,\gamma}{2}\left\|\frac{d\langle\pi_{\theta_t},r\rangle}{d(\phi\,\theta_t)}\right\|_2^2. 
\end{align*}
Multiplying both sides by $-1$ and adding $\dpd{\pi^\star, r}$ on both sides,
\begin{align*}
    \E_t \Bigl[ \underbrace{\dpd{\pi^{\star} - \pitt, r}}_{\eqcolon \delta(\thtt)} \Bigr]  \leq \underbrace{ \dpd{\pi^{\star} - \pit, r}  }_{\eqcolon \delta(\tht)} - \frac{\eta\,\gamma}{2} \normsq{\frac{d\,\dpd{\pit, r}}{d (\phi\,\theta_t)}} .
\end{align*}
Using~\cref{lemma:non_uniform_l}, we have that
\begin{align*}
    \mathbb{E}_t[\delta(\theta_{t+1})]\leq\delta(\theta_t)-\frac{\eta\,\gamma}{2}[\pi_{\theta_t}(a^\star)]^2[\delta(\theta_t)]^2 .
\end{align*}
Define that $\nu \coloneq \inf_{t \geq 1} [\pit(a^\star)]^2 $.
Since the softmax policy assigns strictly positive probability to every action at each finite iteration and $\pi_{\theta_t}(a^\star)\to1$ almost surely by~\cref{theorem:stochastic_linear_bandits}, we have $\nu>0$ almost surely.
Therefore, we have that
\begin{align*}
    \E_t[\delta(\theta_{t+1})]
\leq \delta(\theta_t)-\frac{\eta\,\gamma\nu}{2}[\delta(\theta_t)]^2 .
\end{align*}
Taking expectation with respect to all previous iterations $t \geq 1$ on both sides,
\begin{align*}
\E[\delta(\thtt)] &\leq \E[\delta(\tht)] - \frac{\eta\,\gamma}{2}\,\E[\nu\,[\delta(\tht)]^2] .
\end{align*}
To lower bound $\E[\nu\,[\delta(\tht)]^2]$, since $\nu > 0$ and $\delta(\tht) > 0$, we can apply the Cauchy--Schwarz inequality and obtain that
\begin{align*}
\E[\delta(\tht)] &= \E\left[ \frac{1}{\sqrt{\nu}}\,\sqrt{\nu}\,\delta(\tht)\right] \leq \sqrt{\E\left[\nu^{-1} \right]}\,\sqrt{\E[\nu\,[\delta(\tht)]^2]} .
\end{align*}
Define that $\mu \coloneq (\E \left[\nu^{-1} \right] )^{-1}$. Then, we have that
\begin{align*}
    \mu\,(\E[\delta(\tht)])^2 \leq \E[\nu\,[\delta(\tht)]^2] .
\end{align*}
Hence, we have that
\begin{align*}
   \E[\delta(\thtt)] \leq \E[\delta(\tht)] - \frac{\eta\,\gamma\,\mu}{2}\,(\E[\delta(\tht)])^2 = \E[\delta(\tht)] - \frac{1}{\alpha}\,(\E[\delta(\tht)])^2,
\end{align*}
where $\alpha \coloneq \frac{2}{\eta\,\gamma\,\mu}$. Dividing each side by $\E[\delta(\tht)]\,\E[\delta(\thtt)]$ implies that
\begin{align*}
   \frac{1}{\E[\delta(\tht)]} \leq \frac{1}{\E[\delta(\theta_{t+1})]} - \frac{1}{\alpha}\,\frac{\E[\delta(\tht)]}{\E[\delta(\theta_{t+1})]}.
\end{align*}
Using the above inequality and recursing from iteration $t=1$ to $T$,
\begin{align*}
    \frac{1}{\E[\delta(\theta_1)]} &\leq \frac{1}{\E[\delta(\theta_{T+1})]} - \frac{1}{\alpha} \sum_{t=1}^{T} \frac{\E[\delta(\tht)]}{\E[\delta(\thtt)]} \\
    &\leq \frac{1}{\E[\delta(\theta_{T+1})]} - \frac{T}{\alpha} \tag{since $\E[\delta(\tht)] \geq \E[\delta(\thtt)]$}
\end{align*}
Therefore, we finally have that
\begin{align*}
\frac{T}{\alpha} \leq \frac{1}{\E[\delta(\theta_{T+1})]} \implies \E\left[\dpd{\pi^\star,r}-\dpd{\pi_{\theta_{T+1}},r}\right]\leq\frac{2}{\eta\,\gamma\,\mu\,T}.
\end{align*}
\end{proof}  

\subsection{Additional Properties}
\label{appendix:stochastic_linear_bandits_additional_properties}

\begin{restatable}[Unbiased Stochastic Gradient]{lemma}{unbiasedstochasticgradients}
\label{lemma:unbiased_gradient}
\cref{alg:sto_spg} ensures that for all $t \geq 1$,
\begin{align*}
\E_t\left[\frac{d \dpd{\pit, \hat{r}_t}}{d \tht}\right] = \frac{d \dpd{\pit, r}}{d \tht}.
\end{align*}
\end{restatable}
\begin{proof}
First, we show that $\E_t \left[ \frac{d \dpd{\pit, \hat{r}_t}}{d z_t} \right] = \frac{d \dpd{\pit, r}}{d z_t}$. \\
For the sampled action $a_t$, we have that
\begin{align*}
    \E_{ R_t(a_t) \sim P_{a_t} }{ \bigg[ \frac{d \dpd{\pit, \hat{r}_t}}{d z_t(a_t)} \bigg] } &= \E_{ R_t(a_t) \sim P_{a_t} }{ \Big[ \left( 1 - \pit(a_t) \right)\,R_t(a_t) \Big] } \\
    &= \left( 1 - \pit(a_t) \right)\,\E_{ R_t(a_t) \sim P_{a_t} }{ \Big[ R_t(a_t) \Big] } \\
    &= \left( 1 - \pit(a_t) \right)\,r(a_t).
\end{align*}
For any other actions $a \not= a_t$ that are not sampled, we have that
\begin{align*}
    \E_{ R_t(a_t) \sim P_{a_t} }{ \bigg[ \frac{d \dpd{\pit, \hat{r}_t}}{d z_t(a)} \bigg] } &= \E_{ R_t(a_t) \sim P_{a_t} }{ \Big[ - \pit(a)\,R_t(a_t) \Big] } \\
    &= - \pit(a)\,\E_{ R_t(a_t) \sim P_{a_t} }{ \Big[ R_t(a_t) \Big] } \\
    &= - \pit(a)\,r(a_t).
\end{align*}
Combining the above two equations, we have that for all $a \in [K]$,
\begin{align*}
    \E_{ R_t(a_t) \sim P_{a_t} }{ \bigg[ \frac{d \dpd{\pit,  \hat{r}_t}}{d z_t(a)} \bigg] } = \left( \sI\left\{ a_t = a \right\} - \pit(a) \right)\,r(a_t).
\end{align*}
Taking expectation over $a_t \sim \pit$, we have that
\begin{align*}
    \mathbb{E}_t{ \bigg[ \frac{d \dpd{\pit,  \hat{r}_t} }{d z_t(a)} \bigg] } &= \sP{\left( a_t = a \right) }\,\E_{ R_t(a_t) \sim P_{a_t} }{ \bigg[ \frac{d \dpd{\pit,  \hat{r}_t}}{d z_t(a)}\,\Big|\,a_t = a \bigg] } \\
    & \quad + \sP{\left( a_t \not= a \right) }\,\E_{ R_t(a_t) \sim P_{a_t} }{ \bigg[ \frac{d \dpd{\pit, \hat{r}_t}}{d z_t(a)}\,\Big|\,a_t \not= a \bigg] } \\
    &= \pit(a)\,\left( 1 - \pit(a) \right)\,r(a) + \sum_{a^\prime \not= a}{ \pit(a^\prime)\,\left( - \pit(a) \right)\,r(a^\prime) } \\
    &= \pit(a)\,\sum_{a^\prime \not= a}{ \pit(a^\prime)\,\left( r(a) - r(a^\prime) \right) } \\
    &= \pit(a)\,\left( r(a) - \dpd{\pit, r} \right) \\
    &= \frac{d \dpd{\pit, r} }{d z_t(a)}.
\end{align*}
Therefore, using the chain rule, we have that
\begin{align*}
\E_t \left[ \frac{d \dpd{\pit, \hat{r}_t}}{d \tht} \right] = \phi^\top\,\E_t \left[ \frac{d \dpd{\pit, \hat{r}_t}}{d z_t} \right] = \phi^\top\,\frac{d \dpd{\pit, r} }{d z_t} = \frac{d \dpd{\pit, r}}{d \tht}.
\end{align*}
\end{proof}

\begin{lemma}\label{lemma:stochastic_noise_bound}
For an arbitrary action $a^\prime \in [K]$ and any $i\in[K]$, let
$y_{i,a^\prime}\coloneq(\phi-\textbf{1}\phi_i^\top)\phi_{a^\prime}\in\mathbb{R}^K$.
Then, it holds that
\begin{align*}
&\E_t[W_{t+1}(a^\prime)] = 0 , \\
&| W_{t+1}(a^\prime) | \leq 4\,\eta\,R_{\max}\,\norm{y_{i,a^\prime}}_{1} , \\
&\Var_t[W_{t+1}(a^\prime)] \leq 2 \eta^2\,R^2_{\max} \sum_{j \in [K] \backslash \{i\}} y^2_{i,a^\prime}(j)\,\pit(j)\,(1 - \pit(j)) .
\end{align*}
\end{lemma}

\begin{proof}
First, we have that
\begin{align*}
W_{t+1}(a^\prime) & = \ztt(a^\prime) - \E_t[\ztt(a^\prime)] = [\phi\,\thtt](a^\prime) - \E_t[[\phi\,\thtt](a^\prime)] \\
& = \langle \phi_{a^\prime}, \eta\,\phi^\top H_t\,(\hat{r}_t-r) \rangle = \eta\,[\phi \phi_{a^\prime}]^\top H_t\,(\hat{r}_t-r) \\
& = \eta\,y_{i,a^\prime}^\top H_t\,(\hat{r}_t-r) \tag{since $\textbf{1}^\top H_t=0$}.
\end{align*}
Since $y_{i,a^\prime}$ and $H_t$ are $\mathcal{F}_t$-measurable and the importance-weighted reward estimate is unbiased, we have that
\begin{align*}
\E_t[W_{t+1}(a^\prime)] = \eta y_{i,a^\prime}^\top H_t\,\E_t[\hat{r}_t-r] = 0.
\end{align*}
Then, we have that
\begin{align*}
|W_{t+1}(a^\prime)| &\leq \eta\,\norm{y_{i,a^\prime}}_1\,\norm{H_t(\hat{r}_t-r)}_\infty \tag{by H{\"o}lder's inequality}\\
&\leq 4\,\eta\,R_{\max}\,\norm{y_{i,a^\prime}}_1 .
\end{align*}
Since all entries of $\phi$ are bounded, $y_{i,a^\prime}$ is bounded and thus $|W_{t+1}(a^\prime)|$ is bounded.

Next, we upper-bound the conditional variance. Let $a_t\sim\pit$ denote the action sampled at iteration $t$. By the definition of the importance-weighted reward estimate,
\begin{align*}
H_t\hat{r}_t=R_t(a_t)(e_{a_t}-\pit).
\end{align*}
Therefore,
\begin{align}
\Var_t[W_{t+1}(a^\prime)] & = \eta^2\Var_t\left[y_{i,a^\prime}^\top H_t\hat{r}_t\right] \nonumber\\
& = \eta^2\Var_t\left[R_t(a_t)\left(y_{i,a^\prime}(a_t)-\dpd{\pit,y_{i,a^\prime}}\right)\right] \nonumber\\
& \leq \eta^2R_{\max}^2\E_t\left[\left(y_{i,a^\prime}(a_t)-\dpd{\pit,y_{i,a^\prime}}\right)^2\right] \nonumber\\
& = \eta^2R_{\max}^2\sum_{1\leq j<k\leq K}\pit(j)\,\pit(k)\left(y_{i,a^\prime}(j)-y_{i,a^\prime}(k)\right)^2. \label{eq:variance-general}
\end{align}
Using $(x-y)^2\leq2x^2+2y^2$, we obtain
\begin{align*}
\Var_t[W_{t+1}(a^\prime)] &\leq 2\eta^2R_{\max}^2\sum_{1\leq j<k\leq K}\pit(j)\,\pit(k)\left(y^2_{i,a^\prime}(j)+y^2_{i,a^\prime}(k)\right)\\
&=2\eta^2R_{\max}^2\sum_{j=1}^{K}y^2_{i,a^\prime}(j)\pit(j)\,(1-\pit(j))\\
&=2\eta^2R_{\max}^2\sum_{j\in[K]\backslash\{i\}}y^2_{i,a^\prime}(j)\pit(j)\,(1-\pit(j)),
\end{align*}
where the last equality follows from $y_{i,a^\prime}(i)=0$. This completes the proof.
\end{proof}

\begin{corollary}
\label{lemma:stochastic_noise_difference_bound}
For a pair of arbitrary actions $a$ and $a^\prime$, let $y_{k,a,a^\prime}\coloneq(\phi-\textbf{1}\phi_k^\top)(\phi_a-\phi_{a^\prime})\in\mathbb{R}^K$ for a specific $k\in[K]$. Then, it holds that
\begin{align*}
&\E_t[W_{t+1}(a)-W_{t+1}(a^\prime)] = 0 , \\
&|W_{t+1}(a)-W_{t+1}(a^\prime)| \leq 4\,\eta\,R_{\max}\,\norm{y_{k,a,a^\prime}}_{1} , \\
&\Var_t[W_{t+1}(a)-W_{t+1}(a^\prime)] \leq 2\eta^2\,R^2_{\max} \sum_{j\in[K]\backslash\{k\}}y^2_{k,a,a^\prime}(j)\,\pit(j)\,(1-\pit(j)) .
\end{align*}
\end{corollary}

\begin{proof}
Define $\widetilde{W}_{t+1}(a,a^\prime)\coloneq W_{t+1}(a)-W_{t+1}(a^\prime)$. Then, we have that
\begin{align*}
\widetilde{W}_{t+1}(a,a^\prime) & = \ztt(a)-\ztt(a^\prime)-\E_t[\ztt(a)]+\E_t[\ztt(a^\prime)] \\
&= [\phi\,\thtt](a)-[\phi\,\thtt](a^\prime)-\E_t[[\phi\,\thtt](a)]+\E_t[[\phi\,\thtt](a^\prime)] \\
&= \langle\phi_a-\phi_{a^\prime},\eta\,\phi^\top H_t(\hat r_t-r)\rangle \\
&= \langle\phi_a-\phi_{a^\prime},\eta\,(\phi-\textbf{1}\phi_k^\top)^\top H_t(\hat r_t-r)\rangle \tag{since $\textbf{1}^\top H_t=0$}\\
&= \eta\,y_{k,a,a^\prime}^\top H_t(\hat r_t-r).
\end{align*}
Since $y_{k,a,a^\prime}(k)=0$, applying the same mean, boundedness, and conditional-variance arguments used in the proof of~\cref{lemma:stochastic_noise_bound} with $i=k$ and coefficient vector $y_{k,a,a^\prime}$, gives the stated results.
\end{proof}

\begin{lemma}\label{lemma:linear_bandit_sg_bounded}
\cref{alg:sto_spg} ensures that
\begin{equation}
    \norm{\frac{d \dpd{\pi_\theta, \hat{r}}}{d \theta}} \leq \sqrt{2\,\lambda_{\max}(\phi\,\phi^\top)\,R_{\max}}.
\end{equation}
\end{lemma}
\begin{proof}
Let $a^\prime \sim \pi_\theta$ be the sampled action. By the definition of $\hat{r}$,
\[
    \left(\operatorname{diag}(\pi_\theta)-\pi_\theta\pi_\theta^\top\right)\hat r
    = R(a')\left(e_{a'}-\pi_\theta\right).
\]
Therefore, we have that
\begin{align*}
\normsq{\frac{d\dpd{\pi_\theta,\hat r}}{d\theta}}
&=
\normsq{\phi^\top\left(\operatorname{diag}(\pi_\theta)-\pi_\theta\pi_\theta^\top\right)\hat r}\\
&\le
\lambda_{\max}(\phi\phi^\top)
\normsq{\left(\operatorname{diag}(\pi_\theta)-\pi_\theta\pi_\theta^\top\right)\hat r}\\
&=
\lambda_{\max}(\phi\phi^\top)R(a')^2
\left[
    \left(1-\pi_\theta(a')\right)^2
    +\sum_{a\neq a'}\pi_\theta(a)^2
\right]\\
&\le
\lambda_{\max}(\phi\phi^\top)R_{\max}^2
\left[
    \left(1-\pi_\theta(a')\right)^2
    +\left(\sum_{a\neq a'}\pi_\theta(a)\right)^2
\right]\\
&=
2\lambda_{\max}(\phi\phi^\top)R_{\max}^2
\left(1-\pi_\theta(a')\right)^2\\
&\le
2\lambda_{\max}(\phi\phi^\top)R_{\max}^2 .
\end{align*}
Taking the square root on both sides completes the proof.
\end{proof}
\begin{restatable}[Non-Uniform Smoothness]{lemma}{nonuniformsmoothness}
\label{lemma:lpg_ns}
For all $\theta \in \R^d$, the spectral radius of Hessian matrix $\frac{d^2\{ \dpd{\pitheta, r}\}}{d \theta^2} \in \R^{d \times d}$ is upper bounded by $3\,\lambda_{\max}(\phi\,\phi^\top) \norm{\frac{d \dpd{\piz, r}}{d z}}$. That is, for all $y \in \R^{d}$, 
\begin{align*}
   \abs{y^\top\,\frac{d^2\{ \dpd{\pitheta, r}\}}{d \theta^2}\,y} \leq 3\,\lambda_{\max}(\phi\,\phi^\top) \norm{\frac{d \dpd{\piz, r}}{d z}}\,\normsq{y},
\end{align*}
where $\piz \coloneq \text{softmax}(z)$ and $z = \phi\,\theta$.
\end{restatable}
\begin{proof}
Following the initial proof of~\cref{lemma:smoothness_expected_reward_log_linear_policy}, let $S \coloneq S(r, \theta) \in \R^{d \times d}$ be the second derivative of the map $\theta \to \dpd{\pitheta, r}$. Then, we have that
\begin{align*}
    S &= \frac{d}{d \theta} \left\{\frac{d \dpd{\pitheta, r}}{d \theta}\right\} = \frac{d}{d \theta} \left\{\phi^\top H(\pitheta)\,r \right\}.
\end{align*}
For fixed $i, j \in [d]$, we have that
\begin{align*}
    S_{i, j} &= \frac{d\,[\phi^\top H(\pitheta)\,r](i)}{d \theta(j)} \\
    &= \frac{d\,[\sum_{a=1}^K \phi_{a, i}\,\pitheta(a)\,(r(a) - \dpd{\pitheta, r})}{d \theta(j)} \\
    &= \sum_{a=1}^K \phi_{a, i}\,\frac{\pitheta(a)}{d \theta(j)}\,(r(a) - \dpd{\pitheta, r}) - \sum_{a=1}^{K} \phi_{a, i}\,\pitheta(a)\,\sum_{a^\prime = 1}^K \frac{d \pitheta(a^\prime)}{d \theta(j)}\,r(a^\prime).
\end{align*}
For all $a \in [K]$ and $j \in [d]$, we have that
\begin{align*}
    \frac{d \pitheta(a)}{d \theta(j)} &= \frac{d}{d \theta(j)} \left\{\frac{\exp([\phi\,\theta](a))}{\sum_{a^\prime \in [K]}\exp([\phi\,\theta](a^\prime))} \right\} \\
    &= \frac{\frac{d\exp([\phi\,\theta](a)) }{d \theta(j)} \sum_{a^\prime \in [K]} \exp([\phi\,\theta](a^\prime)) - \exp([\phi\,\theta](a))\,\frac{d \sum_{a^\prime \in [K]} \exp([\phi\,\theta](a^\prime))}{d\theta(j)} }{(\sum_{a^\prime \in [K]}\exp([\phi\,\theta](a^\prime)))^2} \\
    &= \frac{\exp([\phi\,\theta](a))\,\phi_{a, j}\,\sum_{a^\prime \in [K]} \exp([\phi\,\theta](a^\prime)) - \exp([\phi\,\theta](a))\,\sum_{a^\prime \in [K]} \exp([\phi\,\theta](a^\prime))\,\phi_{a^\prime, j} }{(\sum_{a^\prime \in [K]}\exp([\phi\,\theta](a^\prime)))^2} \\
    &= \frac{\exp([\phi\,\theta](a))\,\phi_{a, j}  - \exp([\phi\,\theta](a))\,\sum_{a^\prime \in [K]} \pitheta(a^\prime)\,\phi_{a^\prime, j} }{\sum_{a^\prime \in [K]}\exp([\phi\,\theta](a^\prime))} \\
    &= \pitheta(a)\,\left(\phi_{a, j} - \sum_{a^\prime \in [K]} \pitheta(a^\prime)\,\phi_{a^\prime, j} \right)
\end{align*}
Combining the above inequalities, 
\begin{align*}
    S_{i, j} &= \sum_{a=1}^K \phi_{a, i}\,\pitheta(a)\,(r(a) - \dpd{\pitheta, r})\,\phi_{a, j} - \sum_{a=1}^K \phi_{a, i}\,\pitheta(a)\,(r(a) - \dpd{\pitheta, r})\,\sum_{a^\prime = 1}^K \pitheta(a^\prime)\,\phi_{a^\prime,\,j} \\
    &\quad - \sum_{a=1}^{K} \phi_{a, i}\,\pitheta(a)\,\sum_{a^\prime=1}^K \pitheta(a^\prime)\,\left(\phi_{a^\prime,\,j} - \sum_{a^{\prime\prime} =1}^K \pitheta(a^{\prime\prime})\,\phi_{a^{\prime\prime},\,j} \right)\,r(a^\prime).
\end{align*}
To show the bound on the spectral radius of $S$, pick $y \in \R^d$. Then, we have that
\begin{align*}
    \abs{y^\top\,S\,y} &= \abs{\sum_{i=1}^d \sum_{j=1}^d S_{i, j}\,y(i)\,y(j)} \\
    &= \left|\sum_{i=1}^d \sum_{j=1}^d \sum_{a=1}^K y(i)\,\phi_{a, i}\,\pitheta(a)\,(r(a) - \dpd{\pitheta, r})\,\phi_{a, j}\,y(j) \right. \\
    &\quad \left. -\sum_{i=1}^d \sum_{j=1}^d \sum_{a=1}^K        y(i)\,\phi_{a, i}\,\pitheta(a)\,(r(a) - \dpd{\pitheta, r})\,\sum_{a^\prime=1}^K \pitheta(a^\prime)\,\phi_{a^\prime, j}\,y(j)  \right. \\
    &\quad \left. -\sum_{i=1}^d \sum_{j=1}^d \sum_{a=1}^K y(i)\,\phi_{a, i}\,\pitheta(a)\,\sum_{a^\prime=1}^K \pitheta(a^\prime)\,\left(\phi_{a^\prime, j} - \sum_{a^{\prime\prime} =1}^K \pitheta(a^{\prime\prime})\,\phi_{a^{\prime\prime},\,j} \right)\,r(a^\prime)\,y(j) \right| \\
    &= \left|\sum_{a=1}^K [\phi\,y](a)\,\pitheta(a)\,(r(a) - \dpd{\pitheta, r})\,[\phi\,y](a)\right. \\
    &\quad -\left. \sum_{a=1}^K [\phi\,y](a)\,\pitheta(a)\,(r(a) - \dpd{\pitheta, r})\,\sum_{a^\prime = 1}^K \pitheta(a^\prime)\,[\phi\,y](a^\prime) \right. \\
    &\quad -\left. \sum_{a=1}^K [\phi\,y](a)\,\pitheta(a)\,\sum_{a^\prime=1}^K \pitheta(a^\prime)\,r(a^\prime)\,\left([\phi\,y](a^\prime) - \sum_{a^{\prime\prime} =1}^K \pitheta(a^{\prime\prime})\,[\phi\,y] (a^{\prime\prime) } \right) \right|.
\end{align*}
By defining that $H(\pitheta) \coloneq \text{diag}(\pitheta) - \pitheta\,\pitheta^\top \in \R^{K \times K}$, we then have that
\begin{align*}
    \abs{y^\top\,S\,y} &= \abs{\parens*{H(\pitheta)\,r}^\top \parens*{\phi\,y \odot \phi\,y} - \parens*{H(\pitheta)\,r}^\top (\phi\,y)\,(\pitheta^\top \phi\,y) - \parens*{\pitheta^\top \phi\,y}\,\parens*{H(\pitheta)\,\phi\,y}^\top r} \tag{since $\odot$ is the Hadamard (component-wise) product} \\
    &= \abs{\parens*{H(\pitheta)\,r}^\top \parens*{\phi\,y \odot \phi\,y} - 2\,\parens*{H(\pitheta)\,r}^\top (\phi\,y)\,(\pitheta^\top \phi\,y)} \\
    &\leq \abs{\parens*{H(\pitheta)\,r}^\top \parens*{\phi\,y \odot \phi\,y}} + 2\,\abs{\parens*{H(\pitheta)\,r}^\top (\phi\,y)\,(\pitheta^\top \phi\,y)}  \tag{by the triangle inequality} \\
    &\leq \|H(\pitheta)\,r\|_\infty \|\phi\,y \odot \phi\,y\|_1 + 2\,\norm{H(\pitheta)\,r}\,\norm{\phi\,y}\,\|\pitheta \|_1 \|\phi\,y\|_\infty \tag{by H\"older's inequality}\\
    &\leq 3 \norm{H(\pitheta)\,r}\,\normsq{\phi\,y} \tag{since $\|\cdot \|_\infty \leq \norm{\cdot}$, $\|\phi\,y \odot \phi\,y\|_1 = \normsq{\phi\,y}$, and $\|\pitheta\|_1 \leq 1$} \\
    &\leq 3\,\lambda_{\max}(\phi\,\phi^\top)\,\norm{H(\pitheta)\,r}\,\normsq{y} \\
    &= 3\,\lambda_{\max}(\phi\,\phi^\top)\,\norm{\frac{d \dpd{\piz, r}}{d z}}\,\normsq{y}.
\end{align*}
\end{proof}

\begin{restatable}[Strong Growth Condition]{lemma}{stronggrowthcondition}
\label{lemma:sgc}
\cref{alg:sto_spg} ensures that for all $\theta \in \R^d$,
\begin{align*}
    \E\left[\normsq{\frac{d \dpd{\pitheta, \hat{r}}}{d \theta}} \right] \leq \frac{8\,R_{\max}^3\,K^{3/2}\,\lambda_{\max}(\phi\,\phi^\top)}{\Delta^2}\,\norm{\frac{d\,\dpd{\bar{\pi}_z, r}}{d z}}
\end{align*}
where $\bar{\pi}_z \coloneq \mathrm{softmax}(z)$ and $z = \phi\,\theta$.
\end{restatable}

\begin{proof}
\begin{align*}
\E\left[\normsq{\frac{d \dpd{\pitheta, \hat{r}}}{d \theta}}\right] 
&= \E\left[\normsq{\phi^\top\,(\text{diag}(\pitheta) - \pitheta \pitheta^\top) \hat{r}}\right] \tag{by the chain rule}\\
&\leq \lambda_{\max}(\phi\,\phi^\top) \E\left[\normsq{(\text{diag}(\pitheta) - \pitheta \pitheta^\top) \hat{r}}\right]  \\
&= \lambda_{\max}(\phi\,\phi^\top) \E\left[\normsq{\frac{d \dpd{\bar{\pi}_z, \hat{r}}}{d z}}\right] \tag{since $\bar{\pi}_z = \text{softmax}(z)$} \\
&\leq \frac{8\,R^3_{\max}\,K^{3/2}\,\lambda_{\max}(\phi\,\phi^\top)}{\Delta^2}\,\norm{\frac{d \dpd{\bar{\pi}_z, r}}{d z}}_2. \tag{by \cref{lemma:sgc_multi_arm}}
\end{align*}
\end{proof}

\subsection{Technical Tools}

\begin{theorem}[Doob's supermartingale convergence~\citep{doob2012measure}]
\label{theorem:doobs}
If $\{M_n\}_{n \geq 1}$ is an $\{\gF_n\}_{n \geq 1}$-adapted sequence such that $\E[M_{n+1} \mid \gF_n] \leq M_n$ and $\sup_t \E[\abs{M_n}] < \infty$, then almost surely, $M_\infty \coloneq \limsup M_n$ exists and is finite in expectation. That is, almost surely, $M_n \to M_\infty$ and $\E[\abs{M_\infty}] < \infty$.
\end{theorem}

\begin{lemma}[Extended Borel--Cantelli]\label{lemma:extended_borel_cantelli}
Suppose $\{\gF_n\}_{n \geq 1}$ is a filtration and $E_n \in \gF_n$. Then, almost surely,
\begin{align*}
   \{E_n \; \text{occurs infinitely often} \} = \left\{ \sum_{n=1}^\infty \mathbb{P} ( E_n\,\mid\,\gF_{n-1}) = \infty
   \right\}.
\end{align*}
\end{lemma}
In the context of~\cref{alg:sto_spg}, fix any action $a\in[K]$ and define $E_1^a\coloneq\varnothing$ and $E_{t+1}^a\coloneq\{a_t=a\}$ for all $t\ge1$. Then, $E_t^a\in\mathcal F_t$ and $\mathbb P(E_{t+1}^a\mid\mathcal F_t)=\pi_{\theta_t}(a)$. Since $N_\infty(a)=\infty$ if and only if $E_t^a$ occurs infinitely often, by the Borel--Cantelli lemma (\cref{lemma:extended_borel_cantelli}), almost surely, $\{N_\infty(a)=\infty\}=\left\{\sum_{t=1}^{\infty}\pi_{\theta_t}(a)=\infty\right\}$ and $\{N_\infty(a)<\infty\}=\left\{\sum_{t=1}^{\infty}\pi_{\theta_t}(a)<\infty\right\}$

\begin{lemma}[{\citealp[Theorem C.3]{mei2023stochastic}}]
\label{lemma:martingale}
Suppose $\{\mathcal{F}_n\}_{n\geq0}$ is a filtration and $\{M_n\}_{n\geq1}$ is an adapted sequence such that $|M_n|\leq\frac{1}{2}$ for all finite $n\geq1$. Define that
\begin{align*}
S_n\coloneq\sum_{t=1}^n\left(\E[M_t\mid\mathcal{F}_{t-1}]-M_t\right), \quad \text{and} \quad V_n\coloneq\sum_{t=1}^n\Var(M_t\mid\mathcal{F}_{t-1}).
\end{align*}
Then, for all $\delta\in(0,1)$, we have that
\begin{align*}
\mathbb{P} \left( \exists n: S_n \geq 6 \sqrt{ \left(V_n + \frac{4}{3}\right) \log \left( \frac{V_n + 1}{\delta}\right)} + 2 \log(\frac{1}{\delta}) + \frac{4}{3} \log 3 \right) \leq \delta.
\end{align*}
\end{lemma}

\begin{lemma}[{\citet[Lemma 5]{lu2024towards}}]
\label{lemma:bound_gtheta_zeta}
Assuming that $f$ is $L_1$-non-uniform smooth and the stochastic gradient is bounded, i.e. $\norm{\hgrad{\thetat}} \leq B$, \cref{alg:sto_spg} with $\etat \in (0, \frac{1}{L_1\,B})$ ensures that
\begin{align*}
    \abs*{f(\thetatt) - f(\thetat) - \dpd{\gradf{\thetat}, \thetatt - \thetat}} \leq \frac{1}{2}\,\frac{L_1\,\norm{\gradf{\thetat}}}{1 - L_1\,B\,\etat}\,\normsq{\thetatt - \thetat},
\end{align*}
where $f(\theta) \coloneq \dpd{\pitheta, r}$, $\tilde{f}(\theta) \coloneq \dpd{\pitheta, \hat{r}}$ and $\gradf{\theta} \coloneq \phi^\top\,(\mathrm{diag}(\pitheta) - \pitheta\,\pitheta^\top) r$.
\end{lemma}

\begin{lemma}[{\citet[Lemma 3]{mei2020global}}]
\label{lemma:non_uniform_l}
Under~\cref{assumption:no_identical_arms}, it holds that
\begin{align*}
    \norm{\frac{d \dpd{\piz, r}}{d z}} \geq \bar{\pi}_z(a^\star)\,\dpd{\pi^\star - \bar{\pi}_z,  r},
\end{align*}
where $\pi^\star \coloneq \argmax_{\pi \in \Delta_K} \dpd{\pi, r}$ and $\bar{\pi}_z \coloneq \mathrm{softmax}(z)$ for $z \in \R^K$.
\end{lemma}

\begin{lemma}[{\citet[Lemma 4.3]{mei2023stochastic}}]\label{lemma:sgc_multi_arm}
\cref{alg:sto_spg} ensures that
\begin{align*}
    \E\left[\normsq{\frac{d \dpd{\bar{\pi}_z, \hat{r}}}{d z}} \right] \leq \frac{8\,R_{\max}^3\,K^{3/2}}{\Delta^2}\,\norm{\frac{d \dpd{\bar{\pi}_z,r}}{d z}}
\end{align*}
where $\Delta \coloneq \min_{i \neq j}\abs{r(i) - r(j)}$ and $\bar{\pi}_z \coloneq \mathrm{softmax}(z)$ for $z \in \R^K$.
\end{lemma}

\section{Proofs of \texorpdfstring{\cref{sec:stochastic_linear_bandits_with_arbitrary_learning_rates}}{}}

\subsection{Guarantee of Global Convergence}
\label{appendix:stochastic_linear_bandits_with_arbitrary_learning_rates_proof}

\subsubsection{Proofs of \texorpdfstring{\cref{theorem:stochastic_linear_bandits_with_arbitrary_learning_rates}}{}}
Here, we prove~\cref{theorem:stochastic_linear_bandits_with_arbitrary_learning_rates} via strong induction. We consider a set of actions $\mathcal{A}_\infty \subseteq [K]$ defined as $\mathcal{A}_\infty \coloneq \{ a \in [K] \mid N_\infty(a) = \infty \}$ where $N_t(a)$ is the number of times action $a$ has been sampled until iteration $t$ and $N_\infty(a) \coloneq \lim_{t \to \infty} N_t (a)$. Therefore, $\mathcal{A}_\infty$ is the set of actions that are sampled infinitely many times as $t \to \infty$. Defining $K_\infty \coloneq \abs{\mathcal{A}_\infty} \leq K$, we can then sort the action indices in $\mathcal{A}_\infty$ in descending order of the reward. In particular,
\begin{align*}
    r(i_{1}) > r(i_{2}) > \cdots > r(i_{K_\infty - 1}) > r(i_{K_\infty}) .
\end{align*}

\stochasticlinearbanditswitharbitrarylearningrates*
\begin{proof}
According to~\cref{lemma:2_arms_pulled_infinitely_many_times}, we have almost surely, $K_\infty \coloneq \abs{\mathcal{A}_\infty} \geq 2$. Moreover, according to \cref{lemma:a_star_in_A_infty}, $a^\star \in \mathcal{A}_\infty$ (i.e., $i_1 = a^\star$) almost surely. We then prove by induction that $\lim_{t \to \infty} \frac{\pit(a^\star)}{\pit(i_m)} = \infty$ from $m = K_\infty$ to $m=2$.\\
\paragraph{Base Case:} $m = K_\infty$. We can invoke~\cref{lemma:induction_on_A_infty} with $m = K_\infty$. Since $i_{K_\infty}$ is the action with the smallest reward that is sampled infinitely often, and thus $[m + 1, K_\infty] = \emptyset$, the conditions in~\cref{lemma:induction_on_A_infty} are satisfied. Therefore, we have that almost surely,
\begin{align*}
    \lim_{t \to \infty} \frac{\pit(a^\star)}{\pit(i_{K_\infty})} = \infty .
\end{align*}
\paragraph{Induction Hypothesis:} Suppose $\lim_{t \to \infty} \frac{\pit(a^\star)}{\pit(i_{p})} = \infty$ almost surely for all $p \in [m+1, K_\infty]$.\\
\paragraph{Inductive Step:} According to~\cref{lemma:induction_on_A_infty} with $m < K_\infty$, we have that almost surely,
\begin{align*}
    \lim_{t \to \infty} \frac{\pit(a^\star)}{\pit(i_{m})} = \infty .
\end{align*}
Therefore, the induction proof concludes that almost surely,
\begin{align*}
     \forall a \in \mathcal{A}_\infty\backslash \{a^\star\}: \; \lim_{t \to \infty} \frac{\pit(a^\star)}{\pit(a)} = \infty . \numberthis \label{eq:lim_ratio_infinite}
\end{align*}
Moreover, according to the Borel--Cantelli lemma (\cref{lemma:extended_borel_cantelli}), we have that almost surely
\begin{align*}
    \forall a \notin \mathcal{A}_\infty: \; \sum_{t=1}^\infty \pit(a) < \infty .
\end{align*}
This implies that almost surely,
\begin{align*}
    \forall a \in [K] \backslash \mathcal{A}_\infty: \;  \lim_{t \to \infty} \pit(a) = 0 \implies \lim_{t \to \infty} \sum_{a \in [K] \backslash \mathcal{A}_\infty} \pit(a) = 0 . \numberthis \label{eq:lim_ratio_finite}
\end{align*}
We also have that
\begin{align*}
    \pit (a^\star) = 1 - \sum_{a \in [K] \backslash \mathcal{A}_\infty} \pit (a) - \sum_{a \in \mathcal{A}_\infty \backslash \{a^\star\}} \pit (a) .
\end{align*}
This further implies that
\begin{align*}
    \pit (a^\star) = \frac{1 - \sum_{a \in [K] \backslash \mathcal{A}_\infty} \pit (a)}{1 + \sum_{a \in \mathcal{A}_\infty \backslash \{a^\star\} } \frac{\pit (a)}{\pit (a^\star)}} .
\end{align*}
Using~\cref{eq:lim_ratio_infinite,eq:lim_ratio_finite} yields that almost surely,
\begin{align*}
    \lim_{t \to \infty} \pit (a^\star) = \lim_{t \to \infty} \frac{1 - \sum_{a \in [K] \backslash \mathcal{A}_\infty} \pit (a)}{1 + \sum_{a \in \mathcal{A}_\infty \backslash \{a^\star\}} \frac{\pit (a)}{\pit (a^\star)}} = 1 .
\end{align*}
This completes the proof.
\end{proof}

\subsubsection{Proofs of Lemmas}\label{subsubsec:proofs_lemmas_arbitrary_learning_rates}

To prove the lemmas used in the proof of~\cref{theorem:stochastic_linear_bandits_with_arbitrary_learning_rates}, we first consider the ratio $\frac{\pit(a^\star)}{\pit(k)}$ where $a^\star$ is the optimal action and $k \in \mathcal{A}_\infty$ is a suboptimal action that is sampled infinitely many times. We can rewrite the ratio using the difference of the logits:
\begin{align}
    \frac{\pit(a^\star)}{\pit(k)} &= \exp\parens*{[\phi\,\tht](a^\star) - [\phi\,\tht](k)} = \exp (z_t(a^\star) - z_t(k)) \label{eq:ratio_a_star_k}.
\end{align}
Next, we show that the difference in logits can be lower-bounded. To start, we first introduce some useful notations. For each pair $i, j \in [K]$ and $i < j$, we define that
\begin{align*}
    \bar{\phi}_{k, i, j} \coloneq \abs{ \dpd{\phi_i - \phi_j, \phi_{a^\star} - \phi_k} } .
\end{align*}
Using this, under~\cref{assumption:general_feature_conditions}, we define $\mathcal{X} (k) \coloneq \left\{ (i, j) \mid 1 \leq i < j \leq K, \bar{\phi}_{k, i, j} > 0 \right\}$ as the set that contains all the pairs $(i, j)$ such that $\bar{\phi}_{k, i, j} > 0$. We can further split this set into two subsets based on whether or not $i < k$:
\begin{align*}
    \mathcal{X}^{+} (k) &\coloneq \left\{ (i, j) \mid 1 \leq i < j \leq K, i < k, \bar{\phi}_{k, i, j} > 0 \right\} , \\
    \mathcal{X}^{-} (k) &\coloneq \left\{ (i, j) \mid 1 \leq i < j \leq K, i \geq k, \bar{\phi}_{k, i, j} > 0 \right\} .
\end{align*}
We then define the following summations within each set of pairs for any iteration $s$:
\begin{align*}
    \Gamma_s^+ (k) \coloneq \sum_{(i, j) \in \mathcal{X}^{+}_k} \pis(i)\,\pis(j) , \quad
    \Gamma_s^- (k) \coloneq \sum_{(i, j) \in \mathcal{X}^{-}_k} \pis(i)\,\pis(j) . 
\end{align*}
Then, we proceed to introduce the following helper lemma.

\begin{lemma} \label{lemma:lower_bound_logit_diff}
    Suppose~\cref{assumption:no_identical_arms,assumption:general_feature_conditions} hold. For any suboptimal action $k \in [K] \backslash \{a^\star\}$, there exist constants $C_{k, 1}, C_{k, 3}, C_{k, 4} > 0$ and $C_{k, 2} \geq 0$ such that for all $t \geq 1$ and any $\delta \in (0, 1)$, it holds that with probability $1 - \delta$,
    \begin{align*}
    \MoveEqLeft
    z_t(a^\star) - z_t(k) \geq z_1(a^\star) - z_1(k) + \eta\,\sum_{s=1}^{t-1} \left( C_{k, 1}\,\Gamma^{+}_{s}(k) - C_{k, 2}\,\Gamma^{-}_s(k) \right) \\
    &\quad - 6\,C_{k, 4} \sqrt{ \left(  C_{k, 3}\,\sum_{s=1}^{t-1} \left(\Gamma^{+}_s(k) + \Gamma^{-}_s(k) \right) + \frac{4}{3} \right) \log \left( \frac{ C_{k, 3}\,\sum_{s=1}^{t-1} \left(\Gamma^{+}_s(k) + \Gamma^{-}_s(k) \right) + 1  }{ \delta } \right) } \\
    &\quad - 2\,C_{k, 4}  \log(\frac{1}{\delta}) - \frac{4\,C_{k, 4}\,\log 3}{3} \numberthis \label{eq:logit_diff_a_star_k}
    \end{align*}
\end{lemma}

\begin{proof}
    Using the decomposition of the stochastic process in \cref{subsec:decomposition_of_stochastic_process} with $a_1 = a^\star$ and $a_2 = k$ and recursing \cref{eq:stochastic_logit_difference_decomposition} from $s = t$ to $s=1$, we have that
    \begin{align}
        z_t(a^\star) - z_t(k) = z_1(a^\star) - z_1(k) + \underbrace{\sum_{s=1}^{t-1} \left[P_s(a^\star) - P_s(k) \right]}_{\text{(i)}} + \underbrace{\sum_{s=1}^{t-1} \left[ W_{s+1}(a^\star) - W_{s+1}(k)\right]}_{\text{(ii)}}.
    \numberthis \label{eq:logit_diff_decomp_a_star_over_k}
    \end{align}
    We first investigate Term (i), which is the cumulative progress.
    \begin{align*}
    \MoveEqLeft
    P_s(a^\star) - P_s(k) = \eta\,\sum_{i \in [K]} \dpd{\phi_i, \phi_{a^\star} - \phi_{k}}\,\pis(i)\,(r(i) - \dpd{\pis, r}) \tag{following~\cref{eq:progress_a_star_a}} \\
    &= \eta\,\sum_{i \in [K]} \dpd{\phi_i, \phi_{a^\star} - \phi_{k}}\,\pis(i)\,\sum_{j \in [K]} \pis(j)\,[r(i) - r(j)] \\
    &= \eta\,\sum_{i \in [K]}\,\sum_{j \in [K]} \dpd{\phi_i, \phi_{a^\star} - \phi_{k}}\,\pis(i)\,\pis(j)\,[r(i) - r(j)] \numberthis \label{eq:progress_decomp_i_j}
    \end{align*}
    By symmetry of the double summations, we can swap the indices of $i$ and $j$ and obtain that
    \begin{align*}
    \MoveEqLeft
    P_s(a^\star) - P_s(k) = \eta\,\sum_{i \in [K]}\,\sum_{j \in [K]} \dpd{\phi_j, \phi_{a^\star} - \phi_{k}}\,\pis(j)\,\pis(i)\,[r(j) - r(i)] \\
    &= \eta\,\sum_{i \in [K]}\,\sum_{j \in [K]} \dpd{-\phi_j, \phi_{a^\star} - \phi_{k}}\,\pis(j)\,\pis(i)\,[r(i) - r(j)]  \numberthis \label{eq:progress_decomp_j_i}
    \end{align*}
    Combining~\cref{eq:progress_decomp_i_j,eq:progress_decomp_j_i} together implies that
    \begin{align*}
        \MoveEqLeft
        P_s(a^\star) - P_s(k) = \frac{\eta}{2}\,\sum_{i \in [K]}\,\sum_{j \in [K]}\,\dpd{\phi_i - \phi_j, \phi_{a^\star} - \phi_{k}}\,\pis(i)\,\pis(j)\,(r(i) - r(j)) \\
        &= \eta\,\sum_{\substack{i, j \in [K] \\ i<j}} \dpd{\phi_i - \phi_j, \phi_{a^\star} - \phi_{k}}\,\pis(i)\,\pis(j)\,[r(i) - r(j)] \\
        &= \eta\,\underbrace{ \sum_{\substack{i, j \in [K] \\ i<j, i<k}} \dpd{\phi_i - \phi_j, \phi_{a^\star} - \phi_{k}}\,\pis(i)\,\pis(j)\,[r(i) - r(j)] }_{\eqcolon G_s(k)}\\
        &\quad + \eta\,\underbrace{\sum_{\substack{i, j \in [K] \\ i<j, i \geq k}} \dpd{\phi_i - \phi_j, \phi_{a^\star} - \phi_{k}}\,\pis(i)\,\pis(j)\,[r(i) - r(j)]}_{\eqcolon L_s(k)} .
    \end{align*}
    According to~\cref{assumption:general_feature_conditions}, we have that
    \begin{align*}
        G_s(k) &\coloneq \sum_{\substack{i, j \in [K] \\ i<j, i < k}} \dpd{\phi_i - \phi_j, \phi_{a^\star} - \phi_{k}}\,\pis(i)\,\pis(j)\,[r(i) - r(j)] \\
        &= \sum_{\substack{i, j \in [K] \\ i<j, i < k}} \underbrace{\dpd{\phi_i - \phi_j, \phi_{a^\star} - \phi_{k}}}_{\geq 0}\,\pis(i)\,\pis(j)\,\underbrace{[r(i) - r(j)]}_{> 0} \\
        &\geq C_1\,\Delta_{\min}\,\underbrace{\sum_{(i, j) \in \mathcal{X}^+(k)} \pis(i)\,\pis(j)}_{\eqcolon \Gamma^{+}_{s}(k)} ,
    \end{align*}
    where $\Delta_{\min} \coloneq \min_{\substack{i, j \in [K] \\i \neq j}} \abs{r(i) - r(j)} > 0$ and $C_1 \coloneq \min_{(i, j) \in \mathcal{X}^+(k)} \bar{\phi}_{k, i, j} > 0$.\\
    By further defining $C_2\coloneq\max\left(\left\{\bar{\phi}_{k,i,j}:(i,j)\in\mathcal{X}^-(k)\right\}\cup\{0\}\right)\geq0$, we also have that
    \begin{align*}
        \abs{L_s(k)} &= \abs{\sum_{\substack{i, j \in [K] \\ i<j, i \geq k }} \dpd{\phi_i - \phi_j, \phi_{a^\star} - \phi_{k}}\,\pis(i)\,\pis(j)\,[r(i) - r(j)]} \\
        &\leq \sum_{(i, j) \in \mathcal{X}^{-}(k)} \pis(i)\,\pis(j)\,\abs{r(i) - r(j)}\,\abs{\dpd{\phi_i - \phi_j, \phi_{a^\star} - \phi_{k}}} \\
        \tag{by the triangle inequality} \\
        &\leq 2\,C_2\,R_{\max}\,\underbrace{\sum_{(i, j) \in \mathcal{X}^{-}(k)} \pis(i)\,\pis(j)}_{\eqcolon \Gamma^{-}_s (k)} .
    \end{align*}
    Therefore, we have that
    \begin{align*}
        P_s(a^\star) - P_s(k) &= \eta\,G_s(k) + \eta\,L_s(k) \\
        &\geq \eta\,G_s(k) - \eta\,\abs{L_s(k)} \\
        &\geq\eta\left(C_1\Delta_{\min}\Gamma_s^+(k)-2C_2\,R_{\max}\Gamma_s^-(k)\right) . \numberthis \label{eq:progress_k_a_star}
    \end{align*}
    
    Next, we investigate Term (ii) in~\cref{eq:logit_diff_decomp_a_star_over_k}. We first provide some useful properties of $W_s(a)$. Let $y_{i, a^\star, k} \coloneq (\phi - \textbf{1} \phi_{i}^\top) (\phi_{a^\star} - \phi_{k})$. According to~\cref{lemma:stochastic_noise_difference_bound}, we know that
    \begin{align*}
       &\E_s[W_{s+1}(a^\star) - W_{s+1}(k)] = 0 , \\
       &\abs{W_{s+1}(a^\star) - W_{s+1}(k)} \leq 4\,\eta\,R_{\max}\,\norm{y_{i, a^\star, k}^i}_{1} \leq 4\,\eta\,R_{\max}\,C_3 , \\
       &\Var[W_{t+1}(a^\star) - W_{t+1}(k)] \leq 2 \eta^2\,R^2_{\max} \sum_{\substack{j = 1 \\ j \neq i}}^{K} y^2_{i, a^\star, k}(j)\,\pit(j)\,(1 - \pit(j)) .
    \end{align*}
    Here, $C_3 \coloneq \max_{i \in [K]} \norm{y_{i, a^\star, k}}_{1}$. According to~\cref{assumption:general_feature_conditions}, \[C_3 \geq \norm{y_{a^\star, a^\star, k}}_1 \geq \dpd{\phi_{a^\star} - \phi_{k}, \phi_{a^\star} - \phi_{k}} > 0 .\] Additionally, $\{ W_{s+1}(a^\star) - W_{s+1}(k) \}_{s \geq 1}$ can be considered as a martingale difference sequence with respect to a specific filtration $\{ \gF\}_{s \geq 1}$. Since it is bounded, it can be normalized to be in the range of $[-\nicefrac{1}{2}, \nicefrac{1}{2}]$. For this, define $\widetilde{W}_{s+1}(a^\star, k) \coloneq \frac{W_{s+1}(a^\star) - W_{s+1}(k)}{8\,\eta\,R_{\max}\,C_3}$. Additionally,
    \begin{align*}
    \Var[\widetilde{W}_{s+1}(a^\star, k)] &= \frac{\Var[W_{s+1}(a^\star) - W_{s+1}(k)]}{(8\,\eta\,R_{\max}\,C_3)^2} \\
    &\leq \frac{1 }{32\,C_3^2}\,\sum_{\substack{j \in [K] \\ j \neq i}} (\dpd{\phi_j - \phi_{i}, \phi_{a^\star} - \phi_{k}})^2\,\pis(j)\,(1 - \pis(j)) \tag{by \cref{lemma:stochastic_noise_difference_bound}}\\
    &\leq \frac{1 }{32\,C_3^2}\,\sum_{\substack{j \in [K] \\ j \neq i}} (\dpd{\phi_j - \phi_{i}, \phi_{a^\star} - \phi_{k}})^2\,\pis(j) \tag{since $1 - \pis(j) \leq 1$} \\
    &\leq \frac{1 }{32\,C_3^2}\,\sum_{i \in [K]} \pis(i) \sum_{\substack{j \in [K] \\ j \neq i}} (\dpd{\phi_j - \phi_{i}, \phi_{a^\star} - \phi_{k}})^2\,\pis(j) \tag{multiply both sides by $\pis(i)$ and sum across $i \in [K]$} \\
    &= \frac{1}{16\,C_3^2}\,\sum_{ \substack{ i, j \in [K] \\ i < j} } (\dpd{\phi_j - \phi_{i}, \phi_{a^\star} - \phi_{k}})^2\,\pis(i)\,\pis(j) \\
    &\leq \frac{C_4^2}{16\,C_3^2}\,\sum_{ \substack{ i, j \in [K] \\ i < j \\ \bar{\phi}_{k, i, j} > 0} } \pis(i)\,\pis(j) \tag{let $C_4 \coloneq \max_{i, j \in [K], i < j} \abs{\bar{\phi}_{k, i, j}} > 0$}\\
    &= \frac{C_4^2}{16\,C_3^2}\,\sum_{ \substack{ i, j \in [K] \\ i < j, i < k \\ \bar{\phi}_{k, i, j} > 0} } \pis(i)\,\pis(j) + \frac{C_4^2}{16\,C_3^2}\,\sum_{ \substack{ i, j \in [K] \\ i < j, i \geq k \\ \bar{\phi}_{k, i, j} > 0} } \pis(i)\,\pis(j) \\
    &\leq \frac{C_4^2}{16\,C_3^2}\,\left( \Gamma^{+}_{s}(k) + \Gamma^{-}_{s}(k) \right) = C_5\,\left( \Gamma^{+}_{s}(k) + \Gamma^{-}_{s}(k) \right) ,
    \end{align*}
    where $C_5 \coloneq \frac{C_4^2}{16\,C_3^2}$. Using the above inequality in combination with~\cref{lemma:martingale} yields that for all $t \geq 1$ and any $\delta \in (0, 1)$, with probability $1 - \delta$,
    \begin{align*}
     \sum_{s=1}^{t-1} \widetilde{W}_{s+1}(a^\star, k) \geq &- 6\,\sqrt{ \left( C_5\,\sum_{s=1}^{t-1} \left( \Gamma^{+}_s(k) + \Gamma^{-}_s(k) \right) + \frac{4}{3}\right)\,\log \left( \frac{ C_5\,\sum_{s=1}^{t-1}\,\left( \Gamma^{+}_s(k) + \Gamma^{-}_s(k) \right) +1}{ \delta } \right) } \\
     &- 2\,\log\left(\frac{1}{\delta}\right) - \frac{4}{3} \log(3).
    \end{align*}
    Recall that $\widetilde{W}_{s+1}(a^\star, k) \coloneq \frac{W_{s+1}(a^\star) - W_{s+1}(k)}{8\,\eta\,R_{\max}\,C_3}$. Setting $C_6 \coloneq 8\,\eta\,R_{\max}\,C_3$, it holds that with probability $1 - \delta$,
    \begin{align*}
    \MoveEqLeft
    \sum_{s=1}^{t-1} \left[ W_{s+1}(a^\star) - W_{s+1}(k) \right]  \\
    &\geq  - 6\,C_6\,\sqrt{ \left( C_5\,\sum_{s=1}^{t-1}\,\left( \Gamma^{+}_s(k) + \Gamma^{-}_s(k) \right) + \frac{4}{3}\right)\,\log \left( \frac{ C_5\,\sum_{s=1}^{t-1}\,\left( \Gamma^{+}_s(k) + \Gamma^{-}_s(k) \right) +1}{ \delta } \right) } \\
    &\quad - 2\,C_6\,\log\left(\frac{1}{\delta}\right) - \frac{4\,C_6}{3} \log(3) . \numberthis \label{eq:variance_k_a_star}
    \end{align*}
    Using the above results and combining it with \cref{eq:logit_diff_decomp_a_star_over_k}, it holds that with probability $1 - \delta$,
    \begin{align*}
    \MoveEqLeft
    z_t(a^\star) - z_t(k) \\
    &= z_1(a^\star) - z_1(k) + \sum_{s=1}^{t-1} \left[ P_s(a^\star) - P_s(k) \right] + \sum_{s=1}^{t-1} \left[ W_{s+1}(a^\star) - W_{s+1}(k) \right] \\
    &\geq z_1(a^\star) - z_1(k) + \eta\,\sum_{s=1}^{t-1} \left( C_1\,\Delta_{\min} \Gamma^{+}_{s}(k) - 2\,C_2\,R_{\max}\,\Gamma^{-}_s \right) + \sum_{s=1}^{t-1} \left[ W_{s+1}(a^\star) - W_{s+1}(k) \right] \tag{using \cref{eq:progress_k_a_star} to lower-bound the progress term} \\
    &\geq z_1(a^\star) - z_1(k) + \eta\,\sum_{s=1}^{t-1} \left( C_1\,\Delta_{\min} \Gamma^{+}_{s}(k) - 2\,C_2\,R_{\max}\,\Gamma^{-}_s(k) \right) \\
    &\quad - 6\,C_6 \sqrt{ \left(  C_5\,\sum_{s=1}^{t-1} \left(\Gamma^{+}_s(k) + \Gamma^{-}_s(k) \right) + \frac{4}{3} \right) \log \left( \frac{ C_5\,\sum_{s=1}^{t-1} \left(\Gamma^{+}_s(k) + \Gamma^{-}_s(k) \right) + 1  }{ \delta } \right) } \\
    &\quad - 2\,C_6  \log(\frac{1}{\delta}) - \frac{4\,C_6\,\log 3}{3} .
    \\ \tag{using \cref{eq:variance_k_a_star} to lower-bound the noise term} 
    \end{align*}
    This completes the proof.
\end{proof}

Since $K$ is finite, for any $\delta\in(0,1/K)$, by applying the above result to each fixed $k\in[K]\setminus\{a^\star\}$ and taking a finite union bound, the bound holds simultaneously for all such $k$ with probability at least $1-K\delta$. For the following proofs, we implicitly use this simultaneous version whenever $k$ is a random action index.

Using the above lemma, we can first prove that $a^\star\in\mathcal{A}_\infty$ almost surely as follows.

\begin{lemma}\label{lemma:a_star_in_A_infty}
Under~\cref{assumption:no_identical_arms,assumption:general_feature_conditions}, \cref{alg:sto_spg} with any constant learning rate $\eta > 0$ ensures that $N_\infty(a^\star) = \infty$ (i.e., $a^\star \in \mathcal{A}_\infty)$ almost surely.
\end{lemma}

\begin{proof}
To start, we look into the ratio of $\frac{\pit(a^\star)}{\pit(i_{K_\infty})}$ where $i_{K_\infty} \coloneq \argmin_{a \in \mathcal{A}_\infty} r(a)$ is the action with the smallest reward that is sampled infinitely many times. According to~\cref{lemma:2_arms_pulled_infinitely_many_times}, $\abs{\mathcal{A}_\infty} \geq 2$ and hence $i_{K_\infty} \neq a^\star$. First, we can rewrite the ratio using the difference of the logits:
\begin{align}
\frac{\pit(a^\star)}{\pit(i_{K_\infty})} &= \exp\parens*{[\phi\,\tht](a^\star) - [\phi\,\tht](i_{K_\infty})} = \exp (z_t(a^\star) - z_t(i_{K_\infty})) . \label{eq:ratio_a_star_i_1}
\end{align}
Using~\cref{lemma:lower_bound_logit_diff} by setting $k = i_{K_\infty}$, it holds that with probability at least $1 - K \delta$,
\begin{align*}
\MoveEqLeft
z_t(a^\star) - z_t(i_{K_\infty}) \geq z_1(a^\star) - z_1(i_{K_\infty}) + \eta\,\sum_{s=1}^{t-1} \left( C_{1}\,\Gamma^{+}_{s}(i_{K_\infty}) - C_{2}\,\Gamma^{-}_s(i_{K_\infty}) \right) \\
&\quad - 6\,C_{4} \sqrt{ \left(  C_{3}\,\sum_{s=1}^{t-1} \left(\Gamma^{+}_s(i_{K_\infty}) + \Gamma^{-}_s(i_{K_\infty}) \right) + \frac{4}{3} \right) \log \left( \frac{ C_{3}\,\sum_{s=1}^{t-1} \left(\Gamma^{+}_s(i_{K_\infty}) + \Gamma^{-}_s(i_{K_\infty}) \right) + 1  }{ \delta } \right) } \\
&\quad - 2\,C_{4} \log(\frac{1}{\delta}) - \frac{4\,C_{4}\,\log 3}{3} .
\numberthis \label{eq:logit_diff_a_star_i_1}
\end{align*}
Here, $C_1, C_3, C_4 > 0$ and $C_2 \geq 0$. We recall the following definitions.
\begin{align*}
    \mathcal{X}^{+} (i_{K_\infty}) &\coloneq \left\{ (i, j) \mid 1 \leq i < j \leq K, i < i_{K_\infty}, \bar{\phi}_{i_{K_\infty}, i, j} > 0 \right\} , \\
    \mathcal{X}^{-} (i_{K_\infty}) &\coloneq \left\{ (i, j) \mid 1 \leq i < j \leq K, i \geq i_{K_\infty}, \bar{\phi}_{i_{K_\infty}, i, j} > 0 \right\} .
\end{align*}
We also recall the following summations within each set of pairs for any iteration $s$:
\begin{align*}
    \Gamma_s^+ (i_{K_\infty}) \coloneq \sum_{(i, j) \in \mathcal{X}^{+}_{i_{K_\infty}}} \pis(i)\,\pis(j) , \quad
    \Gamma_s^- (i_{K_\infty}) \coloneq \sum_{(i, j) \in \mathcal{X}^{-}_{i_{K_\infty}}} \pis(i)\,\pis(j) . 
\end{align*}
By construction, for all $(i, j) \in \mathcal{X}^{-}_{i_{K_\infty}}$, we have that $j > i_{K_\infty}$ and hence $r(j) < r(i_{K_\infty})$. Since $i_{K_\infty}$ is the action with the smallest reward in $\mathcal{A}_\infty$, $j \notin \mathcal{A}_\infty$. Therefore, we have that
\begin{align*}
\sum_{s=1}^{t-1} \Gamma_s^-(i_{K_\infty}) &= \sum_{s=1}^{t-1} \sum_{(i,j)\in \mathcal{X}^{-} (i_{K_\infty})} \pi_{\theta_s}(i)\,\pi_{\theta_s}(j) \\
&\leq \sum_{(i,j)\in \mathcal{X}^{-} (i_{K_\infty})} \sum_{s=1}^{t-1} \pi_{\theta_s}(j) \tag{since $\pis(i) \leq 1$} \\
&\leq \sum_{j \notin \mathcal{A}_\infty} \sum_{s=1}^{\infty} \pi_{\theta_s}(j) \tag{since $j \notin \mathcal{A}_\infty$ for all $j \in \mathcal{X}^{-} (i_{K_\infty})$} \\
&< \infty \tag{by the Borel--Cantelli lemma (\cref{lemma:extended_borel_cantelli})} .
\end{align*}
We then bound $\sum_{s=1}^{t-1} \Gamma^{+}_s$ by dividing into the following two cases.\\

\noindent \textit{Case A:} When $\lim_{t \to \infty} \sum_{s=1}^{t-1} \Gamma^{+}_s (i_{K_\infty}) < \infty$. Given that $\lim_{t \to \infty} \sum_{s=1}^{t-1} \Gamma^{-}_s (i_{K_\infty}) < \infty$, we know that the RHS of~\cref{eq:logit_diff_a_star_i_1} is bounded below (by some finite constant) as $t \to \infty$, and hence with probability at least $1 - K\delta$, \[\liminf_{t \to \infty} z_t(a^\star) - z_t(i_{K_\infty}) > - \infty .\]

\noindent \textit{Case B:} When $\lim_{t \to \infty} \sum_{s=1}^{t-1} \Gamma^{+}_s (i_{K_\infty}) = \infty$. Using the fact that \[O\!\left(\sqrt{\sum_{s=1}^{t-1}\Gamma_s^+(i_{K_\infty})\log\left(\sum_{s=1}^{t-1}\Gamma_s^+(i_{K_\infty})\right)}\right) = o\!\left(\sum_{s=1}^{t-1}\Gamma_s^+(i_{K_\infty})\right),\] we know that the RHS of~\cref{eq:logit_diff_a_star_i_1} tends to $+\infty$ as $t \to \infty$. This is sufficient to imply that with probability at least $1 - K\delta$, \[\liminf_{t \to \infty} z_t(a^\star) - z_t(i_{K_\infty}) > - \infty .\]
Combining both cases, since $\delta > 0$ can be arbitrarily small, we have that $\liminf_{t \to \infty} z_t(a^\star) - z_t(i_{K_\infty}) > - \infty$ almost surely. Therefore, there exists a large enough finite iteration $\tau \geq 1$ such that $\inf_{t \geq \tau} \frac{\pit(a^\star)}{\pit(i_{K_\infty})} \geq C > 0$ almost surely. Finally, we have that almost surely,
\begin{align*}
    \sum_{t=1}^\infty \pit(a^\star) &\geq \sum_{t=\tau}^\infty \pit(a^\star) \\
    &= \sum_{t=\tau}^\infty \frac{\pit(a^\star)}{\pit(i_{K_\infty})}\,\pit(i_{K_\infty}) \\
    &\geq C\,\sum_{t=\tau}^\infty \pit(i_{K_\infty}) = \infty. \\ \tag{by $C > 0$ and the Borel--Cantelli lemma (\cref{lemma:extended_borel_cantelli})}
\end{align*}
Using this along with the Borel--Cantelli lemma (\cref{lemma:extended_borel_cantelli}), we can conclude that $N_\infty(a^\star) = \infty$ and thus $a^\star \in \mathcal{A}_\infty$.
\end{proof}

Using the above lemmas, we can proceed to prove~\cref{lemma:induction_on_A_infty} for the strong induction proof in~\cref{theorem:stochastic_linear_bandits_with_arbitrary_learning_rates}. Recall that $K_\infty \coloneq \abs{\mathcal{A}_\infty}$ and the indices in $\mathcal{A}_\infty$ are sorted in descending order of the reward. That is, $r(i_{1}) > r(i_{2}) > \cdots > r(i_{K_\infty - 1}) > r(i_{K_\infty})$.

\begin{lemma}\label{lemma:induction_on_A_infty}
    Suppose~\cref{assumption:no_identical_arms,assumption:general_feature_conditions} hold. Consider any suboptimal action $i_m \in \mathcal{A}_\infty \backslash \{a^\star\}$. If $m = K_\infty$ or for all $p \in [m+1, K_\infty]$, $\lim_{t \to \infty} \frac{\pit(a^\star)}{\pit(i_p)} = \infty$ almost surely, then \cref{alg:sto_spg} with any constant learning rate $\eta > 0$ ensures that almost surely
    \begin{align*}
        \lim_{t \to \infty} \frac{\pit(a^\star)}{\pit(i_{m})} = \infty .
    \end{align*}
\end{lemma}

\begin{proof}
    We can rewrite the ratio using the difference of the logits:
    \begin{align*}
    \frac{\pit(a^\star)}{\pit(i_m)} &= \exp\parens*{[\phi\,\tht](a^\star) - [\phi\,\tht](i_m)} = \exp (z_t(a^\star) - z_t(i_m)) .
    \end{align*}
    Using the simultaneous version of~\cref{lemma:lower_bound_logit_diff} by setting $k = i_m$, it holds that with probability at least $1 - K\delta$,
    \begin{align*}
    \MoveEqLeft
        z_t(a^\star) - z_t(i_m) \geq z_1(a^\star) - z_1(i_m) + \eta\,\sum_{s=1}^{t-1} \left( C_{1}\,\Gamma^{+}_{s}(i_m) - C_{2}\,\Gamma^{-}_s(i_m) \right) \\
        &\quad - 6\,C_{4} \sqrt{ \left(  C_{3}\,\sum_{s=1}^{t-1} \left(\Gamma^{+}_s(i_m) + \Gamma^{-}_s(i_m) \right) + \frac{4}{3} \right) \log \left( \frac{ C_{3}\,\sum_{s=1}^{t-1} \left(\Gamma^{+}_s(i_m) + \Gamma^{-}_s(i_m) \right) + 1  }{ \delta } \right) } \\
        &\quad - 2\,C_{4} \log(\frac{1}{\delta}) - \frac{4\,C_{4}\,\log 3}{3} 
    \numberthis \label{eq:logit_diff_a_star_i_m}
    \end{align*}
    Here, $C_1, C_3, C_4 > 0$ and $C_2 \geq 0$. We recall the following definitions.
    \begin{align*}
        \mathcal{X}^{+} (i_m) &\coloneq \left\{ (i, j) \mid 1 \leq i < j \leq K, i < i_m, \bar{\phi}_{i_m, i, j} > 0 \right\} , \\
        \mathcal{X}^{-} (i_m) &\coloneq \left\{ (i, j) \mid 1 \leq i < j \leq K, i \geq i_m, \bar{\phi}_{i_m, i, j} > 0 \right\} .
    \end{align*}
    We also recall the following summations within each set of pairs for any iteration $s$:
    \begin{align*}
        \Gamma_s^+ (i_m) \coloneq \sum_{(i, j) \in \mathcal{X}^{+} (i_m)} \pis(i)\,\pis(j) , \quad
        \Gamma_s^- (i_m) \coloneq \sum_{(i, j) \in \mathcal{X}^{-} (i_m)} \pis(i)\,\pis(j) . 
    \end{align*}
\cref{lemma:sum_gamma_plus_i_m_=_infty} shows that $\lim_{t \to \infty} \sum_{s=1}^{t-1} \Gamma_s^+ (i_m)= \infty$ almost surely. Next, we will look into $\sum_{s=1}^{t-1}\Gamma_s^-(i_m)$. For every $(i,j)\in\mathcal{X}^-(i_m)$, we have that $i_m\leq i<j$, and hence $r(j)<r(i_m)$. Therefore, for every finite $t\geq1$,
\begin{align}
\sum_{s=1}^{t-1}\Gamma_s^-(i_m)
&=
\sum_{s=1}^{t-1}\sum_{\substack{(i,j)\in\mathcal{X}^-(i_m)\\j\notin\mathcal{A}_\infty}}\pi_{\theta_s}(i)\,\pi_{\theta_s}(j)
+
\sum_{s=1}^{t-1}\sum_{\substack{(i,j)\in\mathcal{X}^-(i_m)\\j\in\mathcal{A}_\infty}}\pi_{\theta_s}(i)\,\pi_{\theta_s}(j). \label{eq:decomposition_gamma_minus}
\end{align}
For the first term in~\cref{eq:decomposition_gamma_minus}, by the Borel--Cantelli lemma (\cref{lemma:extended_borel_cantelli}), there exists a finite constant $C_5>0$ such that, for every finite $t\geq1$,
\begin{align*}
\sum_{s=1}^{t-1}\sum_{\substack{(i,j)\in\mathcal{X}^-(i_m)\\j\notin\mathcal{A}_\infty}}\pi_{\theta_s}(i)\,\pi_{\theta_s}(j)
\leq
\sum_{\substack{(i,j)\in\mathcal{X}^-(i_m)\\j\notin\mathcal{A}_\infty}}\sum_{s=1}^{\infty}\pi_{\theta_s}(j)
\leq C_5.
\end{align*}
For the second term, if $m=K_\infty$, then it is zero. Otherwise, when $m<K_\infty$, according to the induction hypothesis, for any $\epsilon>0$, there exists a finite iteration $\tau_\epsilon$ such that, for every $s\geq\tau_\epsilon$ and every $(i,j)\in\mathcal{X}^-(i_m)$ with $j\in\mathcal{A}_\infty$,
$\pi_{\theta_s}(j)\leq\epsilon\pi_{\theta_s}(a^\star)$. Hence, there exists a finite constant $C_6\geq0$ such that, for every $t>\tau_\epsilon$,
\begin{align*}
\sum_{s=1}^{t-1}\sum_{\substack{(i,j)\in\mathcal{X}^-(i_m)\\j\in\mathcal{A}_\infty}}\pi_{\theta_s}(i)\,\pi_{\theta_s}(j)
\leq C_6+\epsilon|\mathcal{X}^-(i_m)|\sum_{s=1}^{t-1}\Gamma_s^+(i_m).
\end{align*}
When $m=K_\infty$, the same inequality holds by setting $\epsilon=0$ and $C_6=0$. Therefore, for every sufficiently large finite $t$,
\begin{align}
\sum_{s=1}^{t-1}\Gamma_s^-(i_m)
\leq C_5+C_6+\epsilon|\mathcal{X}^-(i_m)|\sum_{s=1}^{t-1}\Gamma_s^+(i_m). \label{eq:gamma_minus_partial_sum_bound}
\end{align}
Consequently, it holds that
\begin{align*}
\sum_{s=1}^{t-1}\left(C_1\,\Gamma_s^+(i_m)-C_2\,\Gamma_s^-(i_m)\right) &\geq \left(C_1-C_2\epsilon|\mathcal{X}^-(i_m)|\right)\sum_{s=1}^{t-1}\Gamma_s^+(i_m)-C_2(C_5+C_6), \\
\sum_{s=1}^{t-1}\left(\Gamma_s^+(i_m)+\Gamma_s^-(i_m)\right) &\leq C_5+C_6+\left(1+\epsilon|\mathcal{X}^-(i_m)|\right)\sum_{s=1}^{t-1}\Gamma_s^+(i_m).
\end{align*}
We can set $\epsilon=0$ when $m=K_\infty$, or choose a sufficiently small $\epsilon>0$ when $m<K_\infty$, such that
$C_1-C_2\epsilon|\mathcal{X}^-(i_m)|>0$. Substituting the above two inequalities into~\cref{eq:logit_diff_a_star_i_m} and using~\cref{lemma:sum_gamma_plus_i_m_=_infty}, we have that the positive term grows linearly in $\sum_{s=1}^{t-1}\Gamma_s^+(i_m)$, whereas the negative noise term grows on the order of the square root of this quantity times its logarithm. Therefore, with probability at least $1-K\delta$,
\begin{align*}
\lim_{t\to\infty}\left[z_t(a^\star)-z_t(i_m)\right]=\infty.
\end{align*}
Since $\delta>0$ can be arbitrarily small, it follows almost surely that $\lim_{t\to\infty}\frac{\pi_{\theta_t}(a^\star)}{\pi_{\theta_t}(i_m)}=\infty$.
This completes the proof.
\end{proof}

\begin{lemma}\label{lemma:sum_gamma_plus_i_m_=_infty}
    Consider any suboptimal action $i_m \in \mathcal{A}_\infty \backslash \{a^\star\}$. Under the conditions of~\cref{lemma:induction_on_A_infty}, \cref{alg:sto_spg} with any constant learning rate $\eta > 0$ ensures that almost surely
    \begin{align*}
        \lim_{t \to \infty} \sum_{s=1}^{t-1} \Gamma_s^+ (i_m)= \infty .
    \end{align*}
\end{lemma}

\begin{proof}
    We will prove this by contradiction. Assume that
    \begin{align}
        \sum_{s=1}^{\infty} \Gamma_s^+ (i_m) < \infty . \label{eq:sum_gamma_plus_i_m_less_infty}
    \end{align}
    Using this assumption, we can first prove the following claim. 

    \paragraph{Claim I:} $\sum_{s=1}^{\infty} \pis(i_m)\,(1 - \pis(i_m)) < \infty$.\\
    First, we have that
    \begin{align*}
        \MoveEqLeft
        \sum_{s=1}^{\infty} \pis(i_m)\,(1 - \pis(i_m)) = \underbrace{\sum_{s=1}^{\infty} \sum_{j < i_m }\pis(i_m)\,\pis(j)}_{\text{(i)}} + \underbrace{\sum_{s=1}^{\infty} \sum_{j > i_m} \pis(i_m)\,\pis(j)}_{\text{(ii)}} . \numberthis \label{eq:sum_pi_1_minus_pi_i_m}
    \end{align*}
    Since $(j, i_m) \in \mathcal{X}^{+} (i_m)$ for every $j < i_m$, we have that
    \begin{align*}
        \text{Term (i)} \leq \sum_{s=1}^{\infty} \Gamma_s^+ (i_m) < \infty .
    \end{align*}
    On the other hand, we have that
    \begin{align*}
        \MoveEqLeft
        \text{Term (ii)} = \sum_{s=1}^{\infty} \sum_{j > i_m} \pis(i_m)\,\pis(j) \\
        &= \underbrace{\sum_{s=1}^{\infty} \sum_{\substack{j > i_m \\ j \not\in \mathcal{A}_\infty}} \pis(i_m)\,\pis(j)}_{\text{(ii).a}} + \underbrace{\sum_{s=1}^{\infty} \sum_{\substack{j > i_m \\ j \in \mathcal{A}_\infty}} \pis(i_m)\,\pis(j)}_{\text{(ii).b}}
    \end{align*}
    For Term (ii).a, we have that
    \begin{align*}
        \MoveEqLeft
        \text{Term (ii.a)} = \sum_{s=1}^{\infty} \sum_{\substack{j > i_m \\ j \not\in \mathcal{A}_\infty}} \pis(i_m)\,\pis(j) \\
        &\leq \sum_{s=1}^{\infty} \sum_{\substack{j > i_m \\ j \not\in \mathcal{A}_\infty}} \pis(j) \tag{since $\pis(i_m) \leq 1$} \\
        &= \sum_{\substack{j > i_m \\ j \not\in \mathcal{A}_\infty}} \sum_{s=1}^{\infty} \pis(j) \\
        &< \infty \tag{by the Borel--Cantelli lemma (\cref{lemma:extended_borel_cantelli})}
    \end{align*}
    For Term (ii).b, we know that for all $j > i_m$ and $j \in \mathcal{A}_\infty$, $\lim_{s \to \infty} \frac{\pis(a^\star)}{\pis(j)} = \infty$. Therefore, for a large enough $\tau$, $\pis(a^\star) > \pis(j)$ for all $s \geq \tau$. Thus,
    \begin{align*}
        \MoveEqLeft
        \text{Term (ii.b)} = \sum_{s=1}^{\infty} \sum_{\substack{j > i_m \\ j \in \mathcal{A}_\infty}} \pis(i_m)\,\pis(j) \\
        &= \sum_{s=1}^{\tau-1} \sum_{\substack{j > i_m \\ j \in \mathcal{A}_\infty}} \pis(i_m)\,\pis(j) + \sum_{s=\tau}^{\infty} \sum_{\substack{j > i_m \\ j \in \mathcal{A}_\infty}} \pis(i_m)\,\pis(j) \\
        &\leq \sum_{s=1}^{\tau-1} \sum_{\substack{j > i_m \\ j \in \mathcal{A}_\infty}} \pis(i_m)\,\pis(j) + \abs{\mathcal{A}_\infty}\,\sum_{s=\tau}^{\infty} \pis(i_m)\,\pis(a^\star) \\ \tag{since $\pis(a^\star) > \pis(j)$ for all $s \geq \tau$}\\
        &\leq \sum_{s=1}^{\tau-1} \sum_{\substack{j > i_m \\ j \in \mathcal{A}_\infty}} \pis(i_m)\,\pis(j) + \abs{\mathcal{A}_\infty}\,\sum_{s=\tau}^{\infty} \Gamma^+_s(i_m) \tag{since $(a^\star, i_m) \in \mathcal{X}^+(i_m)$} \\
        &< \infty . \tag{by~\cref{eq:sum_gamma_plus_i_m_less_infty}}
    \end{align*}
    Putting everything back into~\cref{eq:sum_pi_1_minus_pi_i_m} yields that
    \begin{align}
        \sum_{s=1}^{\infty} \pis(i_m)\,(1 - \pis(i_m)) < \infty . \label{eq:sum_pi_1_minus_pi_i_m_less_infty}
    \end{align}
    This completes the proof of Claim I. Using Claim I, we proceed to prove Claims II and III.

    \paragraph{Claim II:} $\pis(i_m) > \frac{1}{2}$ infinitely often and $\pis(i_m) \leq \frac{1}{2}$ infinitely often.\\

    We first prove that $\pis(i_m) > \frac{1}{2}$ infinitely often. We will prove this by contradiction. We first assume that $\pis(i_m) \leq \frac{1}{2}$ for all large enough $s$. That is,
    \begin{align*}
        \exists \tau < \infty \; \text{s.t.} \; \forall s \geq \tau, \pis(i_m) \leq \frac{1}{2} . \numberthis \label{eq:pi_s_less_one_half_io}
    \end{align*}
    Then, $1 - \pis(i_m) > \frac{1}{2}$. This implies that
    \begin{align*}
        \sum_{s=1}^{\infty} \pis(i_m)\,(1 - \pis(i_m)) &= \sum_{s=1}^{\tau-1} \pis(i_m)\,(1 - \pis(i_m)) + \sum_{s=\tau}^{\infty} \pis(i_m)\,(1 - \pis(i_m)) \\
        &\geq \frac{1}{2}\sum_{s=\tau}^{\infty}\pis(i_m) \\
        &= \infty . \tag{by the Borel--Cantelli lemma (\cref{lemma:extended_borel_cantelli})}
    \end{align*}
    This contradicts~\cref{eq:sum_pi_1_minus_pi_i_m_less_infty}. Thus, \cref{eq:pi_s_less_one_half_io} cannot be true. The negation of~\cref{eq:pi_s_less_one_half_io} is that
    \begin{align*}
        \forall \tau < \infty \; \text{s.t.} \; \exists s \geq \tau, \pis(i_m) > \frac{1}{2} .
    \end{align*}
    This implies that $\pis(i_m) > \frac{1}{2}$ infinitely often. \\
    
    Similarly, we can prove that $\pis(i_m) \leq \frac{1}{2}$ infinitely often. We assume that $\pis(i_m) > \frac{1}{2}$ for all large enough $s$. That is,
    \begin{align*}
        \exists \tau < \infty \; \text{s.t.} \; \forall s \geq \tau, \pis(i_m) > \frac{1}{2} . \numberthis \label{eq:pi_s_larger_one_half_io}
    \end{align*}
    This implies that
    \begin{align*}
        \sum_{s=1}^{\infty} \pis(i_m)\,(1 - \pis(i_m)) &= \sum_{s=1}^{\tau-1} \pis(i_m)\,(1 - \pis(i_m)) + \sum_{s=\tau}^{\infty} \pis(i_m)\,(1 - \pis(i_m)) \\
        &\geq \frac{1}{2} \sum_{s=\tau}^{\infty} (1 - \pis(i_m)) \\
        &= \infty . \tag{by the Borel--Cantelli lemma (\cref{lemma:extended_borel_cantelli})}
    \end{align*}
    The last equality also relies on the fact that there exist at least two actions that are sampled infinitely often (\cref{lemma:2_arms_pulled_infinitely_many_times}). This contradicts~\cref {eq:sum_pi_1_minus_pi_i_m_less_infty}. Therefore, \cref{eq:pi_s_larger_one_half_io} cannot be true. The negation of~\cref{eq:pi_s_larger_one_half_io} is that
    \begin{align*}
        \forall \tau < \infty \; \text{s.t.} \; \exists s \geq \tau, \pis(i_m) \leq \frac{1}{2} .
    \end{align*}
    This implies that $\pis(i_m) \leq \frac{1}{2}$ infinitely often. \\

    This completes the proof of Claim II. Since the $\pis(i_m) > \frac{1}{2}$ and $\pis(i_m) \leq \frac{1}{2}$ happen infinitely often, we know that $\pis(i_m)$ oscillates between $[0, \frac{1}{2}]$ and $(\frac{1}{2}, 1]$ infinitely often. However, we can also prove the following contradictory claim, independent of Claim II. In particular, we define the upward crossing event as $\mathcal{E}_{\text{UP}} \coloneq \{ \pi_{\theta_s} (i_m) \leq \frac{1}{2} \; \text{and} \; \pi_{\theta_{s+1}} (i_m) > \frac{1}{2} \}$ and the set of iterations at which such an event happen as $\mathcal{S}_{\text{UP}} \coloneq \{ s \geq 1 \mid \pi_{\theta_s} (i_m) \leq \frac{1}{2} \; \text{and} \; \pi_{\theta_{s+1}} (i_m) > \frac{1}{2} \}$.
    
    \paragraph{Claim III:} $\mathcal{E}_{\text{UP}}$ only happens finitely often.\\
    We will prove this by contradiction. Suppose this event happens infinitely often. Then, according to~\cref{eq:lipschitz_update}, there exists a finite constant $M > 0$ such that for every $s \in \mathcal{S}_{\text{UP}}$, 
    \begin{align*}
        \pis (i_m) \geq \frac{\pi_{\theta_{s+1}}(i_m)}{M} \geq \frac{1}{2 M} .
    \end{align*}
    This further implies that $\pis (i_m) (1 - \pis (i_m)) \geq \frac{1}{4 M} > 0$ and
    \begin{align*}
         \sum_{s =1}^\infty \pis (i_m) (1 - \pis (i_m)) &\geq \sum_{s \in \mathcal{S}_{\text{UP}}} \pis (i_m) (1 - \pis (i_m)) \\
         &\geq \sum_{s \in \mathcal{S}_{\text{UP}}} \frac{1}{4 M} \\
         &= \infty , \tag{since $\mathcal{E}_{\text{UP}}$ happens infinitely often and $\abs{\mathcal{S}_{\text{UP}}} = \infty$}
    \end{align*}
    which contradicts~\cref{eq:sum_pi_1_minus_pi_i_m_less_infty}. Therefore, $\mathcal{E}_{\text{UP}}$ can occur only finite many times.

    \paragraph{The Contradiction.} Assuming that $\sum_{s=1}^{\infty} \Gamma_s^+ (i_m) < \infty$, we can prove Claim I. Using Claim I, we can independently prove Claim II and Claim III. Claim II can further imply that $\pis(i_m)$ oscillates between $[0, \frac{1}{2}]$ and $(\frac{1}{2}, 1]$ infinitely often, whereas Claim III suggests that this cannot happen. Putting everything together, the contradiction suggests that $\sum_{s=1}^{\infty} \Gamma_s^+ (i_m) < \infty$ cannot be true, hence implying that $\sum_{s=1}^{\infty} \Gamma_s^+ (i_m) = \infty$.
    
    This completes the proof.
\end{proof}

\begin{lemma} \label{eq:lipschitz_update}
    For any $a \in [K]$ and finite $t \geq 1$,~\cref{alg:sto_spg} with any constant learning rate $\eta > 0$ ensures that there exists a constant $M > 0$ such that
    \begin{align*}
        M^{-1} \leq \frac{\pi_{\theta_{t+1}}(a)}{\pi_{\theta_{t}}(a)} \leq M .
    \end{align*}
\end{lemma}

\begin{proof}
Starting from the update in~\cref{alg:sto_spg}, we have that for all $t \geq 1$
\begin{align*}
\theta_{t+1} &= \theta_t + \eta\,\phi^\top \left( \operatorname{diag}(\pi_{\theta_t}) -\pi_{\theta_t}\pi_{\theta_t}^\top \right) \hat{r}_t \\
&= \theta_t + \eta\,R_t (a_t) \phi^\top \left( e_{a_t} - \pit\right) \tag{since $\hat{r}_t = \frac{\mathbbm{1}\{a=a_t\}}{\pi_{\theta_t}(a)} R_t(a_t)$} \\
&= \theta_t + \eta\,R_t (a_t) \left( \phi_{a_t} - \sum_{j \in [K]} \pit(j) \phi_j \right) .
\end{align*}
Therefore, we have that for every $a \in [K]$,
\begin{align*}
|z_{t+1}(a)-z_t(a)| &= |\langle \phi_a,\theta_{t+1}-\theta_t\rangle| \\
&= \eta |R_t(a_t)| \left| \left\langle \phi_a, \phi_{a_t} - \sum_{j\in[K]}\pi_{\theta_t}(j)\phi_j \right\rangle \right| \\
&\leq \eta\,R_{\max} \left( |\langle \phi_a,\phi_{a_t}\rangle| + \sum_{j\in[K]}\pi_{\theta_t}(j) |\langle \phi_a,\phi_j\rangle| \right)\\
&\leq 2 \eta\,R_{\max}\,C_\phi \tag{let $C_\phi \coloneq \max_{i,j\in[K]}|\langle \phi_i,\phi_j\rangle|<\infty$}
\end{align*}
Therefore, we have that
\begin{align*}
\frac{\pi_{\theta_{t+1}}(a)}{\pi_{\theta_t}(a)} &= \frac{\frac{\exp(z_{t+1}(a))}{\sum_{a^\prime \in[K]}\exp(z_{t+1}(a^\prime))}}{\frac{\exp(z_t(a))}{\sum_{a^{\prime\prime}\in[K]}\exp(z_t(a^{\prime\prime}))}}\\
&= \exp(z_{t+1}(a)-z_t(a)) \cdot \frac{\sum_{a^{\prime\prime}\in[K]}\exp(z_t(a^{\prime\prime}))}{\sum_{a^\prime\in[K]}\exp(z_{t+1}(a^\prime))} \\
&= \exp(z_{t+1}(a)-z_t(a)) \cdot \sum_{a^{\prime\prime}\in[K]}\frac{\exp(z_{t+1}(a^{\prime\prime})) \exp(z_t(a^{\prime\prime}) - z_{t+1}(a^{\prime\prime}))}{\sum_{a^\prime\in[K]}\exp(z_{t+1}(a^\prime))} \\
&= \exp(z_{t+1}(a)-z_t(a)) \cdot \sum_{a^{\prime\prime}\in[K]} \pi_{\theta_{t+1}}(a^{\prime\prime}) \exp(z_t(a^{\prime\prime}) - z_{t+1}(a^{\prime\prime}))
\end{align*}
Since $|z_{t+1}(a)-z_t(a)|\leq 2\eta\,R_{\max}\,C_\phi$ for every $a\in[K]$, we have that
\begin{align*}
e^{-2\eta\,R_{\max}\,C_\phi} \leq \exp(z_{t+1}(a)-z_t(a)) \leq e^{2\eta\,R_{\max}\,C_\phi} \implies e^{-4\eta\,R_{\max}\,C_\phi} \leq \frac{\pi_{\theta_{t+1}}(a)}{\pi_{\theta_t}(a)} \leq e^{4\eta\,R_{\max}\,C_\phi}.
\end{align*}
Setting $M\coloneq e^{4\eta\,R_{\max}\,C_\phi} < \infty$ completes the proof.
\end{proof}

In addition, we also prove that $\abs{\mathcal{A}_\infty} \geq 2$ (i.e., at least two actions have been sampled infinitely often). For any action $a \in [K]$, define that $N_\infty(a) \coloneq \lim_{t \rightarrow \infty} N_t(a)$.

\begin{lemma} \label{lemma:2_arms_pulled_infinitely_many_times}
\cref{alg:sto_spg} with any constant learning rate $\eta>0$ ensures that $\sP\!\left(|\mathcal{A}_\infty|\geq2\right)=1$. Equivalently, almost surely, there exist two distinct, possibly random, actions $i,j\in[K]$ such that $N_\infty(i)=N_\infty(j)=\infty$.
\end{lemma}
\begin{proof}
By the pigeonhole principle, $|\mathcal{A}_\infty|\geq1$ for every sample path. We prove that $|\mathcal{A}_\infty|=1$ occurs with probability zero. For each $i\in[K]$, define the event
\begin{align*}
\mathcal{E}_i\coloneq\{\mathcal{A}_\infty=\{i\}\}.
\end{align*}
Fix $i\in[K]$ and suppose that $\mathcal{E}_i$ occurs. Then, $N_\infty(i)=\infty$ and $N_\infty(j)<\infty$ for every $j\neq i$.
We argue the existence of another action by contradiction. Suppose for every other action $j \in [K]$ and $j \neq i$, we have $N_\infty(j) < \infty$. According to the Borel--Cantelli lemma (\cref{lemma:extended_borel_cantelli}), for all $j \neq i$, we have that almost surely,
\begin{align*}
    \sum_{t=1}^\infty \pit(j) \coloneq \lim_{t \to \infty} \sum_{s=1}^t \pis(j) < \infty.
\end{align*}
Recall that from~\cref{alg:sto_spg}, we have the following update: $\thetatt = \thetat + \eta\,\phi^\top\,(\text{diag}(\pitheta) - \pitheta\,\pitheta^\top) \hat{r}_t$. Multiplying both sides by $\phi$ implies that
\begin{align*}
    z_{t+1} &= z_t + \eta\,\phi\,\phi^\top\,(\text{diag}(\pitheta) - \pitheta\,\pitheta^\top) \hat{r}_t .
\end{align*}
Then, for any action $\tilde{a} \in [K]$, 
\begin{align*}
    z_{t+1}(\tilde{a}) &= z_t(\tilde{a}) + \eta\,\sum_{a=1}^K \dpd{\phi_{\tilde{a}}, \phi_a}\,\pit(a)\,[\hat{r}(a) - \dpd{\pit, \hat{r}}]   \\
    &= z_t(\tilde{a}) + \eta\,\left[ \sum_{a=1}^K I_t(a) \parens*{ \dpd{\phi_{\tilde{a}}, \phi_a}\,(1 - \pit(a))\,R_t - \sum_{j \neq a} \dpd{\phi_{\tilde{a}}, \phi_j}\,\pit(j)\,R_t } \right] \\
    &= z_t(\tilde{a}) + \eta\,\left[I_t(i) \parens*{ \dpd{\phi_{\tilde{a}}, \phi_i}\,(1 - \pit(i))\,R_t - \sum_{j \neq i} \dpd{\phi_{\tilde{a}}, \phi_j}\,\pit(j)\,R_t } \right. \\
    & \quad + \left. \sum_{\substack{a=1 \\ a \neq i}}^K I_t(a) \parens*{ \dpd{\phi_{\tilde{a}}, \phi_a}\,(1 - \pit(a))\,R_t - \sum_{j \neq a} \dpd{\phi_{\tilde{a}}, \phi_j}\,\pit(j)\,R_t } \right].
\end{align*}
Recursing the above equation from $1$ to $t-1$, and using the triangle inequality, we have that
\begin{align*}
    \abs{z_{t}(\tilde{a}) - z_1(\tilde{a})} &\leq \eta\,\sum_{s=1}^{t-1} \abs{I_s(i) 
    \parens*{\dpd{\phi_{\tilde{a}}, \phi_i}\,(1 - \pis(i))\,R_s - \sum_{j \neq i} \dpd{\phi_{\tilde{a}}, \phi_j}\,\pis(j)\,R_s }} \\
    & \quad + \eta\,\sum_{s=1}^{t-1}\,\abs{\sum_{\substack{a=1 \\ a \neq i}}^K I_s(a) \parens*{ \dpd{\phi_{\tilde{a}}, \phi_a}\,(1 - \pis(a))\,R_s - \sum_{j \neq a} \dpd{\phi_{\tilde{a}}, \phi_j}\,\pis(j)\,R_s }} .
\end{align*}
Set $C \coloneq \max_{a, a^\prime} \abs{\dpd{\phi_a, \phi_{a^\prime}}}$. Since $\abs{R_t} \leq R_{\max}$ and using triangle inequality, we have that
\begin{align*}
\MoveEqLeft
    \abs{z_{t}(\tilde{a}) - z_1(\tilde{a})} \\ 
    &\leq \eta\,R_{\max}\,C\,\sum_{s=1}^{t-1} \left[ I_s(i)  \parens*{(1 - \pis(i)) + \sum_{j \neq i} \pis(j)} + \sum_{\substack{a=1 \\ a \neq i}}^K I_s(a) \parens*{(1 - \pis(a)) + \sum_{j \neq a} \pis(j)}  \right]  \\
    &= 2\,\eta\,R_{\text{max}}\,C\,\sum_{s=1}^{t-1} \left[I_s(i) \sum_{j \neq i} \pis(j) +  \sum_{\substack{a=1 \\ a \neq i}}^K I_s(a) \sum_{j \neq a} \pis(j) \right] \\
    &\leq 2 \eta\,R_{\text{max}}\,C\,\sum_{s=1}^{t-1} \left[\sum_{j \neq i} \pis(j) +  (K - 1)\,\sum_{\substack{a=1 \\ a \neq i}}^K I_s(a)\right] \\
    &= 2\,\eta\,R_{\text{max}}\,C \left[ \sum_{j \neq i} \sum_{s=1}^{t-1} \pis(j) + (K - 1)\,\sum_{\substack{a=1 \\ a \neq i}}^K \sum_{s=1}^{t-1} I_s(a) \right] \\
    &= 2\,\eta\,R_{\text{max}}\,C \left[ \sum_{j \neq i} \sum_{s=1}^{t-1} \pis(j) + (K - 1)\,\sum_{\substack{a=1 \\ a \neq i}}^K N_{t-1}(a) \right].
\end{align*}
From the assumption that $N_\infty(j) < \infty$, for any action $\tilde{a} \in [K]$, almost surely,
\begin{align*}
    \sup_{t \geq 1} \abs{z_t(\tilde{a})} \leq \sup_{t \geq 1} \abs{z_t(\tilde{a}) - z_1(\tilde{a})} + \abs{z_1(\tilde{a})} < \infty. 
\end{align*}
Since for all actions $\tilde{a} \in [K]$, the logit is always finite, there exists a finite constant $c_{\tilde{a}} > 0$ such that
\begin{align*}
\inf_{t \geq 1} \pit(\tilde{a}) &= \inf_{t \geq 1} \frac{\exp(z_t (\tilde{a}))}{\sum_{a^\prime \in [K]} \exp( z_t (a^\prime))} \geq c_{\tilde{a}} > 0 .
\end{align*}
This further implies that $\sum_{t=1}^\infty \pit(\tilde{a}) = \lim_{t \to \infty} \sum_{s=1}^t \pis(\tilde{a}) \geq \lim_{t \to \infty} t\,c_{\tilde{a}} = \infty$. According to the Borel--Cantelli lemma (\cref{lemma:extended_borel_cantelli}), we have almost surely, for all $\tilde{a} \in [K]$, $N_\infty(\tilde{a}) = \infty$, which contradicts the assumption that $N_\infty(j) < \infty$ for all $j \neq i$. Therefore, $\sP(\mathcal{E}_i)=0$ for every $i\in[K]$. Since $\{|\mathcal{A}_\infty|=1\}=\bigcup_{i\in[K]}\mathcal{E}_i$ and $K$ is finite, we conclude that $\sP(|\mathcal{A}_\infty|=1)=0$. Hence, $\sP(|\mathcal{A}_\infty|\geq2)=1$.
\end{proof}

\subsection{Rate of Convergence}
\label{subsec:asymptotic_convergence_rate}

\subsubsection{Proofs of \texorpdfstring{\cref{theorem:asymptotic_convergence_rate}}{}}

\asymptoticconvergencerate*
\begin{proof}
To start, we will show that there exists a large enough $\tau > 0$ and $C > 0$ such that for any action $i_m \in \mathcal{A}_\infty \backslash \{a^\star\}$ and all $t \geq \tau$,
\begin{align*}
    \pit(i_m) < \exp\parens*{-C\,\sum_{s=\tau}^{t-1} \pis(i_m)}.
\end{align*}
Given~\cref{theorem:stochastic_linear_bandits_with_arbitrary_learning_rates}, we know that $\lim_{t \to \infty} \pit(a^\star) = 1$. This implies that almost surely, there exists a large enough $\tau_1 \geq 1$ such that $\pit(a^\star) \geq \frac{1}{2}$ for all $t \geq \tau_1$.

Following the proof of~\cref{theorem:stochastic_linear_bandits_with_arbitrary_learning_rates}, for every $i_m\in\mathcal{A}_\infty\setminus\{a^\star\}$, there exist a constant $C_m>0$ and a sufficiently large finite iteration $\tau_m\geq1$ such that, for all $t>\tau_m$, almost surely,
\begin{align*}
    z_t(a^\star) - z_t(i_m) &\geq C_m\,\sum_{s=1}^{t-1} \Gamma^{+}_s (i_m) .
\end{align*}
We recall the following definitions.
\begin{align*}
    \mathcal{X}^{+} (i_m) &\coloneq \left\{ (i, j) \mid 1 \leq i < j \leq K, i < i_m, \bar{\phi}_{i_m, i, j} > 0 \right\} , \\
    \Gamma_s^+ (i_m) &\coloneq \sum_{(i, j) \in \mathcal{X}^{+} (i_m)} \pis(i)\,\pis(j) .
\end{align*}
This implies that almost surely,
\begin{align*}
\frac{\pit(a^\star)}{\pit(i_m)} &\geq \exp \left( C_m\,\sum_{s=1}^{t-1} \Gamma^{+}_s (i_m) \right) \\
&= \exp \left( C_m\,\sum_{s=1}^{t-1} \sum_{(i, j) \in \mathcal{X}^{+} (i_m)} \pis(i)\,\pis(j) \right) \\
&\geq \exp \left( C_m\,\sum_{s=1}^{t-1} \pis(a^\star)\,\pis(i_m) \right) \tag{since $(a^\star, i_m) \in \mathcal{X}^{+} (i_m)$}\\
&\geq \exp \left( \frac{C_m}{2}\,\sum_{s=\tau_1}^{t-1} \pis(i_m) \right). \tag{since $\pit(a^\star) \geq \frac{1}{2}$ for all $t \geq \tau_1$}
\end{align*}
Therefore, by defining that $C = \min_{m \in [2, K_\infty]} C_m / 2$ and $\tau = \max \{\tau_i\}_{i \in [K_\infty]}$, it implies that almost surely,
\begin{align*}
    \frac{\pit(i_m)}{\pit(a^\star)} < \exp(- C\,\sum_{s=\tau}^{t-1} \pis(i_m)) .
\end{align*}
Since $\pit(a^\star) \leq 1$, we have that almost surely,
\begin{align*}
    \pit(i_m) < \exp(- C\,\sum_{s=\tau}^{t-1} \pis(i_m)) .
\end{align*}
Therefore, we have that almost surely,
\begin{align*}
\sum_{s=\tau}^t \pis(i_m) - \sum_{s=\tau}^{t-1} \pis(i_m) < \exp(- C\,\sum_{s=\tau}^{t-1} \pis(i_m))
\end{align*}
Using \cref{lemma:sequence_equality,lemma:sequence_inequality} with $x_n=\sum_{s=\tau}^{\tau+n}\pis(i_m)$ and $y_{n+1}=y_n+e^{-A\,y_n}$ where $y_0 = \max \{ x_0, 1\} = 1$, and $A = C$, we have that
\begin{align*}
\sum_{s=\tau}^t \pis(i_m) &\leq \frac{1}{C} \ln \left( C\,t + e^C  \right) + \frac{\pi^2}{6\,C} .
\end{align*}
According to the Borel--Cantelli lemma (\cref{lemma:extended_borel_cantelli}), we also know that almost surely,
\begin{align*}
    \forall k \in [K] \backslash \mathcal{A}_\infty: \; \sum_{s=1}^\infty \pis(k) < \infty .
\end{align*}
We then define that
\begin{align*}
    C^\prime \coloneq
    \begin{cases}
    \max_{k\in[K]\setminus\mathcal A_\infty}\sum_{s=1}^{\infty}\pi_{\theta_s}(k), & [K]\setminus\mathcal A_\infty\neq\varnothing \\
    0, & [K]\setminus\mathcal A_\infty=\varnothing
    \end{cases} .
\end{align*}
Combining the above two inequalities yields that almost surely,
\begin{align*}
    \sum_{s=\tau}^t (1 - \pis(a^\star)) &= \sum_{k \in [K] \backslash \{a^\star\}} \sum_{s=\tau}^t \pis(k) \\
    &= \sum_{k \in \mathcal{A}_\infty \backslash \{a^\star\}} \sum_{s=\tau}^t \pis(k) + \sum_{k \in [K] \backslash \mathcal{A}_\infty} \sum_{s=\tau}^t \pis(k) \\
    &\leq \frac{K}{C} \ln \left( C\,t + e^C  \right) + \frac{\pi^2\,K}{6\,C} + K\,C^\prime .
\end{align*}
Finally, the suboptimality gap can be written as
\begin{align*}
r(a^\star) - \dpd{\pis, r} = \sum_{a \neq a^\star} \pis(a)\,(r(a^\star) - r(a)) \leq 2 R_{\max}\,(1 - \pis(a^\star)) .
\end{align*}
Averaging the suboptimality gap from $s = \tau$ to $T$, we finally have that almost surely,
\begin{align*}
\frac{\sum_{s=\tau}^T r(a^\star) - \dpd{\pis, r}}{T - \tau + 1} &\leq \frac{2 R_{\max}\,\sum_{s=\tau}^T (1 - \pis(a^\star)) }{T - \tau + 1} \\
&\leq \frac{2 R_{\max} \left[ \frac{K}{C} \ln \left( C\,T + e^C  \right) + \frac{\pi^2\,K}{6\,C} + K\,C^\prime \right] }{T - \tau + 1} .
\end{align*}
This completes the proof.
\end{proof}

\subsubsection{Proofs of Lemmas}

\begin{lemma}\label{lemma:sequence_equality}
Consider a sequence $\{y_n\}_{n=0}^\infty$ by the recurrence relation $y_{n+1} = y_n + e^{-A\,y_n}$ where $A > 0$. If $y_0 \geq \frac{\ln (A)}{A}$, then, for all $n > 0$,
\begin{align*}
    y_n \leq \frac{1}{A} \ln(A n + e^{A\,y_0})  + \frac{\pi^2}{6\,A} .
\end{align*}
\end{lemma}

\begin{proof}
Define the function $f$ as $f(t) \coloneq \frac{1}{A} \ln (A\,t + e^{A\,y_0})$. Take the derivative of $f(t)$ w.r.t. $t$, then we have that
\begin{align*}
f^\prime(t) = \frac{1}{A\,t + e^{A\,y_0}} = e^{-A\,f(t)} > 0.
\end{align*}
Hence, $f(t)$ is increasing on $(0, +\infty)$. We then prove by induction $f(n) \leq y_n$ for all $n \in \mathbb{N}$.
\paragraph{Base Case:} $f(0) = y_0$.
\paragraph{Inductive Hypothesis:} Suppose $f(k) \leq y_k$ for some $k \geq 0$.
\paragraph{Inductive Step:} 
Consider the function $g(x) = x + e^{-A x}$. $g(x)$ is decreasing on $(-\infty, \frac{\ln (A)}{A})$ and is increasing on $(\frac{\ln (A)}{A}, \infty)$. Given that $f(0) = y_0 \geq \frac{\ln (A)}{A}$ and $f(t)$ is increasing on $(0, +\infty)$, we have $f(n) \geq \frac{\ln (A)}{A}$. Using the fundamental theorem of calculus, we have that
\begin{align*}
f(k+1) &= f(k) + \int_k^{k+1} f^\prime(s)\,ds \\
&= f(k) + \int_k^{k+1} e^{-A\,f(s)}\,ds \\
&\leq f(k) + e^{-A\,f(k)} \tag{since $e^{-A\,f(s)}$ is decreasing for $s \in [k, k+1]$} \\
&= g(f(k)) \\
&\leq g(y_k) \\
\tag{since $f(k) \leq y_k$, $f(k) \geq \frac{\ln (A)}{A}$, and $g(x)$ is increasing on $(\frac{\ln (A)}{A}, \infty)$} \\
&= y_{k+1} \tag{by the definition of $y_{k+1}$}
\end{align*}
This completes the inductive proof. \\
Next, we can upper-bound $y_n$ as follows.
\begin{align*}
y_n &= f(n) + \underbrace{y_n - f(n)}_{\eqcolon \Delta_n} \\
&= f(n) + (\Delta_n - \Delta_{n-1}) + (\Delta_{n-1} - \Delta_{n-2}) + \cdots + (\Delta_1 - \Delta_0) + \Delta_0 \\
&= f(n) + \sum_{i=0}^{n-1} (\Delta_{i+1} - \Delta_i) \tag{since $\Delta_0 = 0$} \\
&= f(n) + \sum_{i=0}^{n-1} e^{-A\,y_i} - \frac{1}{A} \ln (1 + \frac{A}{A\,i + e^{A\,y_0}}) \\
&\leq f(n) + \sum_{i=0}^{n-1} e^{-A\,y_i} - \frac{1}{A\,i + A + e^{A\,y_0}} \tag{since $\forall x > -1, \ln(x + 1) \geq \frac{x}{1+x}$} \\
&\leq f(n) + \sum_{i=0}^{n-1} \frac{1}{A\,i + e^{A\,y_0}} - \frac{1}{A\,i + A + e^{A\,y_0}} \tag{since $f(i) \leq y_i$}\\
&= f(n) + \sum_{i=0}^{n-1} \frac{A}{(A\,i + e^{A\,y_0}) (A\,i + A + e^{A\,y_0})} \\
&\leq f(n) + \frac{1}{A} \sum_{i=0}^{n-1} \frac{1}{(i + \frac{1}{A}  e^{A\,y_0})^2} \\
&\leq f(n) + \frac{1}{A} \sum_{i=0}^{n-1} \frac{1}{(i + 1)^2} \tag{since $y_0 \geq \frac{\ln(A)}{A}$} \\
&\leq f(n) + \frac{\pi^2}{6\,A} \tag{since $\lim_{n \to \infty} \sum_{i=1}^{n} \frac{1}{i^2} = \frac{\pi^2}{6}$} \\
&= \frac{1}{A} \ln (A n + e^{A\,y_0}) + \frac{\pi^2}{6\,A} .
\end{align*}
This completes the proof.
\end{proof}

\begin{lemma} \label{lemma:sequence_inequality}
Given a sequence $\{y_n\}_{n=0}^\infty$ such that $y_{n+1} = y_n + e^{-A\,y_n}$ for all $n \geq 0$ where $A > 0$. Considering a nonnegative sequence $\{x_n\}_{n=0}^\infty$ such that $x_{n+1} \leq x_n + e^{-A x_n}$ for all $n \geq 0$. If $y_0 \geq \max(x_0, 1)$, then $x_n \leq y_n$ for all $n \geq 0$.
\end{lemma}

\begin{proof}
First, we know that $y_0 \geq 1 > \frac{\ln(A)}{A}$. We will then prove this lemma by induction.
\paragraph{Base Case:} $x_0 \leq y_0$.
\paragraph{Inductive Hypothesis:} Suppose that $x_k \leq y_k$ for some $k \geq 0$.
\paragraph{Inductive Step:}
Consider the function $g(x) = x + e^{-A x}$. $g(x)$ is decreasing on $(-\infty, \frac{\ln (A)}{A})$ and increasing on $(\frac{\ln (A)}{A}, \infty)$. Since $y_0 \geq \frac{\ln (A)}{A}$, $y_n \geq \frac{\ln (A)}{A}$ for all $n \geq 0$. Similarly, $y_n \geq 1$ for all $n \geq 0$. Then, we have the following two cases. \\

\noindent \textit{Case A:} When $0 < x_k \leq \frac{\ln (A)}{A}$. 
\begin{align*}
x_{k+1} &\leq x_k + e^{-A x_k} \\
&\leq g(0) \tag{since $x_k \geq 0$ and $g(x)$ is decreasing on $(0, \frac{\ln (A)}{A})$} \\
&= 1  \\
&\leq y_{k+1} \tag{since $y_n \geq 1$ for all $n \geq 0$}.
\end{align*}

\noindent \textit{Case B:} When $x_k > \frac{\ln (A)}{A}$.
\begin{align*}
x_{k+1} &\leq x_k + e^{-A x_k} \\
&= g(x_k) \\
&\leq g(y_k) \tag{since $x_k \leq y_k$ and $g(x)$ is increasing on $(\frac{\ln (A)}{A}, \infty)$} \\
&= y_{k+1}.
\end{align*}
Combining both cases, we have $x_{k+1} \leq y_{k+1}$, which completes the inductive proof.
\end{proof}

\clearpage

\section{Experiments}
\label{appendix:experiments}

In this section, we empirically evaluate the global convergence of $\LinSPG$ in both the exact and stochastic bandit settings.

\subsection{Exact Setting}
\label{subsec:experiments_deterministic_linear_bandits}

We first evaluate $\LinSPG$ in the exact setting, where the reward is known, and the policy gradient update can be computed exactly without sampling noise. We consider the bandit problems with $K=3$ and $K=6$ actions. 

For every environment, the features $\phi$ and the rewards $r$ are generated so that the reward-ordering preservation condition in~\cref{assumption:reward_ordering_preservation} holds. For $K=3$, the features additionally satisfy~\cref{assumption:feature_conditions_for_three_armed_linear_bandits}, whereas for $K=6$, they satisfy~\cref{assumption:general_feature_conditions}. We use the learning rate specified in~\cref{eq:step_size_for_deterministic_bandits} and evaluate various feature dimensions $d$. For each value of $K$, we randomly generate $50$ environments and run $\LinSPG$ for $10^6$ iterations. Across all the tested environments and feature dimensions, the suboptimality gap decreases toward $0$. These results illustrate that $\LinSPG$ converges to the optimal policy for different feature dimensions $d$, providing empirical support for the convergence behavior predicted by \cref{theorem:three_armed_deterministic_linear_bandits,theorem:deterministic_linear_bandits}.

\begin{figure}[ht]
  \centering
  \begin{minipage}[b]{0.49\linewidth}
      \centering
      \includegraphics[width=0.77\linewidth]{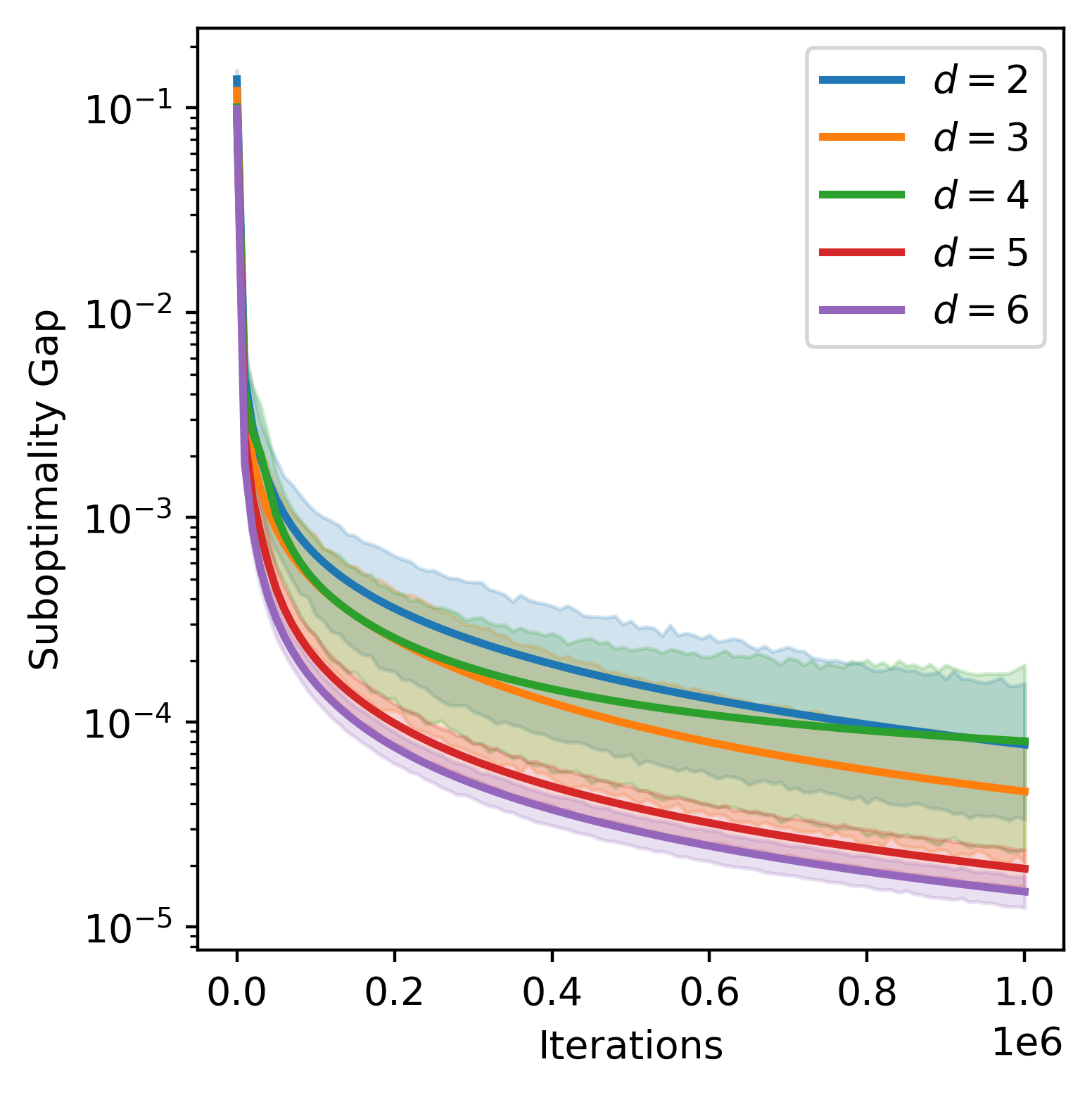}
      \caption*{(a) $K=3$}
  \end{minipage}
  \begin{minipage}[b]{0.49\linewidth}
      \centering
      \includegraphics[width=0.77\linewidth]{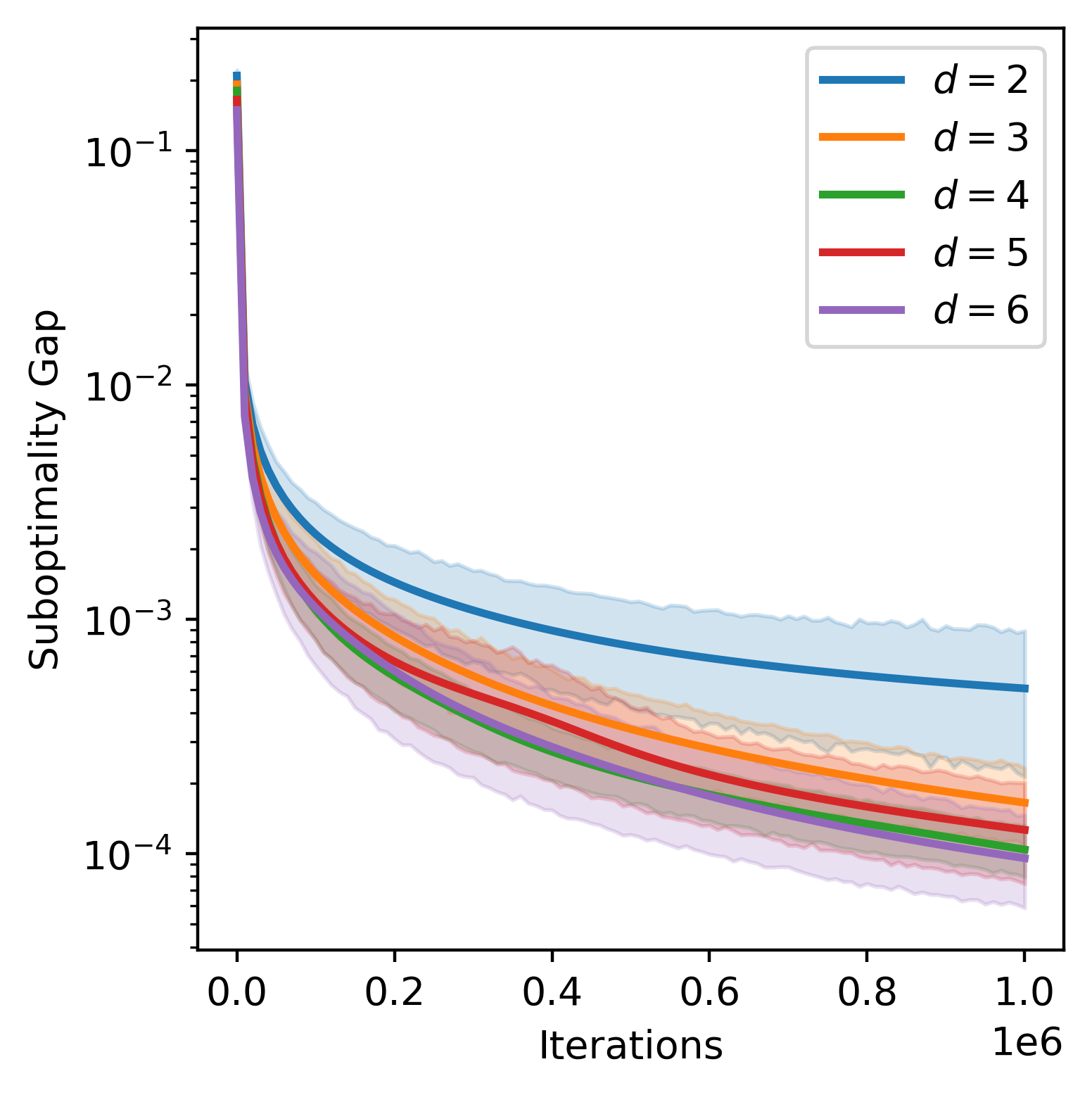}
      \caption*{(b) $K=6$}
  \end{minipage}
  \caption{
    Convergence of $\LinSPG$ in the exact setting for (a) $K=3$ and (b) $K=6$ across different feature dimensions.
  }
\end{figure}

\subsection{Stochastic Setting}
\label{subsec:experiments_stochastic_linear_bandits}

We next evaluate $\LinSPG$ in the stochastic setting, where only noisy rewards are observed. We choose $K=6$ actions and feature dimension $d=3$. To examine the effect of the learning rate, we run $\LinSPG$ with several constant learning rates.

We generate five different environments and evaluate each of them under three reward distributions (Bernoulli, Gaussian (truncated to have bounded support), and Beta) with a fixed mean reward $r\in[0,1]^K$. For each choice of environment and learning rate, we perform $5$ independent runs with different seeds, each run for $10^6$ iterations. For each environment, the features $\phi$ and the mean reward $r$ are randomly generated such that~\cref{assumption:no_identical_arms,assumption:reward_ordering_preservation,assumption:general_feature_conditions} are satisfied. For all the tested learning rates and across all environments, the suboptimality gap approaches $0$ within $10^6$ iterations for most of the runs, consistent with the convergence guarantees in~\cref{theorem:stochastic_linear_bandits,theorem:stochastic_linear_bandits_with_arbitrary_learning_rates}.

\begin{figure}[ht]
    \centering
    \includegraphics[width=0.88\linewidth]{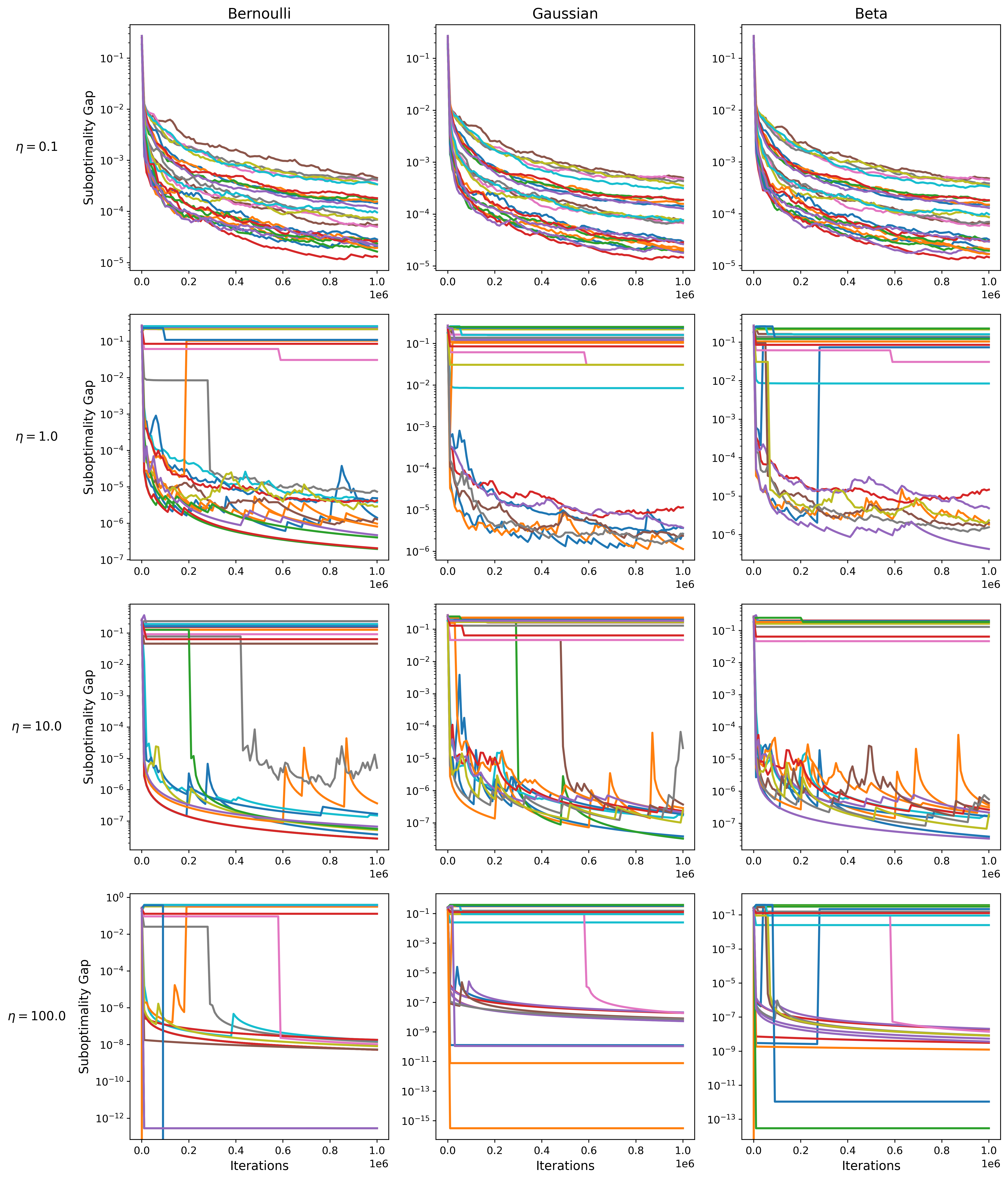}
    \caption{
    Convergence of $\LinSPG$ in the stochastic setting for $K=6$ and $d=3$ under different constant learning rates across different reward distributions.
    }
    \label{fig:arbitrary_learning_rates}
\end{figure}

\vspace{-1cm}
\end{document}

%% file: preamble.tex
\usepackage[utf8]{inputenc}
\usepackage[T1]{fontenc}
\usepackage{url}
\usepackage{booktabs}
\usepackage{amsfonts}
\usepackage{nicefrac}
\usepackage{microtype}
\usepackage{xcolor}
\usepackage{bm}
\usepackage{physics}
\usepackage{mathtools}
\usepackage{bbm}
\usepackage{multicol}
\usepackage{algorithm}
\usepackage{algorithmic}
\usepackage{graphicx}
\usepackage{colortbl}
\usepackage{utfsym}
\usepackage{multirow}
\usepackage{multicol}
\usepackage{enumitem}
\usepackage{minitoc}
\usepackage{amsmath}
\usepackage{amssymb}
\usepackage{thmtools}
\usepackage{hyperref}
\usepackage[capitalize]{cleveref}
\usepackage{silence}
\usepackage{caption}

\allowdisplaybreaks

\def\1{\bm{1}}

\newcommand{\hgrad}[1]{{\nabla \widetilde{f}}(#1)}

\def\indicator#1{{\mathbbm{1}\left\{ #1 \right\}}}

\def\pit{\pi_{\theta_t}}
\def\pis{\pi_{\theta_s}}

\def\pitheta{\pi_{\theta}}
\def\piz{\bar{\pi}_{z}}

\def\pitt{{\pi_{\theta_{t+1}}}}
\def\thetat{{\theta_t}}
\def\thetatt{{\theta_{t+1}}}

\def\etat{{\eta_t}}

\DeclarePairedDelimiter\parens{\lparen}{\rparen}

\DeclarePairedDelimiter\angles{\langle}{\rangle}

\renewcommand{\norm}[1]{\left\lVert#1\right\rVert}

\def\dpd#1{{\angles*{ #1 }}}
\def\eps{{\epsilon}}

\def\supnorm#1{\left\lVert#1\right\rVert_\infty}
\def\normsq#1{{\left\lVert#1\right\rVert^2_2}}

\DeclareMathAlphabet{\mathsfit}{\encodingdefault}{\sfdefault}{m}{sl}
\SetMathAlphabet{\mathsfit}{bold}{\encodingdefault}{\sfdefault}{bx}{n}

\def\gC{{\mathcal{C}}}

\def\gF{{\mathcal{F}}}

\def\sI{{\mathbb{I}}}

\def\sP{{\mathbb{P}}}

\def\sR{{\mathbb{R}}}

\newcommand{\E}{\mathbb{E}}

\newcommand{\R}{\mathbb{R}}

\newcommand{\softmax}{\mathrm{softmax}}

\newcommand{\Var}{\mathrm{Var}}

\DeclareMathOperator*{\argmax}{arg\,max}
\DeclareMathOperator*{\argmin}{arg\,min}

\def\Var{\text{Var}}

\newcommand{\z}{z}

\newcommand{\gradf}[1]{\nabla f(#1)}

\newcommand{\tht}{{\theta_{t}}}
\newcommand{\thtt}{{\theta_{t+1}}}

\newcommand{\ztt}{z_{t+1}}

\renewcommand{\gradf}[1]{\nabla f(#1)}

\newcommand{\LinSPG}{\texttt{Lin-SPG}}